\date{}
\documentclass[oneside,onecolumn]{article}

\usepackage{microtype} 
\usepackage{fullpage}

\usepackage[english]{babel} 

\usepackage{booktabs} 
\usepackage{listings}
\usepackage{lettrine} 

\usepackage[normalem]{ulem}
\usepackage{cite}
\usepackage{url,multirow,comment}

\usepackage{amsmath,amssymb,amsthm}
\usepackage{graphicx,color}
\usepackage[figuresright]{rotating}
\usepackage[utf8]{inputenc} 
\usepackage[T1]{fontenc}    
\usepackage{url}            
\usepackage{booktabs}       
\usepackage{enumitem}
\usepackage{amsfonts}       
\usepackage{amssymb,mathtools}
\usepackage{amsfonts}
\usepackage{nicefrac}       
\usepackage{microtype}      
\usepackage{amsmath,amssymb,amsthm}
\usepackage{subcaption}
\usepackage{algorithm}
\usepackage[compatible]{algpseudocode}
\usepackage{pbox}
\usepackage{tkz-euclide}
\usetkzobj{all}
\usepackage{tikz}
\usepackage{url}
\usepackage{todonotes}
\usepackage{bm}
\usepackage{cleveref}

\allowdisplaybreaks
\renewcommand\footnotemark{}

\newtheorem{lem}{Lemma}
\newtheorem{thm}{Theorem}
\newtheorem{defn}{Definition}
\newtheorem{coro}{Corollary}

\newtheorem{rem}{Remark}

\newcommand{\E}[1]{\mathbb{E}\left[{#1}\right]}

\newcommand{\Esub}[2]{\mathbb{E}_{#1}\left[{#2}\right]}
\newcommand{\wts}{\mathbf{w}}
\newcommand{\mb}{m}
\newcommand{\iters}{J}
\newcommand{\workers}{P}
\newcommand{\lips}{L}

\newcommand{\Exp}{\textit{Exp}}

\crefname{equation}{}{}
\Crefname{equation}{}{}
\crefname{thm}{theorem}{theorems}
\Crefname{thm}{Theorem}{Theorems}
\crefname{clm}{claim}{claims}
\Crefname{clm}{Claim}{Claims}
\Crefname{coro}{Corollary}{Corollaries}
\Crefname{lem}{Lemma}{Lemmas}
\Crefname{sec}{Section}{Sections}
\crefname{app}{appendix}{appendices}
\Crefname{app}{Appendix}{Appendices}
\Crefname{part}{Part}{Parts}
\crefname{prop}{proposition}{propositions}
\Crefname{prop}{Proposition}{Propositions}
\Crefname{propty}{Property}{Properties}
\crefname{figure}{figure}{figures}
\Crefname{figure}{Figure}{Figures}
\crefname{defn}{definition}{definitions}
\Crefname{defn}{Definition}{Definitions}
\crefname{fact}{fact}{facts}
\Crefname{fact}{Fact}{Facts}
\crefname{appendix}{appendix}{appendices}
\Crefname{appendix}{Appendix}{Appendices}
\crefname{algo}{algorithm}{algorithms}
\Crefname{algo}{Algorithm}{Algorithms}
\crefname{algorithm}{algorithm}{algorithms}
\Crefname{algorithm}{Algorithm}{Algorithms}
\crefname{conj}{conjecture}{conjectures}
\Crefname{conj}{Conjecture}{Conjectures}
\crefname{obs}{observation}{observations}
\Crefname{obs}{Observation}{Observations}
\crefname{rem}{remark}{remarks}
\Crefname{rem}{Remark}{Remarks}

\begin{document}

\title{Slow and Stale Gradients Can Win the Race}

\author{Sanghamitra Dutta, Jianyu Wang and Gauri Joshi \\
\thanks{The authors are with the Department of Electrical and Computer Engineering, Carnegie Mellon University. Author Contacts: S. Dutta (sanghamd@andrew.cmu.edu), J. Wang (jianyuw1@andrew.cmu.edu), G. Joshi (gaurij@andrew.cmu.edu)}
\thanks{Some of the results have appeared in AISTATS 2018 (see \cite{dutta2018slow}). This is an extended version with additional results, in particular, an adaptive synchronicity strategy called AdaSync.}
\normalsize Carnegie Mellon University
}

\maketitle

\begin{abstract}
Distributed Stochastic Gradient Descent (SGD) when run in a synchronous manner, suffers from delays in \emph{runtime} as it waits for the slowest workers (stragglers). Asynchronous methods can alleviate stragglers, but cause gradient staleness that can adversely affect the convergence \emph{error}. In this work we present a novel theoretical characterization of the speedup offered by asynchronous methods by analyzing the trade-off between the \emph{error} in the trained model and the actual training \emph{runtime} (wallclock time). The main novelty in our work is that our runtime analysis considers random straggling delays, which helps us design and compare distributed SGD algorithms that strike a balance between straggling and staleness. We also provide a new error convergence analysis of asynchronous SGD variants without bounded or exponential delay assumptions. Finally, based on our theoretical characterization of the error-runtime trade-off, we propose a method of gradually varying synchronicity in distributed SGD and demonstrate its performance on CIFAR10 dataset.
\end{abstract}

\section{Introduction}

Stochastic gradient descent (SGD) is the backbone of most state-of-the-art machine learning algorithms. Thus, improving the stability and convergence rate of SGD algorithms is critical for making machine learning algorithms fast and efficient. Classical SGD was designed to be run on a single computing node, and its error-convergence with respect to the number of iterations has been extensively analyzed and improved in optimization and learning theory literature. Due to the massive training data-sets and deep neural network architectures used today, running SGD at a single node can be prohibitively slow. This calls for distributed implementations of SGD, where gradient computation and aggregation is parallelized across multiple worker nodes. Although parallelism boosts the amount of data processed per iteration, it exposes SGD to unpredictable node slowdown and communication delays stemming from variability in the computing infrastructure. Thus, there is a critical need to make distributed SGD fast, and yet robust to system variability.

The convergence speed of distributed SGD is a product of two factors: 1) the error in the trained model versus the number of iterations, and 2) the number of iterations completed per second. Traditional single-node SGD analysis focuses on optimizing the first factor, because the second factor is generally a constant when SGD is run on a single dedicated server. In distributed SGD, which is often run on shared cloud infrastructure, the second factor depends on several aspects such as the number of worker nodes, their gradient computation delays, and the protocol (synchronous or asynchronous) used to aggregate their gradients. Hence, in order to achieve the fastest convergence speed we need: 1) optimization techniques to maximize the error-convergence rate with respect to iterations, and 2) scheduling techniques to maximize the number of iterations completed per second. These directions are inter-dependent and need to be explored together rather than in isolation. While many works have advanced the first direction, the second is less explored from a theoretical point of view, and the juxtaposition of both is an unexplored problem. Our goal is to design SGD algorithms that easily lend themselves to distributed implementations, and are robust to fluctuations in computation and network delays as well as unpredictable node failures. This work improves the true convergence speed of distributed SGD with respect to wallclock time by jointly designing scheduling techniques to reduce per-iteration delay, and optimization algorithms to minimize error-versus-iterations.

A commonly used distributed SGD framework, which is first deployed at a large-scale in Google's DistBelief \cite{dean2012large}, is the parameter server framework, which consists of a central parameter server (PS) that is used to aggregate gradients computed by worker nodes as shown in \Cref{fig:overview} (a). In synchronous SGD, the PS waits for all workers to push gradients before it updates the model parameters. Random delays in computation (referred to as straggling) are common in today's distributed systems as pointed out in the influential work of \cite{dean2013tail}. Waiting for slow and straggling workers can diminish the speedup offered by parallelizing the training. To alleviate the problem of stragglers, SGD can be run in an asynchronous manner, where the central parameters are updated without waiting for all workers. However, workers may return \emph{stale} gradients that were evaluated at an older version of the model, and this can make the algorithm unstable. Synchronous SGD typically has better convergence error but has a higher wallclock runtime per iteration because it requires synchronization of straggling workers. On the other hand, asynchronous SGD has faster wallclock runtime per iteration but it also has higher convergence error due to the problem of gradient staleness. 

Our goal is to achieve the lower envelope of the error-runtime trade-offs achieved by synchronous and asynchronous SGD (see \Cref{fig:overview}(b)), which characterizes the best error-runtime trade-off. Towards achieving this goal, in this work we present a systematic theoretical analysis of the trade-off between error and the actual runtime (instead of iterations), modelling wallclock runtimes as random variables with a general distribution. 
Based on our analysis, we propose AdaSync, which is a method of adaptively increasing the number of nodes whose gradients are aggregated synchronously by the central PS. Our theoretical results are also substantiated with experiments on 
CIFAR10~\cite{krizhevsky2009learning} dataset.

\begin{figure}[t]
\centering
 \begin{subfigure}{.4\textwidth}
    \centering
    \includegraphics[width=0.85\textwidth]{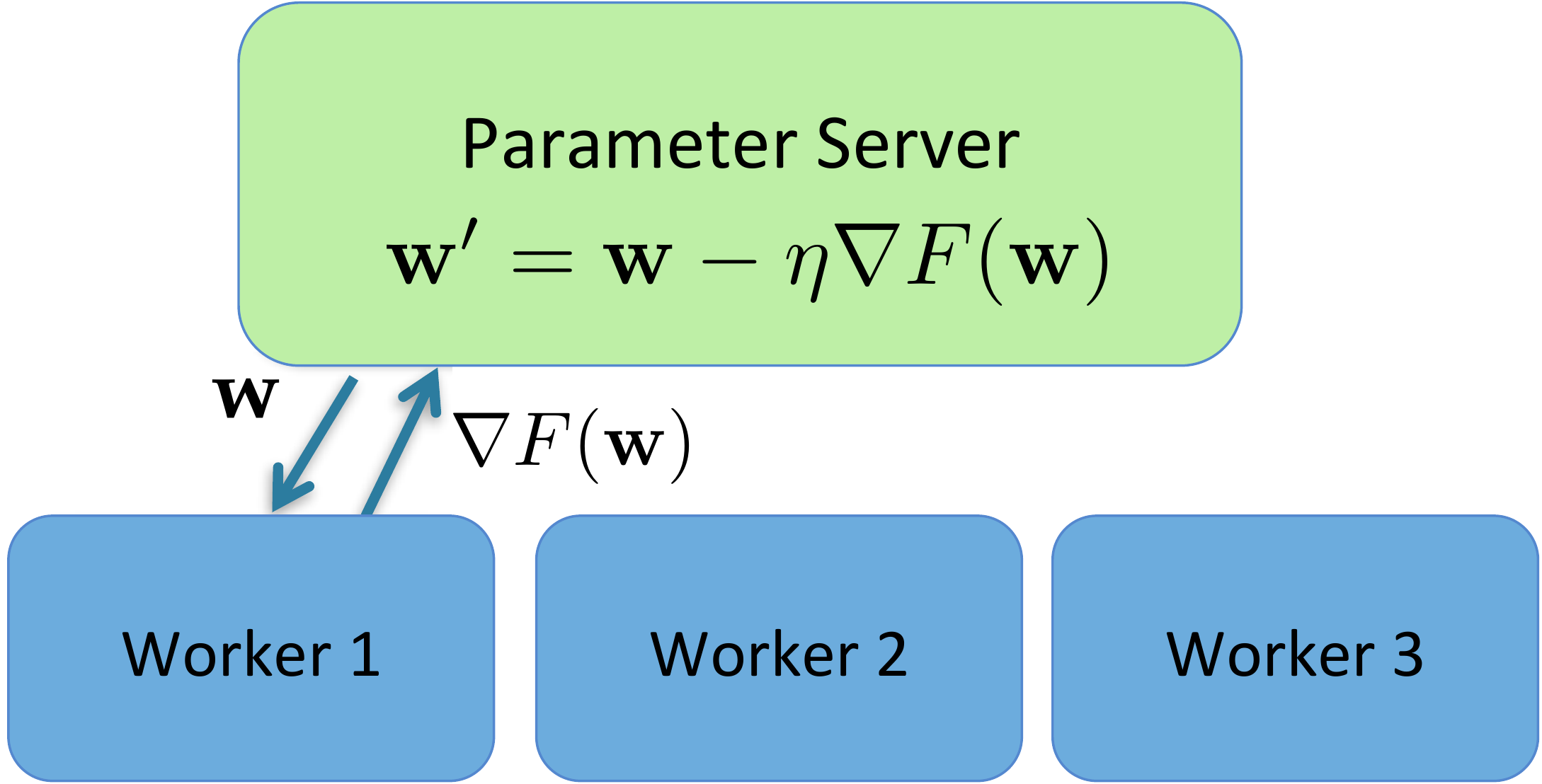}
    \label{fig:param_server}
 \end{subfigure}%
 ~
  \begin{subfigure}{.6\textwidth}
    \centering
    \includegraphics[width=0.8\textwidth]{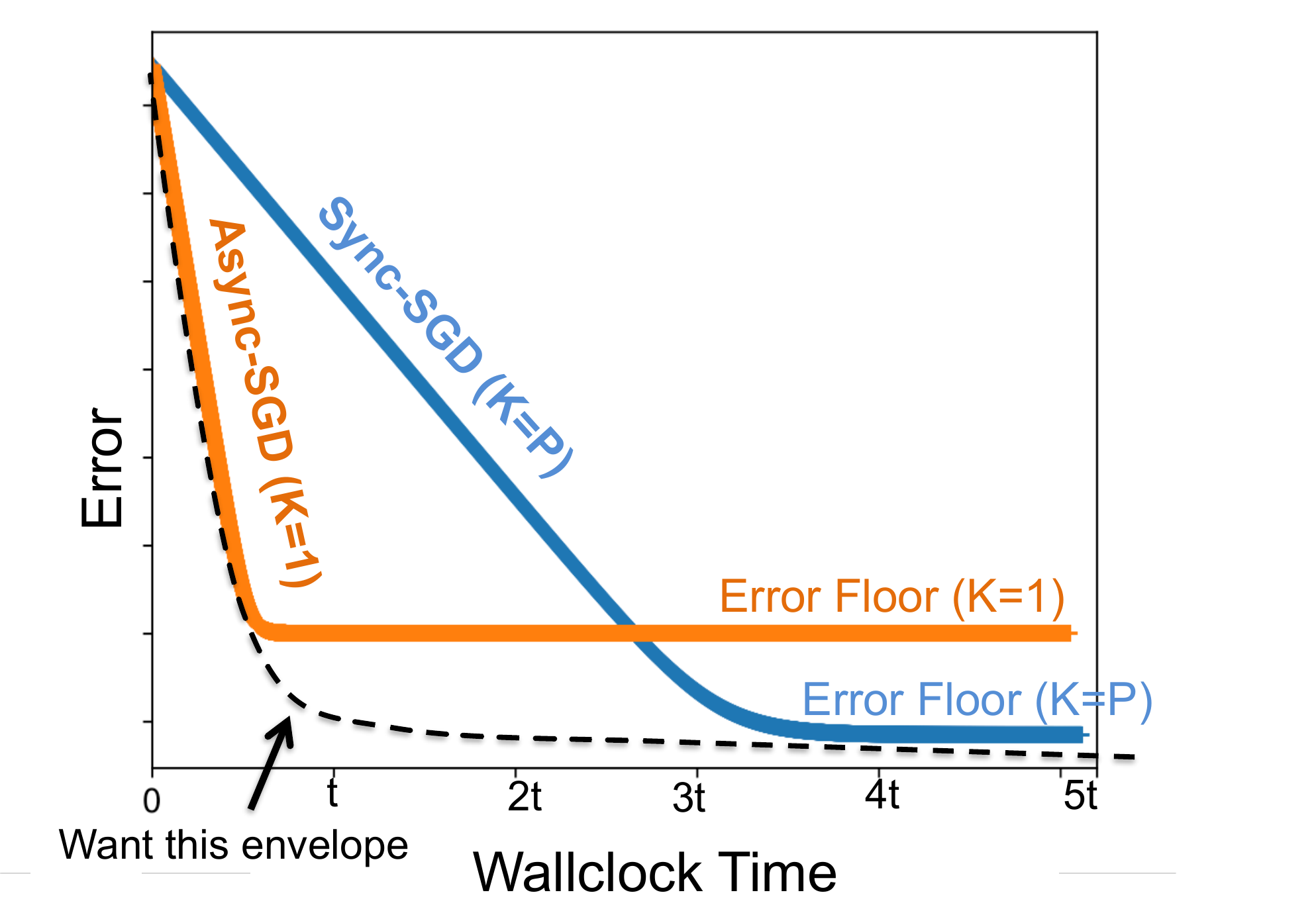}
 \end{subfigure}%
\caption{(a) The parameter server framework (b) Synchronous SGD has lower error floor but higher runtime, while, asynchronous SGD converges faster but has a higher error floor. We want to achieve the the lower envelope between the two curves which characterizes the best error-runtime trade-off. \label{fig:overview}}
\end{figure}

\subsection{Related Works}

\textbf{Single Node SGD:}
Analysis of gradient descent dates back to classical works \cite{boyd2004convex} in the optimization community. The problem of interest is the minimization of empirical risk of the form:
\begin{equation}
\min_{\wts} \left\{ F(\wts)\overset{\text{def}}{=} \frac{1}{N}\sum_{n=1}^{N} f(\wts, \xi_n) \right\}.
\label{eq:minimization}
\end{equation}
Here, $\xi_n$ denotes the $n-$th data point and its label where $n=1,2,\dots,N$, and $f(\wts, \xi_n)$ denotes the composite loss function. Gradient descent is a way to iteratively minimize this objective function by updating the parameter $\wts$ in the opposite direction of the gradient of $F(\wts)$ at every iteration, as given by: $$\wts_{j+1}=\wts_{j}-\eta \nabla F(\wts_{j}) = \wts_{j} - \frac{\eta}{N} \sum_{n=1}^N \nabla f(\wts_{j}, \xi_n).$$
The computation of $\sum_{n=1}^N \nabla f(\wts_{j}, \xi_n)$ over the entire dataset is expensive. Thus, stochastic gradient descent \cite{robbins1951stochastic} with mini-batching is generally used in practice, where the gradient is evaluated over small, randomly chosen subsets of the data. Smaller mini-batches result in higher variance of the gradients, which affects convergence and error floor \cite{dekel2012optimal, li2014efficient, bottou2016optimization}. Algorithms such as AdaGrad \cite{duchi2011adaptive} and Adam \cite{kingma2015adam} gradually reduce learning rate to achieve a lower error floor. Another class of algorithms includes stochastic variation reduction techniques that include SVRG \cite{johnson2013accelerating}, SAGA \cite{roux2012stochastic} and their variants listed out in \cite{nguyen2017sarah}. For a detailed survey of different SGD variants, refer to \cite{ruder2016overview}.

\noindent
\textbf{Synchronous SGD and Stragglers:} 
To process large datasets, SGD is parallelized across multiple workers with a central PS. Each worker processes one mini-batch, and the PS aggregates all the gradients. The convergence of synchronous SGD is same as mini-batch SGD, with a $P$-fold larger mini-batch, where $P$ is the number of workers. However, the time per iteration grows with the number of workers, because some straggling workers that slow down randomly \cite{dean2013tail}. Thus, it is important to juxtapose the error reduction per iteration with the runtime per iteration to understand the true convergence speed of distributed SGD. 

To deal with stragglers and speed up machine learning, system designers have proposed several straggler mitigation techniques such as \cite{harlap2016addressing} that try to detect and avoid stragglers. An alternate direction of work is to use redundancy techniques, e.g., replication or erasure codes, as proposed in \cite{joshi2014delay,wang2015using,joshi2015queues, joshi2017efficient, lee2017speeding, tandon2017gradient,dutta2016short,halbawi2017improving,yang2017coded,yang2016fault,karakus2017encoded,karakus2017straggler,charles2017approximate,li2017terasort,fahim2017optimal,ye2018communication,li2018fundamental,NewsletterPaper,DNNPaperISIT,mallick2018rateless,dutta2017coded,sheth2018application} to deal with the stragglers, as also discussed in \Cref{rem:redundancy_techniques}. See also \cite{ozfatura2019speeding,al2019anytime,maity2019robust,reisizadeh2019robust,amiri2019computation} for other interesting related works in this direction.

\noindent
\textbf{Asynchronous SGD and Staleness:}
A complementary approach to deal with the issue of straggling is to use asynchronous SGD. In asynchronous SGD, any worker can evaluate the gradient and update the central PS without waiting for the other workers. Asynchronous variants of existing SGD algorithms have also been proposed and implemented in systems \cite{dean2012large, gupta2016model,cipar2013solving, cui2014exploiting,ho2013more}.
In general, analyzing the convergence of asynchronous SGD with the number of iterations is difficult in itself because of the randomness of gradient staleness. There are only a few pioneering works such as \cite{tsitsiklis1986distributed,lian2015asynchronous,mitliagkas2016asynchrony,recht2011hogwild,agarwal2011distributed, mania2017perturbed,chaturapruek2015asynchronous,zhang2016staleness,peng2016arock, hannah2017more, hannah2016unbounded,sun2017asynchronous,leblond2017asaga} in this direction. In \cite{tsitsiklis1986distributed}, a fully decentralized analysis was proposed that considers no central PS. In \cite{recht2011hogwild}, a new asynchronous algorithm called Hogwild was proposed and analyzed under bounded gradient and bounded delay assumptions. This direction of research has been followed upon by several interesting works such as \cite{lian2015asynchronous} which proposed novel theoretical analysis under bounded delay assumption for other asynchronous SGD variants. In \cite{peng2016arock, hannah2017more, hannah2016unbounded,sun2017asynchronous}, the framework of ARock was proposed for parallel co-ordinate descent and analyzed using Lyapunov functions, relaxing several existing assumptions such as bounded delay assumption and the independence of the delays and the index of the blocks being updated. In algorithms such as Hogwild, ARock etc. every worker only updates a part of the central parameter vector $\wts$ at every iteration and are thus essentially different in spirit from conventional asynchronous SGD settings \cite{lian2015asynchronous,agarwal2011distributed} where every worker updates the entire $\bm{\wts}$. In an alternate direction of work \cite{mania2017perturbed}, asynchrony is modelled as a perturbation.

In this work, we present a new and simpler analysis of asynchronous SGD with number of iterations that relaxes some of the assumptions in previous literature, and helps us to characterize the error-runtime trade-off as well as easily derive adaptive update rule for gradually increasing synchrony.
\subsection{Main Contributions}
Existing machine learning algorithms mostly try to optimize the trade-off of error with the number of iterations, epochs or ``work complexity'' \cite{bottou2016optimization}, while assuming the time spent per iteration to be a constant. However, due to straggling and synchronization bottle-necks in the system, the same gradient computation task can often take different time to complete across different workers or iterations~\cite{dean2013tail}. This work departs from the classic optimization theory view of analyzing error convergence with respect to the number of iterations and takes the novel approach of minimizing the error with respect to the wallclock time. By taking a joint runtime and error optimization approach, we provide the first comprehensive runtime-per-iteration comparison of SGD variants and design adaptive synchronous SGD algorithms that can achieve a super-linear runtime speed-up over naive synchronous SGD, while still preserving a low error floor. The main contributions of this paper are summarized below.
\begin{itemize}
\item \textbf{Straggler-Resilient Variants of Synchronous and Asynchronous SGD.} In order to a strike a balance between the two extremes: synchronous and asynchronous SGD, we propose partially synchronous SGD variants such as $K$-sync, $K$-batch-sync, $K$-async and $K$-batch-async SGD, where $K$ is the number of workers (out of the total of $P$ workers) that the parameter waits for when aggregating gradients. Although some of these distributed SGD variants have been proposed previously, to the best of our knowledge, this is the first work to provide a unified error convergence and runtime analysis of these variants.

\item \textbf{Runtime Analysis of the Distributed SGD Variants.} We provide the first systematic analysis of the expected runtime per iteration of synchronous and asynchronous SGD and their variants. We do so by modelling the runtimes at each worker as random variables with an arbitrary general distribution. For commonly used delay distributions such as exponential, asynchronous SGD is $O(P\log P)$ times faster than synchronous SGD where $P$ is the total number of workers. 

\item \textbf{More General Error Analysis of Asynchronous SGD Variants.} We propose a new error convergence analysis for asynchronous SGD and its variants for \emph{strongly convex} objectives that can also be extended to provide relaxed guarantees for \emph{non-convex} formulations. In this analysis we relax the bounded delay assumption in \cite{lian2015asynchronous} and the bounded gradient assumption in \cite{recht2011hogwild}. We also remove the assumption of exponential computation time and the staleness process being independent of the parameter values \cite{mitliagkas2016asynchrony} as we will elaborate in \Cref{sec:main_async_fixed}. Interestingly, our analysis also brings out the regimes where asynchronous SGD can be better or worse than synchronous SGD in terms of speed of convergence. 

\item \textbf{Insights from the Error-versus-wallclock Time Trade-off.} By combining our runtime and error analyses described above, we can theoretically characterize the error-versus-wallclock time trade-off for different SGD variants. \Cref{fig:error runtime tradeoff} illustrates the error at convergence (or error floor) versus the time to reach convergence of different SGD variants. Observe how the $K$-batch-async and $K$-async strategies can span different points on the trade-off as $K$ varies. By choosing the right value of $K$ we can achieve a desired error at convergence in minimum time. The theoretical results presented in this paper are corroborated by rigorous experiments on training deep neural networks for classification of the CIFAR10 dataset. 

\begin{figure}[t]
 \centering
\includegraphics[height=5cm]{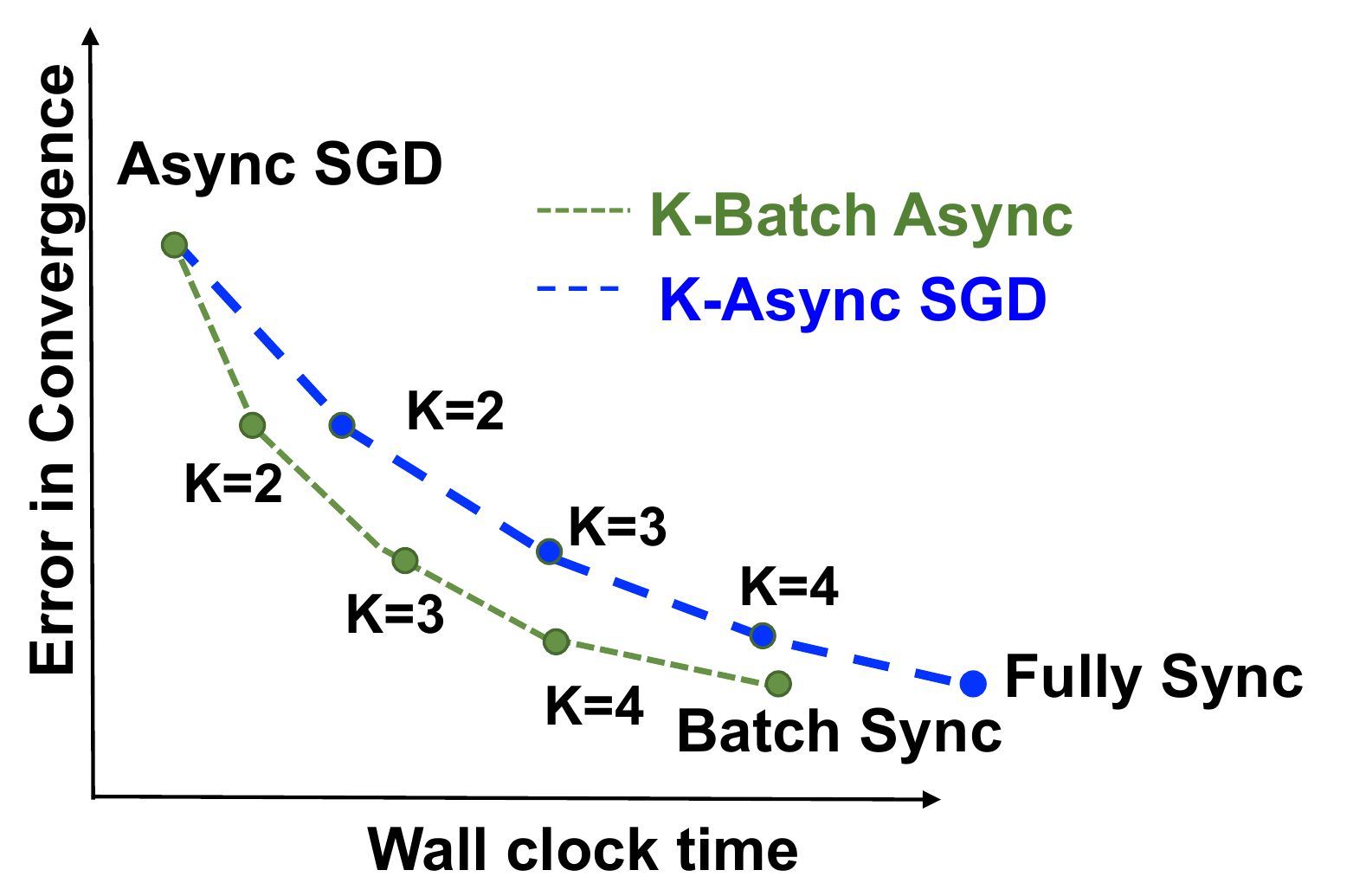}
\caption{Distributed SGD variants span the error-runtime trade-off between fully Sync-SGD and fully Async-SGD. Here $K$ is the number of workers or mini-batches the PS waits for before updating the model parameters, as we elaborate in \Cref{sec:system model}.}
\label{fig:error runtime tradeoff}
\end{figure}

\item \textbf{AdaSync strategy to Adapt Synchronocity during Training.} Instead of fixing $K$, we can achieve a win-win in the error-runtime trade-off by adapting $K$ so as to gradually increasing the synchrony of the different SGD variants. We propose AdaSync, a method that uses the theoretical characterization of the error-runtime trade-off, to decide how to adapt $K$, as illustrated in \Cref{fig:adasync1}. This method is inspired from \cite{wang2018adaptive,wang2018cooperative} which adapts the communication frequency for a different class of SGD methods known as periodic averaging SGD. Interestingly, similar to \cite{wang2018adaptive}, our proposed method does not require knowledge of the algorithm parameters such as Lipschitz constant, variance of the stochastic gradient etc.\ as one would otherwise require if they choose to simply minimize the error-runtime trade-off with respect to parameter $K$. Experimental results on CIFAR 10 classification (see \Cref{fig:cifar_async_var_1}) show that AdaSync not only helps achieve the same training loss much faster but also gives smaller test error than fixed-$K$ strategies.

\end{itemize}

\begin{figure}
\begin{subfigure}{.45\textwidth}
 \centering
\includegraphics[height=5.3cm]{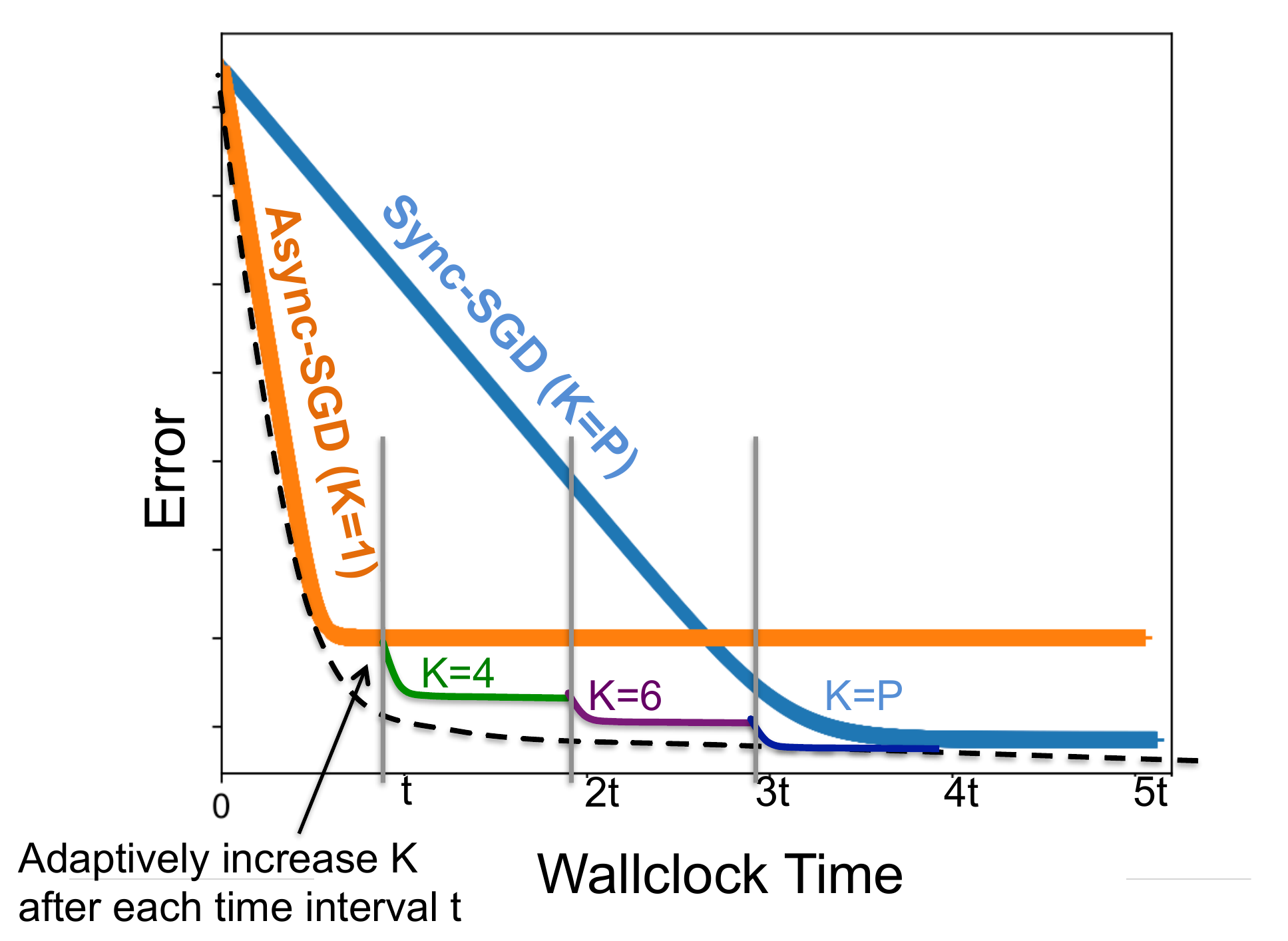}
\caption{Idea behind the AdaSync strategy.}
\end{subfigure}
~
 \begin{subfigure}{.5\textwidth}
 \centering
 \includegraphics[height=4.8cm]{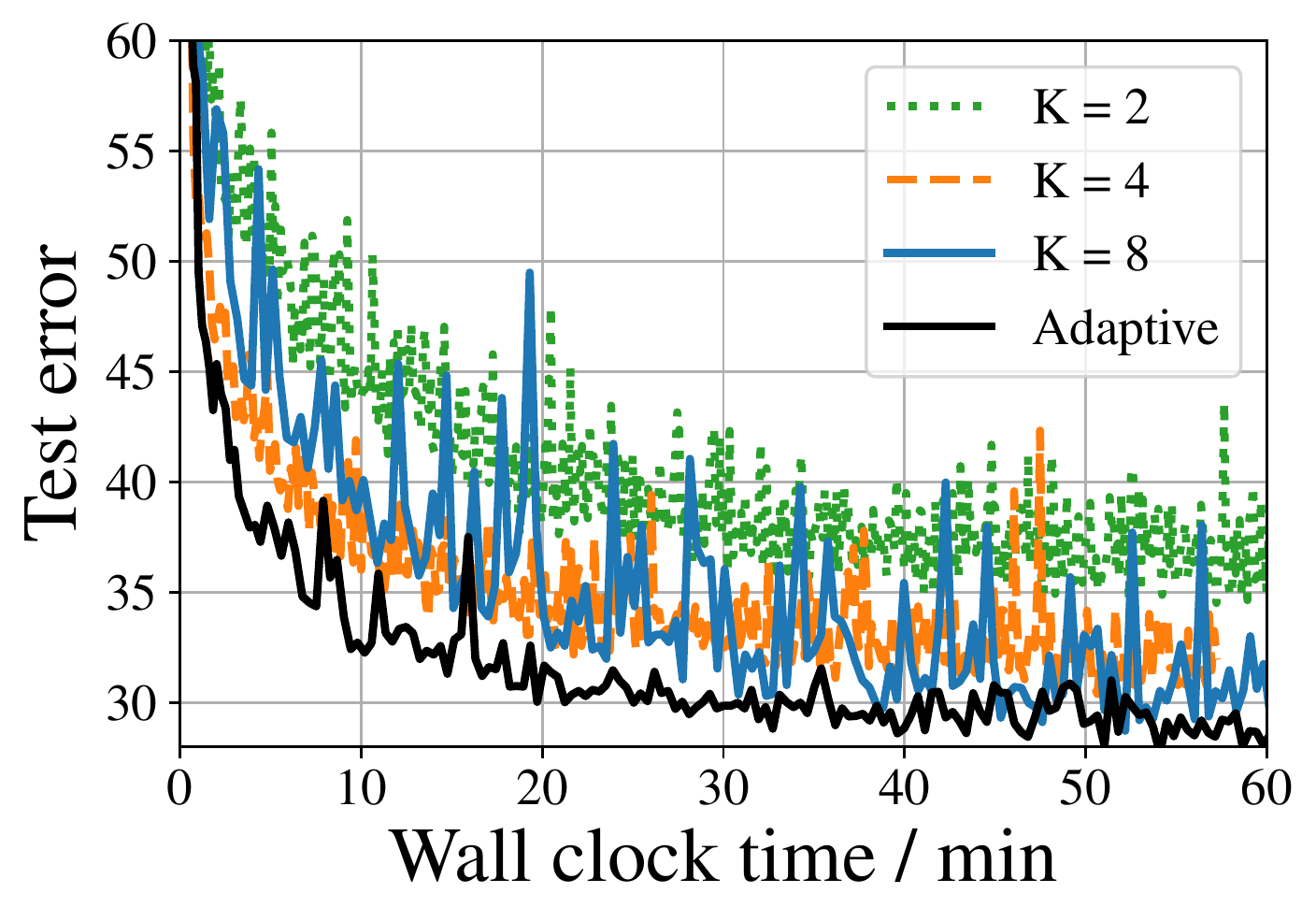}
 \caption{Adaptive $K$-async for CIFAR 10 classification.}
\label{fig:cifar_async_var_1}
\end{subfigure}%
    
\caption{We propose AdaSync, a method of adaptively increasing $K$ which is a measure of the synchronicity of the algorithm. AdaSync aggregates gradients from any $K$ out the $P$ total nodes. It helps achieve the best error-runtime trade-off by gradually increasing $K$. 
}
\label{fig:adasync1}
\end{figure}

The rest of the paper is organized as follows.  \Cref{sec:system model} describes our problem formulation introducing the system model and assumptions. \Cref{sec:runtime} provides the theoretical results on the analysis of true wallclock runtime per iteration for the different SGD variants and provides insights on quantifying the speedups that one variant provides over another. In \Cref{sec:main_async_fixed}, we discuss our analysis of error convergence with number of iterations where we also include our new convergence analysis for asynchronous and $K$-async SGD. Proofs and detailed discussions are presented in the Appendix. In \Cref{sec:error_runtime}, we combine the runtime analysis with the error analysis to derive novel error-runtime trade-offs and demonstrate how our analysis could inform predicting the trend of the trade-off for distributed systems with different runtime distributions. Finally, in \Cref{sec:vary_synchronicity}, we introduce our proposed method AdaSync that gradually varies synchronicity (parameter $K$) to achieve the desirable error-runtime trade-off, followed by experimental results in \Cref{subsec:adasync_experiments}. We conclude with a brief discussion in \Cref{sec:conclusion}.

\section{Problem Formulation}
\label{sec:system model}
Our objective is to minimize the risk function of the parameter vector $\wts$ as mentioned in \cref{eq:minimization} given $N$ training samples. Let $S$ denote the total set of $N$ training samples, \textit{i.e.}, a collection of some data points with their corresponding labels or values. We use the notation $\xi$ to denote a random seed $ \in S$ which consists of either a single data and its label or a single mini-batch ($\mb$ samples) of data and their labels. 
\subsection{System Model}
We assume that there is a central parameter server (PS) with $P$ parallel workers as shown in \Cref{fig:param_server}. The workers fetch the current parameter vector $\wts_j$ from the PS as and when instructed in the algorithm. Then they compute gradients using one mini-batch and push their gradients back to the PS as and when instructed in the algorithm. At each iteration, the PS aggregates the gradients computed by the workers and updates the parameter $\wts$. Based on how these gradients are fetched and aggregated, we have different variants of synchronous or asynchronous SGD.

\subsection{Variants of SGD}
We now describe the SGD variants considered in this paper. We note that some of these variants have been proposed earlier under alternate names in different papers, as we will refer to during our descriptions. In this work, we give a unified runtime and error analysis to compare them with each other in terms of their true error-runtime trade-off, a problem that has not been considered in prior works. Please refer to \Cref{fig:ksync} and \Cref{fig:kasync} for a pictorial illustration of the SGD variants.

\textbf{$K$-sync SGD:} This is a generalized form of synchronous SGD, also suggested in \cite{gupta2016model,chen2016revisiting} to offer some resilience to straggling as the PS does not wait for all the workers to finish. The PS only waits for the first $K$ out of $P$ workers to push their gradients. Once it receives $K$ gradients, it updates $ \wts_{j}$ and cancels the remaining workers. The updated parameter vector $\wts_{j+1}$ is sent to all $P$ workers for the next iteration. The update rule is given by:
\begin{equation}
\wts_{j+1} = \wts_j - \frac{\eta}{K}\sum_{l=1}^K g(\wts_{j}, \xi_{l,j}).
\label{eq:ksync}
\end{equation}
Here $l=1,2,\ldots,K$ denotes the index of the $K$ workers that finish first, $\xi_{l,j}$ denotes the mini-batch of $m$ samples used by the $l$-th worker at the $j$-th iteration and  $g(\wts_{j}, \xi_{l,j})= \frac{1}{m} \sum_{\xi \in \xi_{l,j} }  \nabla f(\wts_{j}, \xi)$ denotes the average gradient of the loss function evaluated over the mini-batch  $\xi_{l,j}$ of size $m$. For $K=P$, the algorithm is exactly equivalent to a fully synchronous SGD with $P$ workers. 

\textbf{$K$-batch-sync SGD:} In $K$-batch-sync, all the $P$ workers start computing gradients with the same $\wts_j$. Whenever any worker finishes, it pushes its update to the PS and evaluates the gradient on the next mini-batch at the same $\wts_j$. The PS updates using the first $K$ mini-batches that finish and cancels the remaining workers. Theoretically, the update rule is still the same as \cref{eq:ksync} but here $l$ now denotes the index of the mini-batch (out of the $K$ mini-batches that finished first) instead of the worker. However $K$-batch-sync will offer advantages over $K$-sync in runtime per iteration as no worker is idle. 
 \begin{figure}[t]
\centerline{\includegraphics[height=3cm]{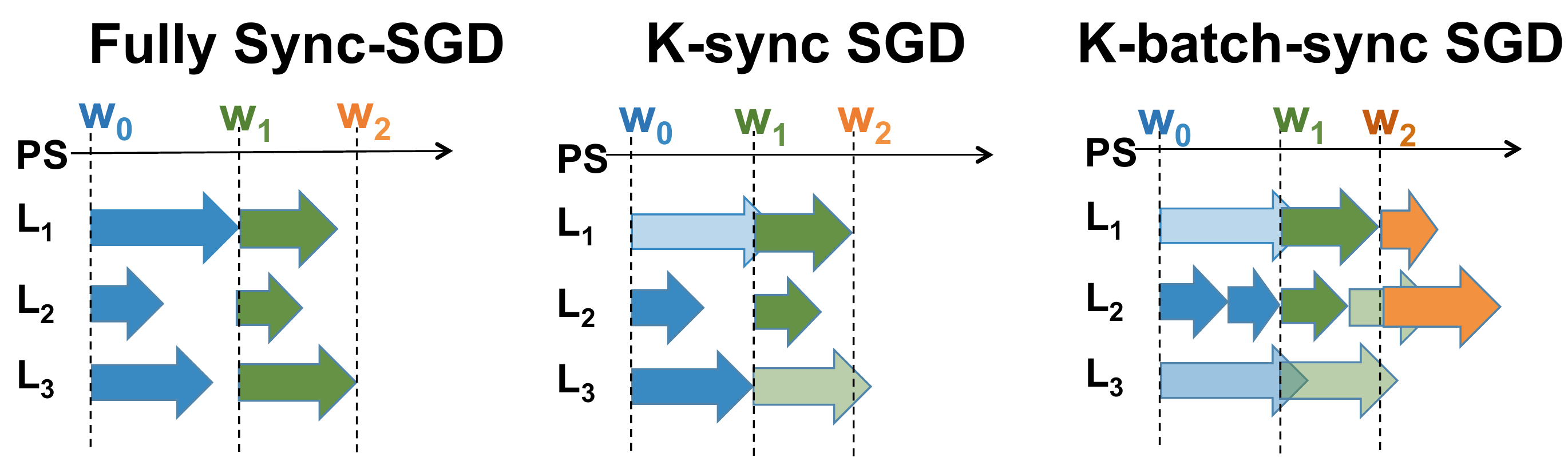}}
\caption{For $K=2$ and $P=3$, we illustrate the $K$-sync and $K$-batch-sync SGD in comparison with fully synchronous SGD. Lightly shaded arrows indicate straggling gradient computations that are cancelled.}
\label{fig:ksync}
\end{figure}
\begin{figure}[t]
\centerline{\includegraphics[height=3cm]{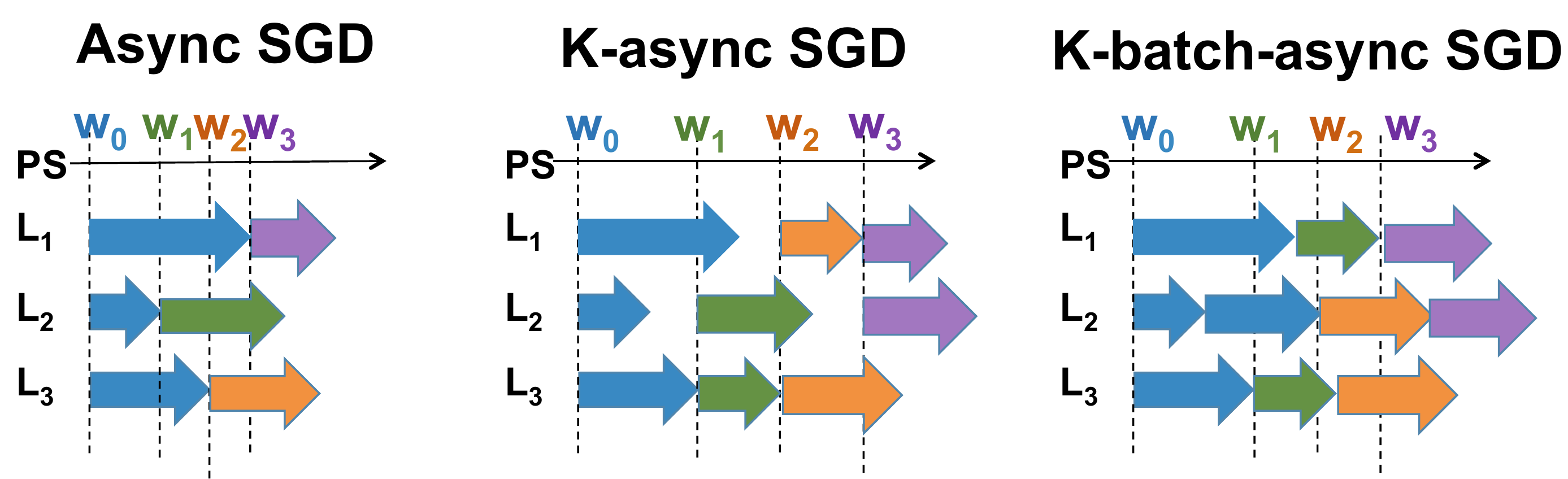}}
\caption{For $K=2$ and $P=3$, we illustrate the $K$-async and $K$-batch-async algorithms in comparison with fully asynchronous SGD.}
\label{fig:kasync}
\end{figure}

\textbf{$K$-async SGD:}
This is a generalized version of asynchronous SGD, also suggested in \cite{gupta2016model}. In $K$-async SGD, all the $P$ workers compute their respective gradients on a single mini-batch. The PS waits for the first $K$ out of $P$ that finish first, but it does not cancel the remaining workers. As a result, for every update the gradients returned by each worker might be computed at a stale or older value of the parameter $\wts$. The update rule is thus given by:
\begin{equation}
\wts_{j+1} = \wts_j - \frac{\eta}{K} \sum_{l=1}^K g(\wts_{\tau(l,j)}, \xi_{l,j}).
\label{eq:kasync}
\end{equation}
Here $l=1,2,\ldots,K$ denotes the index of the $K$ workers that contribute to the update at the corresponding iteration,  $\xi_{l,j}$ is one mini-batch of $m$ samples used by the $l$-th worker at the $j$-th iteration and $\tau(l,j)$ denotes the iteration index when the $l$-th worker last read from the central PS where $\tau(l,j) \leq j $.  Also, $g(\wts_{\tau(l,j)}, \xi_{l,j})= \frac{1}{m} \sum_{\xi \in \xi_{l,j} }  \nabla f(\wts_{\tau(l,j)}, \xi_{l,j})$ is the average gradient of the loss function evaluated over the mini-batch $\xi_{l,j}$ based on the stale value of the parameter $\wts_{\tau(l,j)}$.  For $K=1$, the algorithm is exactly equivalent to fully asynchronous SGD, and the update rule can be simplified as:
\begin{equation}
\wts_{j+1} = \wts_j - \eta g(\wts_{\tau(j)}, \xi_{j}).
\end{equation}
Here $\xi_{j}$ denotes the set of samples used by the worker that updates at the $j$-th iteration such that $|\xi_{j}| = m$ and $\tau(l,j)$ denotes the iteration index when that particular worker last read from the central PS. Note that $\tau(j) \leq j $.  

\textbf{$K$-batch-async SGD:} Observe in \Cref{fig:kasync} that $K$-async also suffers from some workers being idle while others are still working on their gradients until any $K$ finish. In $K$-batch-async (proposed in \cite{lian2015asynchronous}), the PS waits for $K$ mini-batches before updating itself but irrespective of which worker they come from.  So wherever any worker finishes, it pushes its gradient to the PS, fetches current parameter at PS and starts computing gradient on the next mini-batch based on the current value of the PS. Surprisingly, the update rule is again similar to \cref{eq:kasync} theoretically except that now $l$ denotes the indices of the $K$ mini-batches that finish first instead of the workers and $\wts_{\tau(l,j)}$ denotes the version of the parameter when the worker computing the $l-$th mini-batch last read from the PS.  While the error convergence of $K$-batch-async is similar to $K$-async, it reduces the runtime per iteration as no worker is idle.

\begin{rem} 
\label{rem:redundancy_techniques}
Recent works such as \cite{tandon2017gradient} propose erasure coding techniques to overcome straggling workers. Instead, the SGD variants considered in this paper such as $K$-sync and $K$-batch-sync SGD exploit the inherent redundancy in the data itself, and ignore the gradients returned by straggling workers. If the data is well-shuffled such that it can be assumed to be i.i.d.\ across workers, then for the same effective batch-size, ignoring straggling gradients will give equivalent error scaling as coded strategies, and at a lower computing cost. However, coding strategies may be useful in the non i.i.d.\ case, when the gradients supplied by each worker provide diverse information that is important to capture in the trained model.
\end{rem}
\subsection{Performance Metrics and Goal}
There are two metrics of interest: \textcolor{black}{Expected Runtime and Error.}
\begin{defn}[Expected Runtime per iteration]
The expected runtime per iteration is the expected time (average time) taken to perform each iteration, \textit{i.e.}, the expected time between two consecutive updates of the parameter $\wts$ at the central PS.
\end{defn}

\begin{defn}[Expected Error after $J$ iterations]
The expected error after $J$ iterations is defined as $\E{F(\wts_J) - F^*}$, \textit{i.e.}, the expected gap of the risk function from its optimal value.
\end{defn}

\noindent \textbf{Goal:} Our goal is to determine the trade-off between the expected error (measures the accuracy of the algorithm) and the expected runtime \textbf{after a total of $J$ iterations} for the different SGD variants. Based on our characterization, we would like to derive a method of gradually varying synchronicity (parameter $K$) in the SGD variants to achieve a desirable error-runtime trade-off.

In \Cref{sample-table}, we provide a list of the notations used in this paper for referencing, that we again revisit in the respective sections where they appear.

\begin{table}
\begin{center}
\begin{tabular}{llll}
CONSTANTS &  & RANDOM VARIABLES\\
\hline 
Mini-batch Size         & $m$ & Runtime of a worker for one mini-batch & $X_i$\\
Total Iterations            & $J$  & Runtime per iteration & $T$\\
Number of workers (Processors) & $P$ \\
Number of workers to wait for & $K$ \\
Learning rate & $\eta$ \\
Lipschitz Constant & $L$ \\
Strong-convexity parameter & $c$\\
\hline
\end{tabular}
\end{center}
\caption{LIST OF NOTATIONS}
 \label{sample-table}
\end{table}

\section{Runtime Analysis: Insights on Quantifying Speedup}
\label{sec:runtime}
Our runtime analysis provides useful insights in quantifying the speedup offered by different SGD variants. We first state our key modeling assumptions in \Cref{subsec:runtime_assumption}, followed by our main theoretical results on quantifying speedup in \Cref{subsec:runtime_main_results}. Next, we include our detailed runtime analysis for the four SGD variants considered in this paper in \Cref{subsec:runtime_lemmas}, some of which are useful in the proofs of the main results on speedups. For a summary of the expected runtime of the different SGD variants, we refer to \Cref{table:runtime_variants}.

\begin{table}[!htbp]
\begin{center}
\caption{Expected Runtime for the different variants of SGD}
\begin{tabular}{llll}
SGD Variant &  Expected Runtime per iteration $\E{T}$ \\
\hline 
$K$-sync & $\E{T}=\E{X_{K:P}}$ for all distributions\\
$K$-batch-sync & $\E{T}\leq K\E{X_{1:P}}$ for new-longer-than-used distributions\\
$K$-async & $\E{T}\leq \E{X_{K:P}}$ for new-longer-than-used distributions\\
$K$-batch-async  & $\E{T}=\frac{K}{P}\E{X}$ for all distributions \\
\hline
\label{table:runtime_variants}
\end{tabular}
\end{center}
\end{table}

\subsection{Modeling Assumptions:}
\label{subsec:runtime_assumption} The time taken by a worker to compute gradient of one mini-batch is denoted by random variable $X_i$ for $i=1,2,\dots,P$. We assume that these $X_i$'s are i.i.d.\ across mini-batches and workers.

\subsection{Main Results on Quantifying Speedups}
\label{subsec:runtime_main_results}
 \begin{figure}[t]
\centerline{\includegraphics[height=4.5cm]{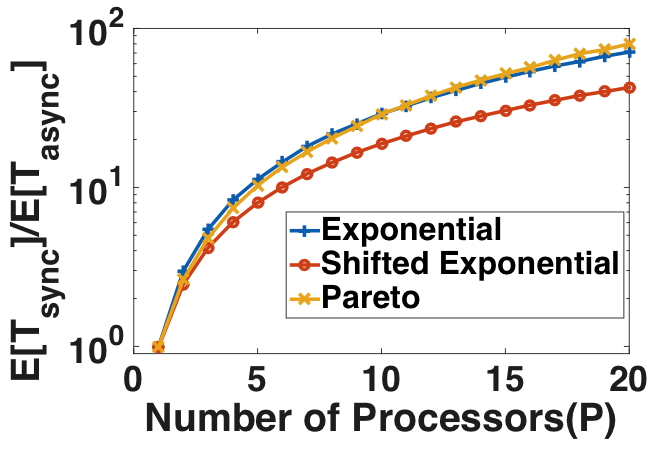}}
\caption{Simulation demonstrating the speedup using asynchronous over synchronous SGD for three different types of distributions, namely, $ \Exp(1)$, $1+ \Exp(1)$ and $Pareto(2,1)$. The speedup is shown in log-scale, \textit{i.e.}, $\log{\frac{\E{T_{Sync}}}{\E{T_{Async}}}}$ is plotted with total number of nodes $P$.}
\label{fig:syncAsync}
\end{figure}

 Our first result (\Cref{thm:runtime1}) analytically characterizes the speedup offered by asynchronous SGD for \textit{any general distribution on the wallclock time of each worker}. 

\begin{thm} Let the wallclock time of each worker to process a single mini-batch be i.i.d.\ random variables $X_1, X_2,\dots,X_P$. Then the ratio of the expected runtimes per iteration for synchronous and asynchronous SGD is
$$ \frac{\E{T_{Sync}}}{\E{T_{Async}}}=P \frac{\E{X_{P:P}}}{\E{X}}$$
where $X_{(P:\workers)}$ is the $P^{th}$ order statistic of $P$ i.i.d.\ random variables $X_1, X_2, \dots , X_{\workers}$. 
\label{thm:runtime1}
\end{thm}

\begin{proof}[\textbf{Proof of \Cref{thm:runtime1}}] Note that fully synchronous SGD is actually $K$-sync SGD with $K=P$, \textit{i.e.}, waiting for all the $P$ workers to finish. On the other hand, fully asynchronous SGD is actually $K$-batch-async with $K=1$. By taking the ratio of the expected runtimes per iteration for $K$-sync SGD (see \Cref{lem:runtime ksync} in \Cref{subsec:runtime_lemmas}) with $K=P$ and $K$-batch-async (see \Cref{lem:runtime kbatch} in \Cref{subsec:runtime_lemmas}) with $K=1$, we get the result in \Cref{thm:runtime1}. 
\end{proof}
In the following corollary, we highlight this speedup for the special case of exponential computation time.
\begin{coro}
\label{coro:syncAsync}
Let the wallclock time of each worker to process a single mini-batch be i.i.d.\ exponential random variables $X_1, X_2,\dots,X_P \sim \Exp({\mu})$. Then the ratio of the expected runtimes per iteration for synchronous and asynchronous SGD is $\Theta(P \log{P})$.
\end{coro}
Thus, the speedup scales with $P$ and can diverge to infinity for large $P$. We illustrate the speedup for different distributions in \Cref{fig:syncAsync}. It might be noted that a similar speedup as \Cref{coro:syncAsync} has also been obtained in a recent work \cite{hannah2017more} under exponential assumptions.

\begin{proof}[\textbf{Proof of \Cref{coro:syncAsync}}] The expectation of the maximum of $P$ i.i.d.\ $X_i \sim \Exp(\mu)$ is $\E{X_{P:P}}=\sum_{i=1}^P \frac{1}{i\mu} \approx \frac{\log{P}}{\mu}$ \cite{sheldon2002first}. This can be substituted in \Cref{thm:runtime1} to get \Cref{coro:syncAsync}.
\end{proof}

The next result illustrates the advantages offered by $K$-batch-sync and $K$-batch-async over their corresponding counterparts $K$-sync and $K$-async respectively.
\begin{thm}
\label{thm:runtime2}
Let the wallclock time of each worker to process a single mini-batch be i.i.d.\  exponential random variables $X_1, X_2,\dots,X_P \sim \Exp({\mu})$. Then the ratio of the expected runtimes per iteration for $K$-async (or sync) SGD and $K$-batch-async (or sync) SGD is 
$$ \frac{\E{T_{K-async}}}{\E{T_{K-batch-async}}} =\frac{P\E{X_{K:P} } }{K\E{X} }  \approx \frac{P \log{\frac{P}{P-K} } }{K}  $$
where $X_{K:P}$ is the $K^{th}$ order statistic of i.i.d.\ random variables $X_1,X_2,\dots,X_P$.
\end{thm}

\begin{proof}[\textbf{Proof of \Cref{thm:runtime2}}] For the exponential $X_i$, equality holds in \eqref{eqn:T_k_async} in \Cref{lem:runtime kasync}, as we justify in \Cref{subsec:runtime_K_async_exp}.  The expectation can be derived as $\E{X_{K:P}}=\sum_{i=P-K+1}^P \frac{1}{i\mu} \approx \frac{\log{(P/P-K)}}{\mu}$. For exponential $X_i$, the expected runtime per iteration for $K$-batch-async is given by $\E{T}  =  \frac{K\E{X}}{P} =  \frac{K}{\mu P }$ from \Cref{lem:runtime kbatch}.
\end{proof}

\Cref{thm:runtime2} shows that as $\frac{K}{P}$ increases, the speedup using $K$-batch-async increases and can be upto $\log{P}$ times higher. For non-exponential distributions, we simulate the behaviour of expected runtime in \Cref{fig:runtime2} for $K$-sync, $K$-async and $K$-batch-async respectively for Pareto and Shifted Exponential.

\begin{figure}[t]
\centering
\includegraphics[width=13cm]{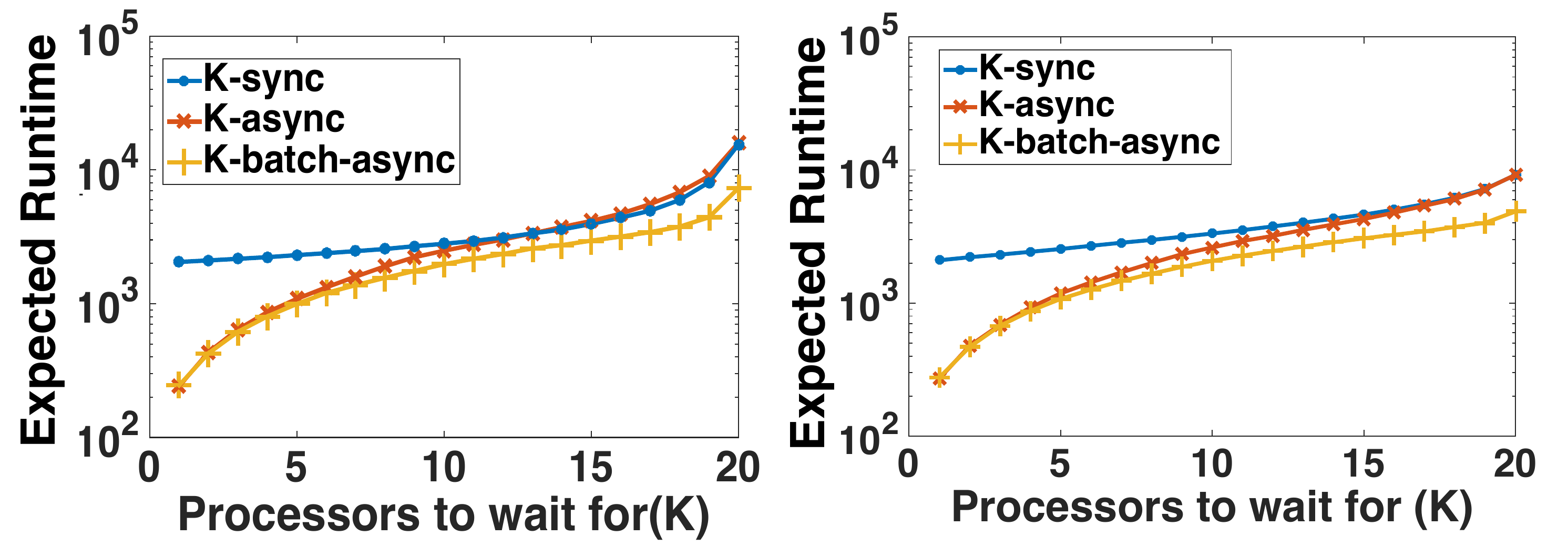}
\caption{Simulation of expected runtime for $2000$ iterations for $K$-sync, $K$-async and $K$-batch-async SGD: (Left) Pareto distribution $Pareto(2,1)$ and (Right) Shifted exponential distribution $1+ \Exp(1)$.}
\label{fig:runtime2}
\end{figure}

\subsection{Runtime Analysis for four different SGD variants}
\label{subsec:runtime_lemmas}

Here, we rigorously analyze the theoretical wallclock runtime of the four different SGD variants. These results have also been used in providing insights on the speedup offered by different asynchronous and batch variants in \Cref{thm:runtime1} and \Cref{thm:runtime2}.

\subsubsection{Runtime of K-sync SGD}
\begin{lem}[Runtime of $K$-sync SGD]
\label{lem:runtime ksync}
The expected runtime per iteration for $K$-sync SGD is,
\begin{align}
\E{T} &=  \E{X_{K:\workers}} 
\end{align}
where $X_{K:\workers}$ is the $K^{th}$ order statistic of $P$ i.i.d.\ random variables $X_1, X_2, \dots , X_{\workers}$.
\end{lem}

\begin{proof}[Proof of \Cref{lem:runtime ksync}]
When all the workers start together, and we wait for the first $K$ out of $P$ i.i.d.\ random variables to finish, the expected computation time for that iteration is $\E{X_{K:P}}$, where $X_{K:P}$ denotes the $K$-th statistic of $P$ i.i.d.\ random variables $X_1,X_2,\dots,X_P$. 
\end{proof}
Thus, for a total of $J$ iterations, the expected runtime is given by $J\E{X_{K:P}}$.

\begin{rem} For  $X_i \sim \Exp(\mu)$, the expected runtime per iteration is given by, $$\E{T} =\frac{1}{\mu} \sum_{i=P-K+1}^P \frac{1}{i} \approx \frac{1}{\mu} \left( \frac{ \log{\frac{P}{P-K}}}{\mu} \right)
$$ where the last step uses an approximation from \cite{sheldon2002first}. For justification, the reader is referred to \Cref{subsec:runtime_K_statistic}.
\end{rem}

\subsubsection{Runtime of K-batch-sync SGD}
The expected runtime of $K$-batch-sync SGD is not analytically tractable in general, but we obtain an upper bound on it for a class of distributions called the ``new-longer-than-used'' distributions, as defined below.

\begin{defn}[New-longer-than-used]
\label{defn:new-longer-than-used}
A random variable is said to have a new-longer-than-used distribution if the following holds for all $t, u \geq 0$:
\begin{equation}
\Pr(U>u+t| U>t) \leq \Pr(U>u).
\label{eq:defn_new_longer}
\end{equation}
\end{defn}
Most of the continuous distributions we encounter like normal, shifted-exponential, gamma, beta are new-longer-than-used. Alternately, the hyper-exponential distribution is new-shorter-than-used and it satisfies 
\begin{equation}
\Pr(U>u+t | U>t) \geq \Pr(U>u) \ \ \forall t, u \geq 0.
\end{equation}
For the exponential distribution, the inequality \Cref{eq:defn_new_longer} holds with equality due to the memoryless property, and thus it can be thought of as both new-longer-than-used and new-shorter-than-used.

\begin{lem}[Runtime of $K$-batch-sync SGD]
\label{lem:runtime kbatch sync}
Suppose that each $X_i$ has a new-longer-than-used distribution. Then, the expected runtime per iteration for $K$-batch-sync is upper-bounded as
\begin{align}
\E{T} & \leq  K \E{X_{1:\workers}} \label{eqn:T_k_batch_sync}
\end{align}
where $X_{1:\workers}$ is the minimum of $P$ i.i.d.\ random variables $X_1, X_2, \dots , X_{\workers}$. 
\end{lem}

The proof is provided in \Cref{subsec:runtime_K_batch_sync}. For the special case of $X_i\sim \Exp(\mu)$, the runtime per iteration is distributed as $Erlang(K, P\mu)$ (see \Cref{subsec:runtime_K_batch_sync}). Thus, for $K$-batch-sync SGD, the expected time per iteration is given by, $$\E{T}=  \frac{K}{P\mu},$$
which is precisely what we obtain when \Cref{eqn:T_k_batch_sync} holds with equality.

\subsubsection{Runtime of K-Async SGD}
The expected runtime per iteration of $K$-async SGD is also not analytically tractable for non-exponential $X_i$, but we again obtain an upper bound on it for ``new-longer-than-used'' distributions.

\begin{lem}[Runtime of $K$-async SGD]
\label{lem:runtime kasync}
Suppose that each $X_i$ has a new-longer-than-used distribution. Then, the expected \textcolor{black}{runtime per iteration} for $K$-async is upper-bounded as
\begin{align}
\E{T} & \leq   \E{X_{K:\workers}} \label{eqn:T_k_async}
\end{align}
where $X_{K:\workers}$ is the $K^{th}$ order statistic of $P$ i.i.d.\ random variables $X_1, X_2, \dots , X_{\workers}$. 
\end{lem}

The proof of this lemma is provided in \Cref{subsec:runtime_K_async}.

\begin{rem} Recall that the runtime of $K$-sync SGD is $\E{X_{K:\workers}}$. Therefore, \Cref{lem:runtime kasync} essentially implies that for new-longer-than-used distributions, the runtime of $K$-async SGD is upper-bounded by the runtime of $K$-sync SGD. For the special case of exponential runtimes, \Cref{eq:defn_new_longer} holds with equality, and the expected runtime of both $K$-sync SGD and $K$-async SGD match theoretically. However, for other classes of distributions where  \Cref{eq:defn_new_longer} holds with strict inequality, the upper bound of 
\Cref{lem:runtime kasync} also holds with strict inequality. In general, intuitively the ``more'' is the new-longer-than-used property, the lower is the runtime of $K$-async SGD as compared to $K$-sync SGD. We show this by explicitly deriving an alternate upper bound for the shifted-exponential distribution (a special case of the new-longer-than-used distributions) that is lower than the expected runtime of $K$-sync SGD when the shift of the distribution is large.
\end{rem}

\begin{lem}[Alternate Upper Bound for Shifted Exponential]
\label{lem:tight kasync}
Let $P=nK$ where $n$ is an integer greater than $1$, and each $X_i$ follow a shifted exponential distribution with shift $\Delta$, \textit{i.e.}, $X_i \sim \Delta + \Exp(\mu)$. Then, the expected runtime for any $n$ consecutive iterations of $K$-async SGD is upper-bounded as
\begin{equation}
\E{T_1+T_2+\ldots+T_n} \leq \Delta + \sum_{i=0}^{n-1} \E{\widetilde{X}_{K:(P-iK)}},
\end{equation}
where $\widetilde{X} \sim \Exp(\mu)$ and $\widetilde{X}_{K:(P-iK)}$ denotes the $K^{th}$ order statistic out of $P-iK$ i.i.d.\ random variables.
\end{lem}
The proof of this lemma is also provided in \Cref{subsec:runtime_K_async}. Based on this lemma, the upper-bound on per-iteration runtime can be approximated as:
\begin{align}
&\frac{1}{n}\left( \Delta + \sum_{i=0}^{n-1} \E{\widetilde{X}_{K:(P-iK)}}\right) \nonumber \\
& \approx \frac{\Delta}{n} + \frac{1}{n\mu} \left( \log{\frac{P}{P-K}} + \log{\frac{P-K}{P-2K}} + \ldots + \log{\frac{P-(n-2)K}{P-(n-1)K}} \right) + \frac{1}{n\mu} \log{(K)} \nonumber \\
&\approx \frac{\Delta}{n} + \frac{ \log{P}}{n\mu} \nonumber \\
&= \frac{K\Delta}{P} + \frac{K \log{P}}{P\mu}.  
 \end{align}

In comparison, the runtime per iteration for $K$-sync SGD is $\left(\Delta + \frac{\log{P/(P-K)} }{\mu}   \right)$. Thus, a high value of $\Delta$ implies that the runtime of $K$-async SGD is strictly lower than $K$-sync SGD.

\subsubsection{Runtime of K-batch-async SGD}

For this variant, we derive an expression that holds for any general distribution on $X_i$.
\begin{lem}[Runtime of $K$-batch-async SGD]
\label{lem:runtime kbatch}
The expected runtime per iteration for $K$-batch-async SGD in the limit of large number of iterations is given by:
\begin{equation}
\E{T} = \frac{K\E{X}}{P}.
\end{equation}
\end{lem}
Unlike the results for the synchronous variants, this result on average runtime per iteration holds only in the limit of large number of iterations. To prove the result we use ideas from renewal theory. For a brief background on renewal theory, the reader is referred to \Cref{subsec:runtime_K_batch_async}. 

\begin{proof}[Proof of \Cref{lem:runtime kbatch}]
For the $i$-th worker, let $\{N_i(t),\ t>0\}$ be the number of times the $i$-th worker pushes its gradient to the PS over in time $t$. The time between two pushes is an independent realization of $X_i$. Thus, the inter-arrival times $X_i^{(1)}, X_i^{(2)},\dots$ are i.i.d.\ with mean inter-arrival time $\E{X_i}$. Using the elementary renewal theorem \cite[Chapter 5]{gallager2013stochastic} we have,
\begin{equation}
\lim_{t \to \infty} \frac{\E{N_i(t)}}{t} = \frac{1}{\E{X_i}}.
\end{equation}
Thus, the rate of gradient pushes by the $i$-th worker is $1/\E{X_i}$. As there are $P$ workers, we have a superposition of $P$ renewal processes and thus the average rate of gradient pushes to the PS is 
\begin{equation}
\lim_{t \to \infty} \sum_{i=1}^P\frac{\E{N_i(t)}}{t} = \sum_{i=1}^P \frac{1}{\E{X_i}} = \frac{P}{\E{X}}.
\end{equation}

Every $K$ pushes are one iteration. Thus, the expected runtime per iteration or effectively the expected time for $K$ pushes is given by
$
 \E{T}= \frac{K\E{X}}{P}.
$ \end{proof}

\section{Error Analysis: New Convergence Analysis for Asynchronous SGD}
\label{sec:main_async_fixed}

In this section, we discuss our analysis of error convergence with the number of iterations. We first state some assumptions on the objective function in \Cref{subsec:error_assumption}, followed by the main theoretical results on the error analysis including our novel analysis for asynchronous SGD and its variants in \Cref{subsec:error_main_results}.

\subsection{Assumptions on the Objective Function}
\label{subsec:error_assumption}

Closely following \cite{bottou2016optimization}, we make the following assumptions:
\begin{enumerate}[leftmargin=*]
\item $F(\wts)$ is an $\lips-$ smooth function. Thus,
\begin{equation}
||\nabla F(\wts)- \nabla F(\widetilde{\wts})||_2 \leq \lips ||\wts -\widetilde{\wts} ||_2 \ \ \forall \ \wts, \widetilde{\wts}.
\end{equation}  
\item $F(\wts)$ is strongly convex with parameter $c$. Thus,
\begin{equation}
\label{eq:strong-convexity}
2c(F(\wts)-F^* ) \leq ||\nabla F(\wts)||_2^2 \ \ \forall \  \wts.
\end{equation}
Refer to \Cref{sec:strong_convexity} for discussion on strong convexity. Our results also extend to non-convex objectives, as discussed in \Cref{thm:non_convex}.
\item The stochastic gradient is an unbiased estimate of the true gradient:
\begin{equation}
\Esub{\xi_{j}|
\wts_k}{g(\wts_k,\xi_{j})}= \nabla F(\wts_k) \ \ \forall \  k \leq j.
\end{equation}
Observe that this is slightly different from the common assumption that says $\Esub{\xi_{j}}
{g(\wts,\xi_{j})}= \nabla F(\wts)$ for all $\wts$.
Observe that all $\wts_j$ for $j>k$ is actually not independent of the data $\xi_{j}$.  We thus make the assumption more rigorous by conditioning on $\wts_k$ for $k \leq j$. Our requirement $k \leq j$ means that $\wts_k$ is the value of the parameter at the PS before the data $\xi_{j}$ was accessed and can thus be assumed to be independent of the data $\xi_{j}$.

\item Similar to the previous assumption, we also assume that the variance of the stochastic update given $\wts_k$ at iteration $k$ before the data point was accessed is also bounded as follows:
\begin{align}
&\Esub{\xi_{j}|\wts_k}{||g(\wts_k,\xi_{j})- \nabla F(\wts_k) ||_2^2} 
\leq \frac{\sigma^2}{m} + \frac{M_G}{m} ||\nabla F(\wts_k) ||_2^2 \ \forall \ k \leq j.
\end{align}

\item In the analysis of $K$-async and $K$-batch-async SGD, we replace some assumptions in existing literature that we discuss in \Cref{subsec:error_main_results}, and instead use an alternate staleness bound that allows for large, but rare delays. We assume that for some $\gamma\leq 1$,
\begin{equation}
\E{|| \nabla F(\wts_{j}) - \nabla F(\wts_{\tau(l,j)})||_2^2 } \leq \gamma \E{|| \nabla F(\wts_{j}) ||_2^2 }.
\end{equation}

\end{enumerate}

\subsection{Main Theoretical Results}
\label{subsec:error_main_results}
In this work, we provide a novel convergence analysis of $K$-async SGD for fixed $\eta$, relaxing the following assumptions in existing literature.
\begin{itemize}[leftmargin=*]
\item In several prior works such as \cite{mitliagkas2016asynchrony,lee2017speeding,dutta2016short,hannah2017more}, it is often assumed, for the ease of analysis, that runtimes are exponentially distributed. In this paper, we extend our analysis for any general service time $X_i$.
\item In \cite{mitliagkas2016asynchrony}, it is also assumed that the staleness process is independent of $\wts$. While this assumption simplifies the analysis greatly, it is not true in practice. For instance, for a two worker case, the parameter $\wts_2$ after $2$ iterations depends on whether the update from $\wts_1$ to $\wts_2$ was based on a stale gradient at $\wts_0$ or the current gradient at $\wts_1$, depending on which worker finished first. In this work, we remove this independence assumption.
\item Instead of the bounded delay assumption in \cite{lian2015asynchronous}, we use a general staleness bound $$\E{|| \nabla F(\wts_{j}) - \nabla F(\wts_{\tau(l,j)})||_2^2 } \leq \gamma \E{|| \nabla F(\wts_{j}) ||_2^2 }$$ which allows for large, but rare delays.
\item In \cite{recht2011hogwild}, the norm of the gradient is assumed to be bounded. However, if we assume that $||\nabla F(\wts) ||_2^2 \leq M$ for some constant $M$, then using \cref{eq:strong-convexity} we obtain $ ||\wts-\wts^*||_2^2 \leq \frac{2}{c} (F(\wts)-F^*) \leq \frac{M}{c^2} $ implying that $\wts$ itself is bounded which is a very strong and restrictive assumption, that we relax in this result.
\end{itemize}

\noindent Some of these assumptions have been addressed in the context of alternative asynchronous SGD variants in the recent works of \cite{hannah2017more,hannah2016unbounded,sun2017asynchronous,leblond2017asaga}.

\subsubsection{Convex Loss Function}
\begin{thm}
\label{thm:error kasync}
Suppose the objective $F(\wts)$ is $c$-strongly convex and the learning rate $\eta \leq \frac{1}{2\lips\left(\frac{M_G}{Km}+ \frac{1}{K}\right)}$. Also assume that for some $\gamma \leq 1$, $$\E{|| \nabla F(\wts_{j}) - \nabla F(\wts_{\tau(l,j)})||_2^2 } \leq \gamma \E{|| \nabla F(\wts_{j}) ||_2^2 }.$$ Then, the error of $K$-async SGD after $J$ iterations is,
\begin{align}
& \E{F(\wts_{J})}-F^* 
\leq \frac{\eta L\sigma^2}{2c\gamma'K m} 
+ (1 - \eta c\gamma')^{J} \left(F(\wts_{0})-F^* - \frac{\eta L\sigma^2}{2c \gamma' Km}  \right)
\end{align}
where $\gamma'= 1-\gamma + \frac{p_0}{2}$ and $p_0$ is a lower bound on the conditional probability that $\tau(l,j)=j$, given all the past delays and parameters.
\end{thm}

Here, $\gamma$ is a measure of staleness of the gradients returned by workers; smaller $\gamma$ indicates less staleness. 
The full proof is provided in \Cref{sec:async_proof}. We first prove the result for $K=1$ in \Cref{subsec:async_proof} for ease of understanding, and then provide the more general proof for any $K$ in \Cref{subsec:K_async_proof}.

\Cref{lem:p_0} below provides bounds on $p_0$ for different classes of distributions.

\begin{lem}[Bounds on $p_0$]
\label{lem:p_0}
Define $p_0=\inf_{j}p_0^{(j)}$.
Then the following holds:
\begin{itemize}
\item For exponential computation times, $p_0^{(j)} = \frac{1}{P}$ for all $j$ (invariant of $j$) and $p_0=\frac{1}{P}$.
\item For new-longer-than-used (See \Cref{defn:new-longer-than-used}) computation times, $p_0^{(j)} \leq \frac{1}{P} $ and thus $p_0 \leq \frac{1}{P} $.
\item For new-shorter-than-used computation times, $p_0^{(j)} \geq \frac{1}{P} $ and thus $p_0 \geq \frac{1}{P} $.
\end{itemize}
\end{lem}

The proof is provided in \Cref{subsec:proof_lem_p_0}.

\begin{rem}For $K$-batch-async, the update rule is same as $K$-async except that the index $l$ denotes the index of the mini-batch. Thus, the error analysis will be exactly similar.
\end{rem}

Now let us compare with $K$-sync SGD. We observe that the analysis of $K$-sync SGD is same as serial SGD with mini-batch size $Km$. Thus,
\begin{lem}[Error of $K$-sync
 \cite{bottou2016optimization}]
Suppose that the objective $F(\wts)$ is $c$-strongly convex and learning rate $\eta \leq \frac{1}{2\lips(\frac{M_G}{Km}+ 1)}$. Then, the error after $J$ iterations of $K$-sync SGD is
\begin{align*}
\E{F(\wts_{\iters}) - F^*}&\leq \frac{ \eta \lips \sigma^2}{2 c (K\mb)} +  (1- \eta c )^{J} \left( F(\wts_0) - F^* - \frac{ \eta \lips \sigma^2}{2 c (K\mb)} \right).
\end{align*}
\label{lem:error ksync}
\end{lem}

\textit{Can stale gradients win the race?}
For the same $\eta$, observe that the error given by \Cref{thm:error kasync} decays at the rate $(1 - \eta c(1-\gamma + \frac{p_0}{2}))$ for $K$-async or $K$-batch-async SGD while for $K$-sync, the decay rate with number of iterations is $(1 - \eta c)$. Thus, depending on the values of $\gamma$ and $p_0$, the decay rate of $K$-async or $K$-batch-async SGD can be faster or slower than $K$-sync SGD. The decay rate of $K$-async or $K$-batch-async SGD is faster if $\frac{p_0}{2}>\gamma $. As an example, one might consider an exponential or new-shorter-than-used service time where $p_0 \geq \frac{1}{P}$ and $\gamma$ can be made smaller by increasing $K$. It might be noted that asynchronous SGD can still be faster than synchronous SGD with respect to wallclock time even if its decay rate with respect to number of iterations is lower as every iteration is much faster in asynchronous SGD (Roughly $P\log{P}$ times faster for exponential service times).

The maximum allowable learning rate for synchronous SGD is $\max\{ \frac{1}{c}, \frac{1}{2L (\frac{M_G}{Pm}+1)}  \}$ which can be much higher than that for asynchronous SGD,\textit{i.e.}, $\max\{ \frac{1}{c(1-\gamma + \frac{p_0}{2})}, \frac{1}{2L (\frac{M_G}{m}+1)}  \}$. Similarly the error floor for synchronous is $ \frac{\eta L\sigma^2}{2c Pm} $ as compared to asynchronous whose error floor is $ \frac{\eta L\sigma^2}{2c(1-\gamma + \frac{p_0}{2})m}$.

\subsubsection{Extension to Non-Convex Loss Function}
\label{subsec:non_convex}
The analysis can be extended to provide weaker guarantees for non-convex objectives. Let $\gamma'=1-\gamma + \frac{p_0}{2} $. For non-convex objectives, we have the following result.

\begin{thm} For non-convex objective function $F(\cdot)$, where $F^*=\min_{\wts} F(\wts)$, we have the following ergodic convergence result for $K$-async SGD:
\begin{equation}
\frac{1}{J} \sum_{j=0}^{J-1} \E{|| \nabla F(\wts_j)   ||_2^2} \leq \frac{2(F(\wts_0)-F^*)}{J \eta \gamma' } +  \frac{L\eta\sigma^2}{Km\gamma'}.
\label{eq:non_convex}
\end{equation}
\label{thm:non_convex}
\end{thm}

The proof is provided in \Cref{subsec:nonconvex}. As before, the same analysis also holds for $K$-batch-async SGD. For $K$-sync and $K$-batch-sync, we can also obtain a similar result, substituting $\gamma'=1$ in \Cref{eq:non_convex} (see \cite{bottou2016optimization}). Next, we combine our runtime analysis with the error analysis to characterize the error-runtime trade-off.

\section{Experiments and Insights on the Error-Runtime Trade-off}
\label{sec:error_runtime}

We can combine our expressions for runtime per iteration with the error convergence per iteration to derive the error-runtime trade-off. 
In \Cref{fig:overview}(b), we compare the theoretical trade-offs between synchronous ($K=P$ in \Cref{lem:error ksync} and \Cref{lem:runtime ksync}) and asynchronous SGD ($K=1$ in \Cref{thm:error kasync} and \Cref{lem:runtime kasync}) under the strongly-convex assumption. Asynchronous SGD converges very quickly, but to a higher floor. On the other hand, synchronous SGD converges slowly with respect to time, but reaches a much lower error-floor. To validate the trend observed in theory, we conduct experiments on training neural networks to perform image classfication on CIFAR 10 dataset. These experiments give insights on how choosing the right $K$ helps us strike the best error-runtime trade-off, depending upon the distribution of the gradient computation delays.

\subsection{Experimental Setting}
\label{sec:experiments}
The algorithms discussed in this paper are implemented in Pytorch (v1.0) using multiple nodes. Ray (v0.7) is used for supporting the distributed execution. 
We use the {CIFAR-10} \cite{krizhevsky2009learning} dataset. This dataset consists of $60,000$ $32\times32$ color images in $10$ classes. There are $50,000$ training images and $10,000$ validation images. The neural network used to classify this dataset has two convolutional layers and three fully connected layers. Experiments were conducted on a local cluster with 8 worker machines, each of which has an NVIDIA TitanX GPU. Machines are connected via a 40 Gbps (5000 Mb/s) Ethernet interface. Mini-batch size per worker machine is $32$ and learning rate is $\eta =0.12$.

\subsection{Speedup in Runtime}
In \Cref{fig:cifar_runtime}(a) and \Cref{fig:cifar_runtime}(b) we compare the average runtime per epoch of $K$-sync and $K$-async SGD for different values of $K$. When $K=P=8$, both the variants become identical to fully synchronous SGD. As we decrease $K$ from $P=8$ to $1$, the computation time drastically reduces since we do not have to wait for straggling nodes. $K$-async SGD gives a larger delay reduction than $K$-sync SGD because we do not cancel partially completed gradient computation tasks. Each plot shows three cases: 1) with no artificial delays added to induce straggling, 2) with an additional exponential delay with mean $0.02$sec, and 3) with an additional exponential delay with mean $0.05$sec. The purpose of these curves is to demonstrate how variability in gradient computation time affects the runtime per iteration. Higher variability means that the system is more susceptible to straggling workers. Thus, as the delay variability increases (as we add a higher mean exponential delay per worker), setting a smaller $K$ gives sharper delay reduction as compared to the $K=8$ (fully synchronous SGD) case.

\begin{figure}[t]
    \centering
    \begin{subfigure}{.4\textwidth}
    \centering
    \includegraphics[width=\textwidth]{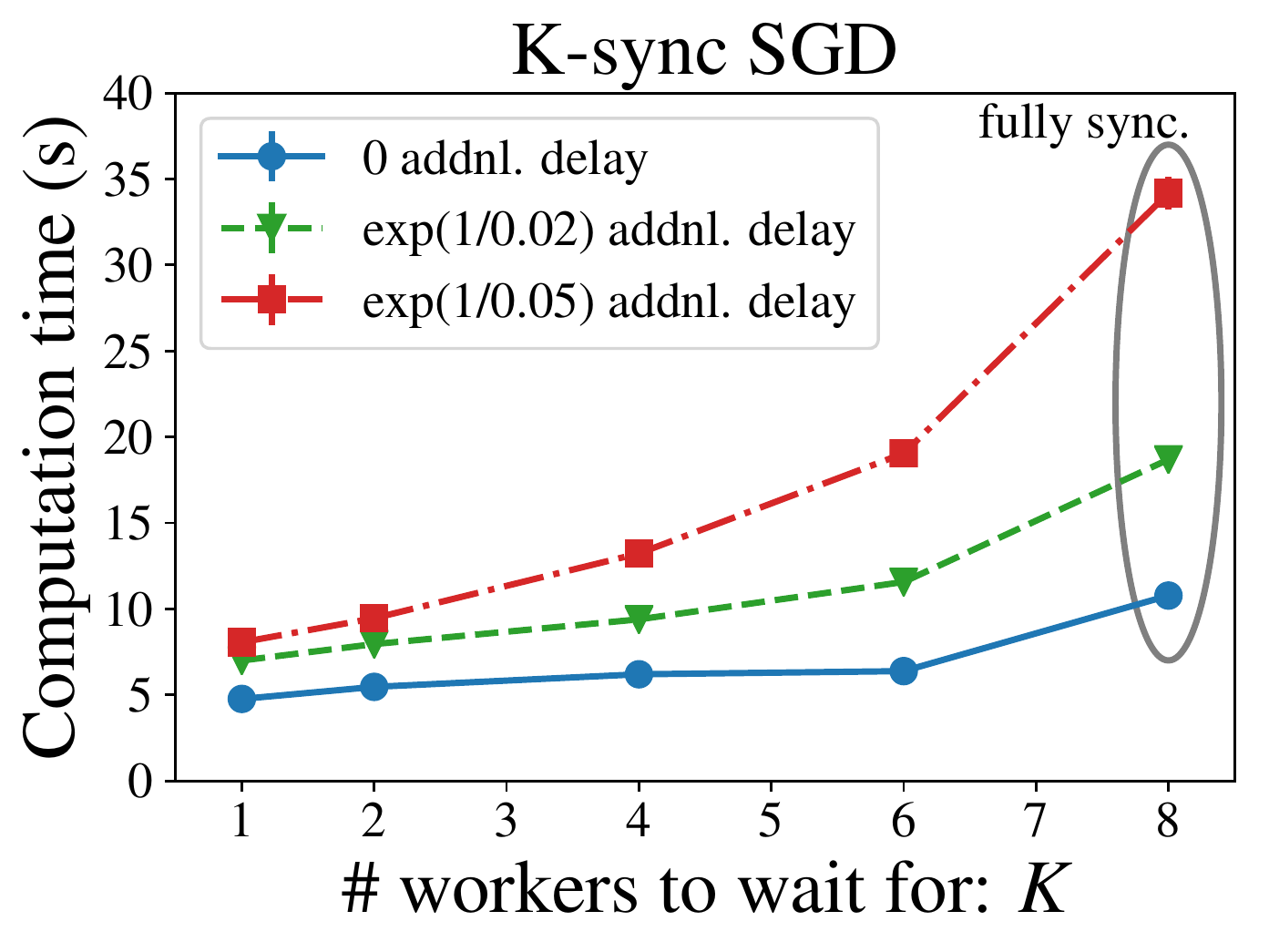}
    \caption{Runtime-versus-$K$ in K-sync SGD.}
    \end{subfigure}%
    ~
    \begin{subfigure}{.4\textwidth}
    \centering
    \includegraphics[width=\textwidth]{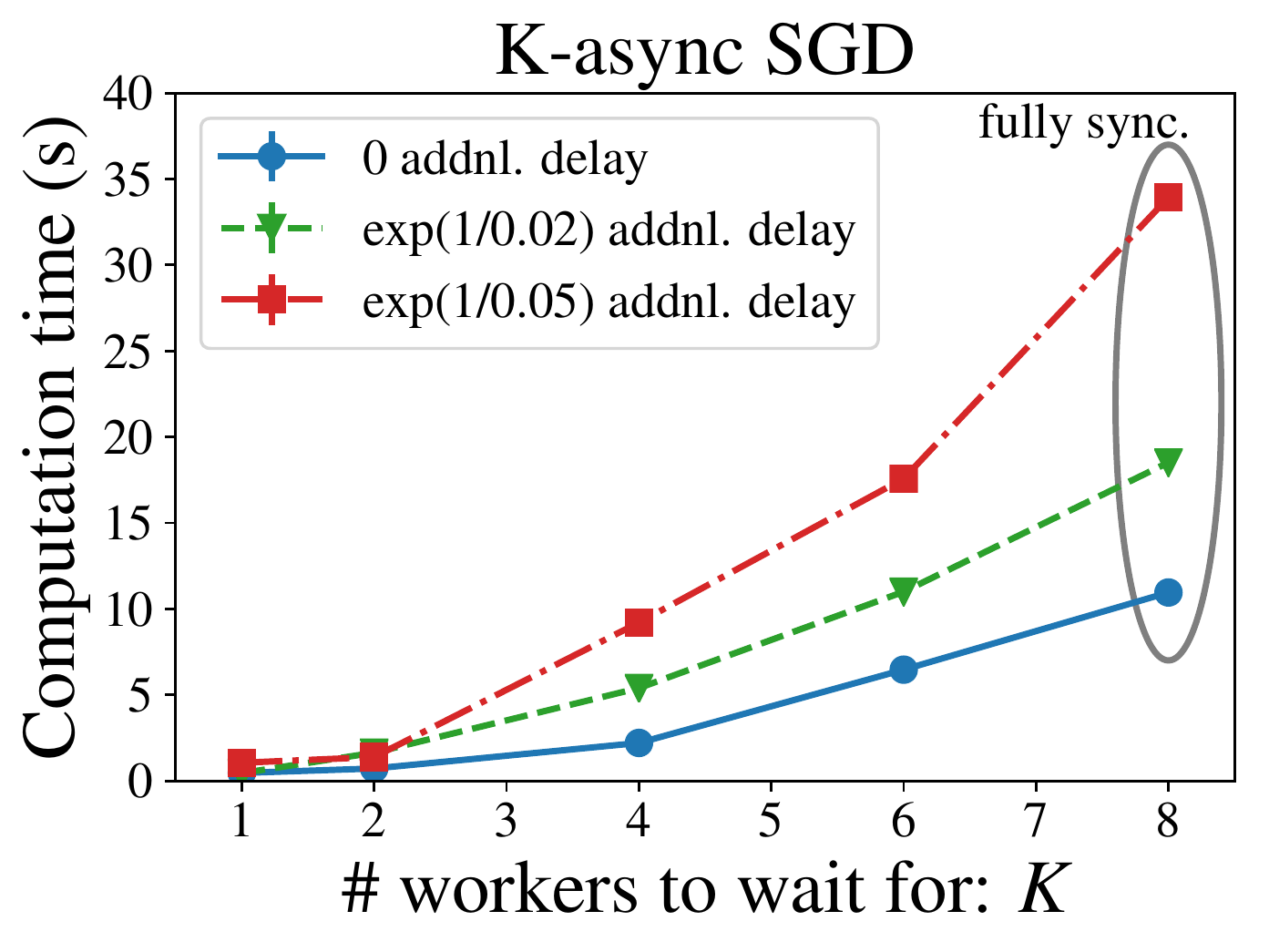}
    \caption{Runtime-versus-$K$ in K-async SGD.}
    \end{subfigure}%
    \caption{Runtime per iteration changes along with the parameter $K$ for CIFAR10 dataset: K-async SGD always has lower runtime than $K$-sync SGD as one might expect from our theoretical analysis. 
    \label{fig:cifar_runtime}}
\end{figure}

\subsection{Accuracy-Runtime Trade-off in $K$-sync and $K$-async SGD}

In this subsection, we examine the accuracy-runtime trade-offs for $K$-sync and $K$-async SGD variants for CIFAR10 dataset. We plot the test error against wallclock runtime for three cases: (a) No artificial delay; (b) Exponential delay with mean 0.02sec; and (c) Exponential delay with mean 0.05sec. $K$-sync SGD in \Cref{fig:cifar_ksync} shows the test error for $K$-sync SGD and \Cref{fig:cifar_kasync} shows $K$-async SGD. For brevity, we present test error plots here and include the training loss plots in \Cref{app:additional_experimental_results}. The training loss follows the same trend as test error.

As predicted from our theoretical analysis, increasing $K$ improves the final error floor for all the SGD variants. But increasing $K$ also increases the runtime per iteration. Hence, when the error is plotted against wallclock run-time, we begin to observe interesting trends in the error-runtime trade-offs -- the highest $K$ does not always achieve the best error-runtime trade-off. An intermediate value of $K$ often achieves a better trade-off than fully-asynchronous or fully-synchronous SGD. For example, in \Cref{fig:cifar_sync_0}, since there is little delay variability in the gradient computation time, $K=8$ (fully synchronous SGD) is the best choice of $K$, but as the delay variability increases in \Cref{fig:cifar_sync_2} and \Cref{fig:cifar_sync_5}, $K=4$ becomes the case that gives the fastest error-versus-wallclock time convergence. A similar trend can be observed in \Cref{fig:cifar_kasync} for $K$-async SGD.

\begin{rem}
Note that the problem is choosing $K$ is similar in spirit to that of selecting the best mini-batch size in standard synchronous SGD. The main difference is that we consider the error-runtime trade-off instead of the error-iterations trade-off when making the choice. 
\end{rem}

So far we have considered partially synchronous/asynchronous SGD variants with a fixed value of $K$. In \Cref{sec:vary_synchronicity} we propose the AdaSync strategy that gradually adapts $K$ during the training process in order to get the best error-runtime trade-off.

\begin{figure}[thbp]
    \centering
    \begin{subfigure}{.33\textwidth}
    \centering
    \includegraphics[width=\textwidth]{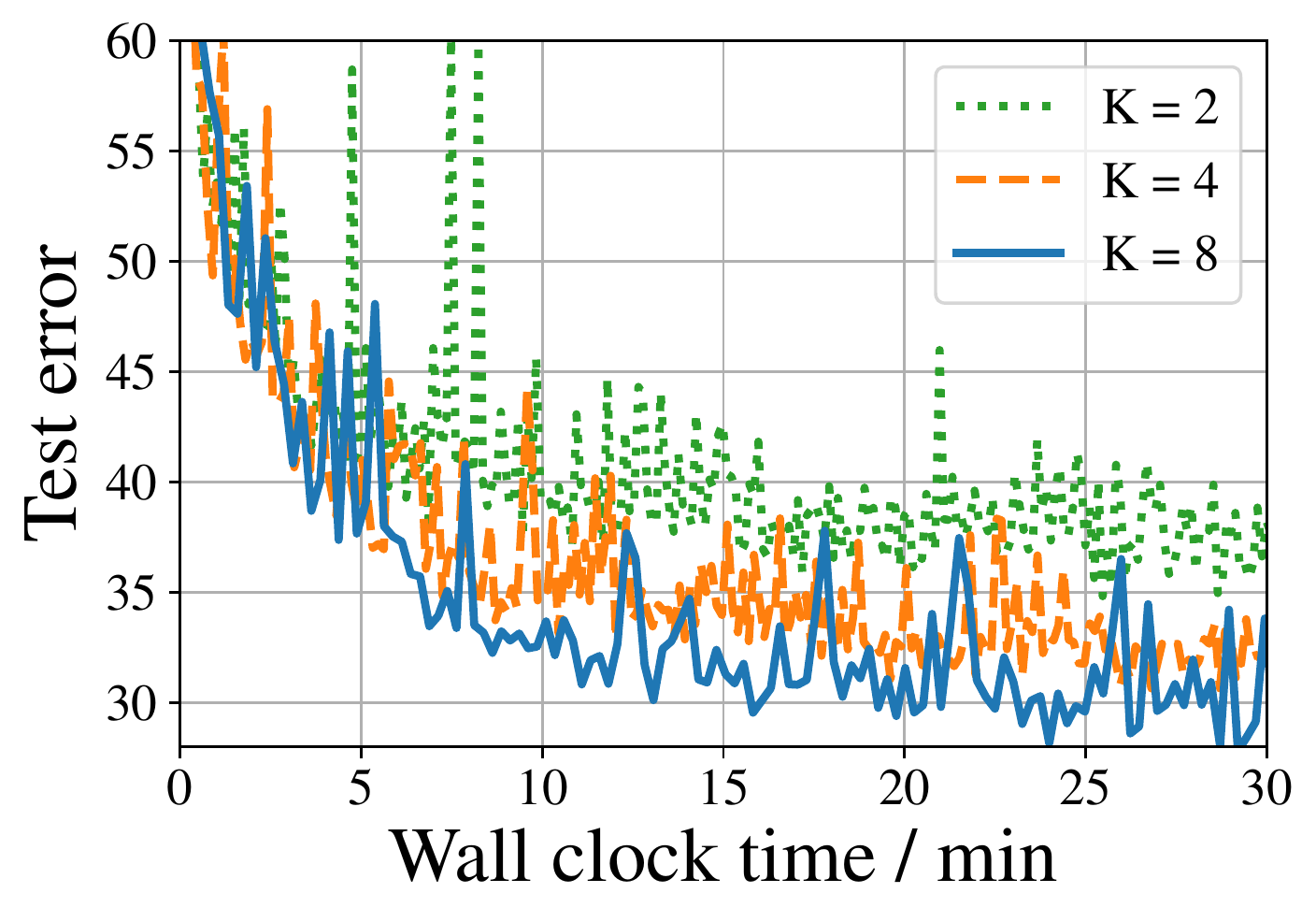}
    \caption{No artificial delay.}
    \label{fig:cifar_sync_0}
    \end{subfigure}%
    ~
    \begin{subfigure}{.33\textwidth}
    \centering
    \includegraphics[width=\textwidth]{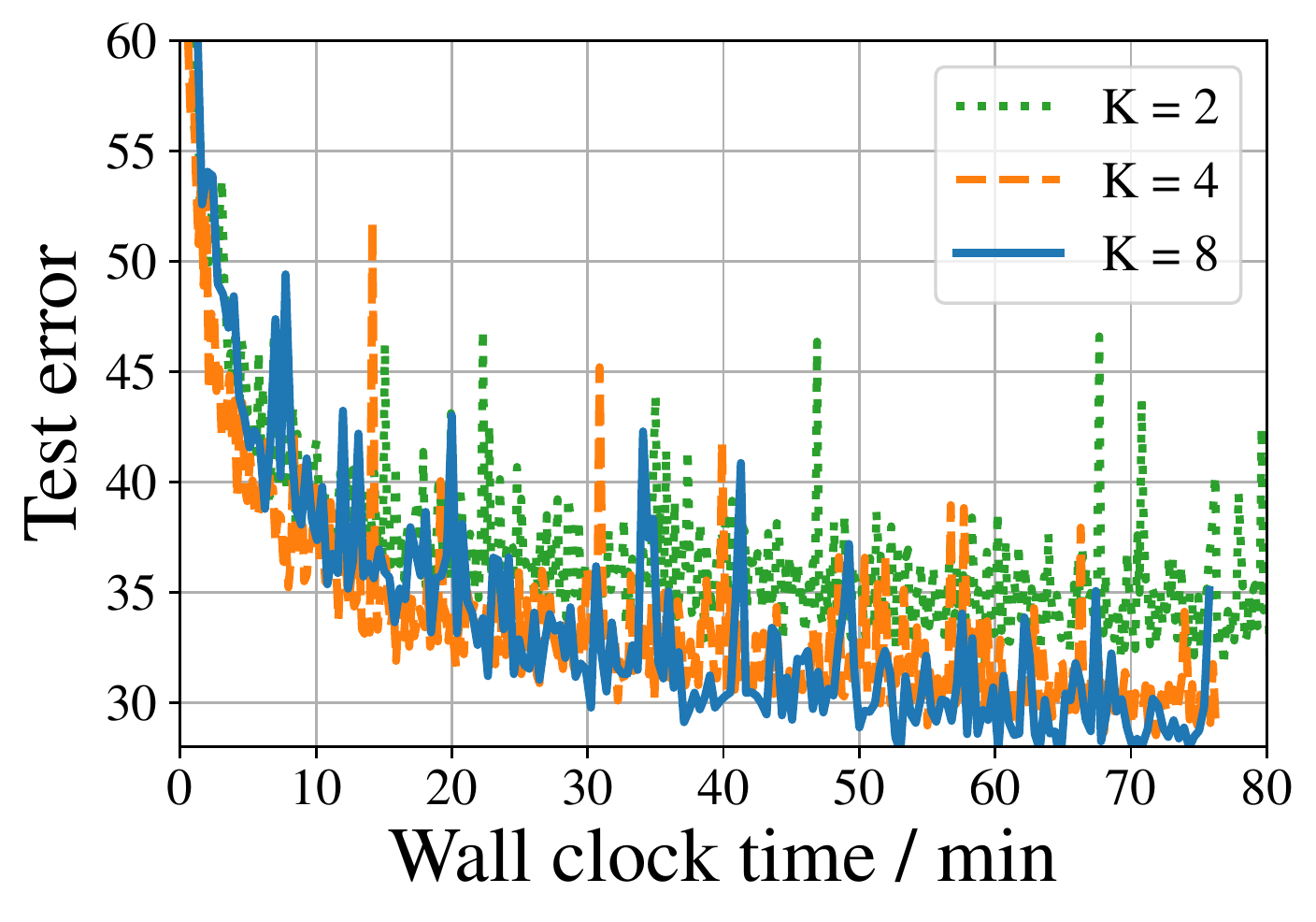}
    \caption{Delay with mean 0.02s.}
    \label{fig:cifar_sync_2}
    \end{subfigure}%
    ~
    \begin{subfigure}{.33\textwidth}
    \centering
    \includegraphics[width=\textwidth]{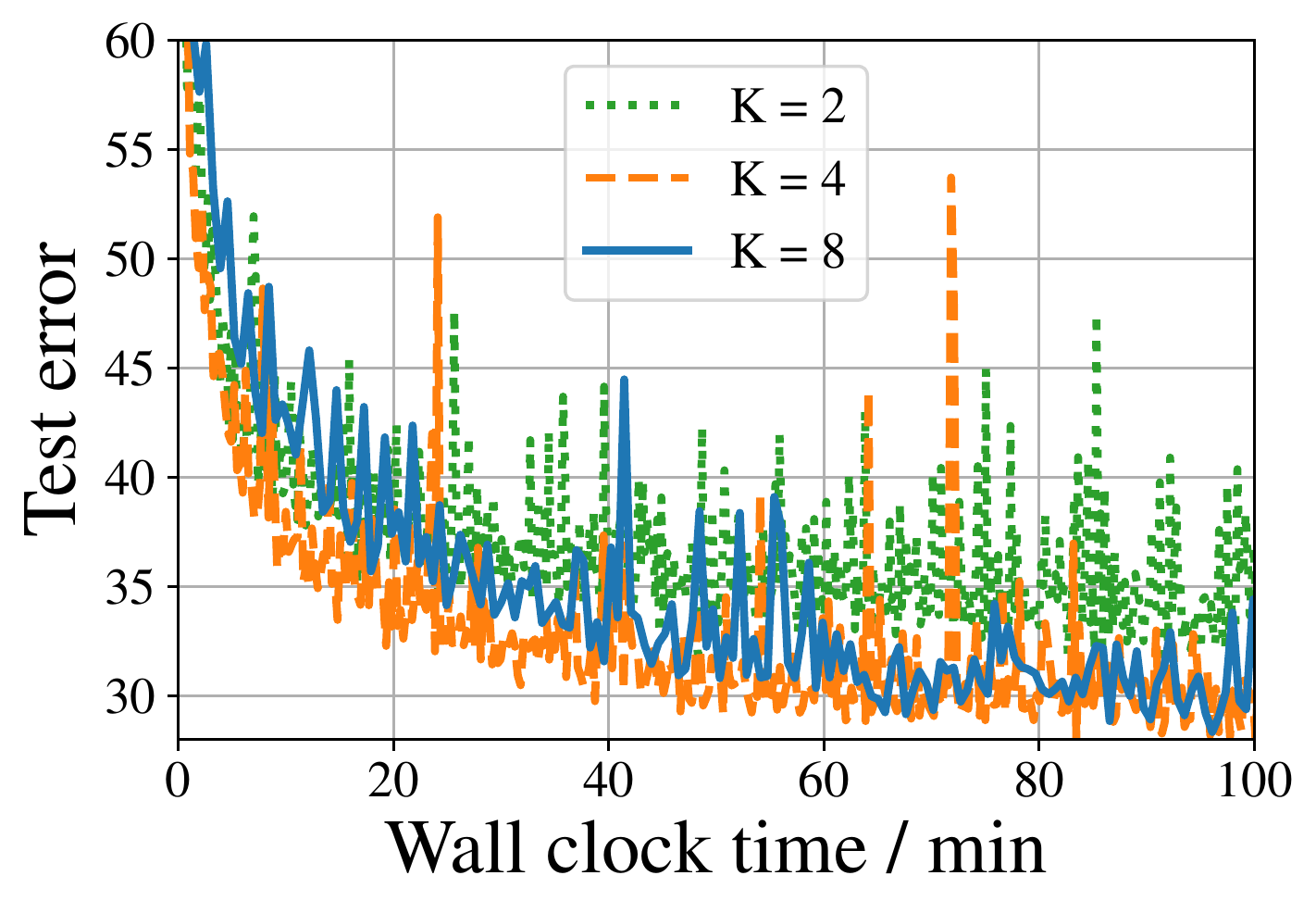}
    \caption{Delay with mean 0.05s.}
    \label{fig:cifar_sync_5}
    \end{subfigure}%
    \caption{Test error ($\%$) of \textbf{K-sync SGD} on CIFAR-10 with $8$ worker nodes. We now demonstrate the error-runtime trade-off for the case with no artificial delay, and then also plot the trade-off as we add an exponential delay on each worker. As the mean of the additional delay increases, using an intermediate value of $K$ achieves a better error-runtime trade-off. The trend of training loss is similar (see \Cref{fig:cifar_ksync_loss} in \Cref{app:additional_experimental_results}).
    \label{fig:cifar_ksync}}
\end{figure}

\begin{figure}[thb]
    \centering
    \begin{subfigure}{.33\textwidth}
    \centering
    \includegraphics[width=\textwidth]{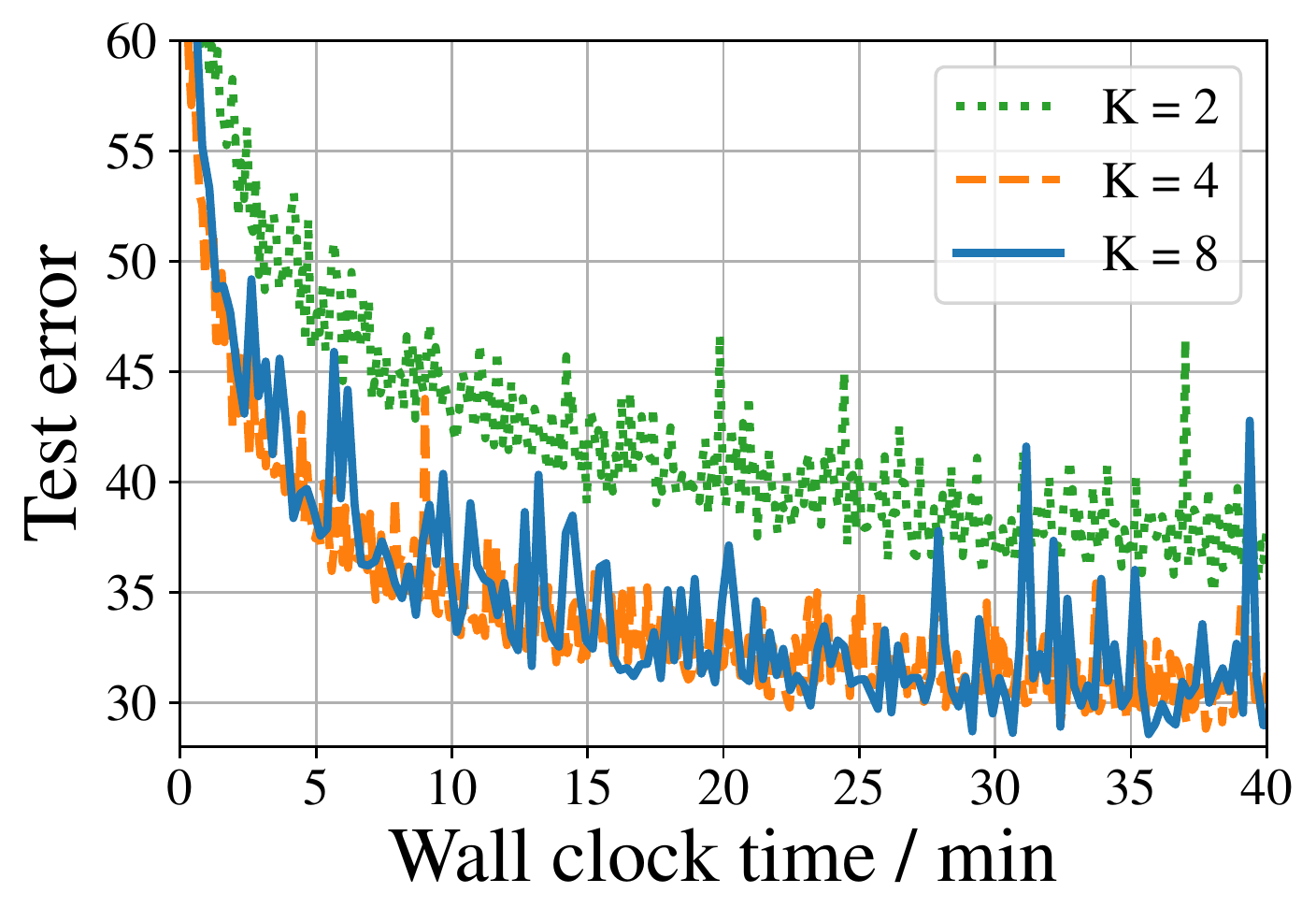}
    \caption{No artificial delay.}
    \label{fig:cifar_async_0}
    \end{subfigure}%
    ~
    \begin{subfigure}{.33\textwidth}
    \centering
    \includegraphics[width=\textwidth]{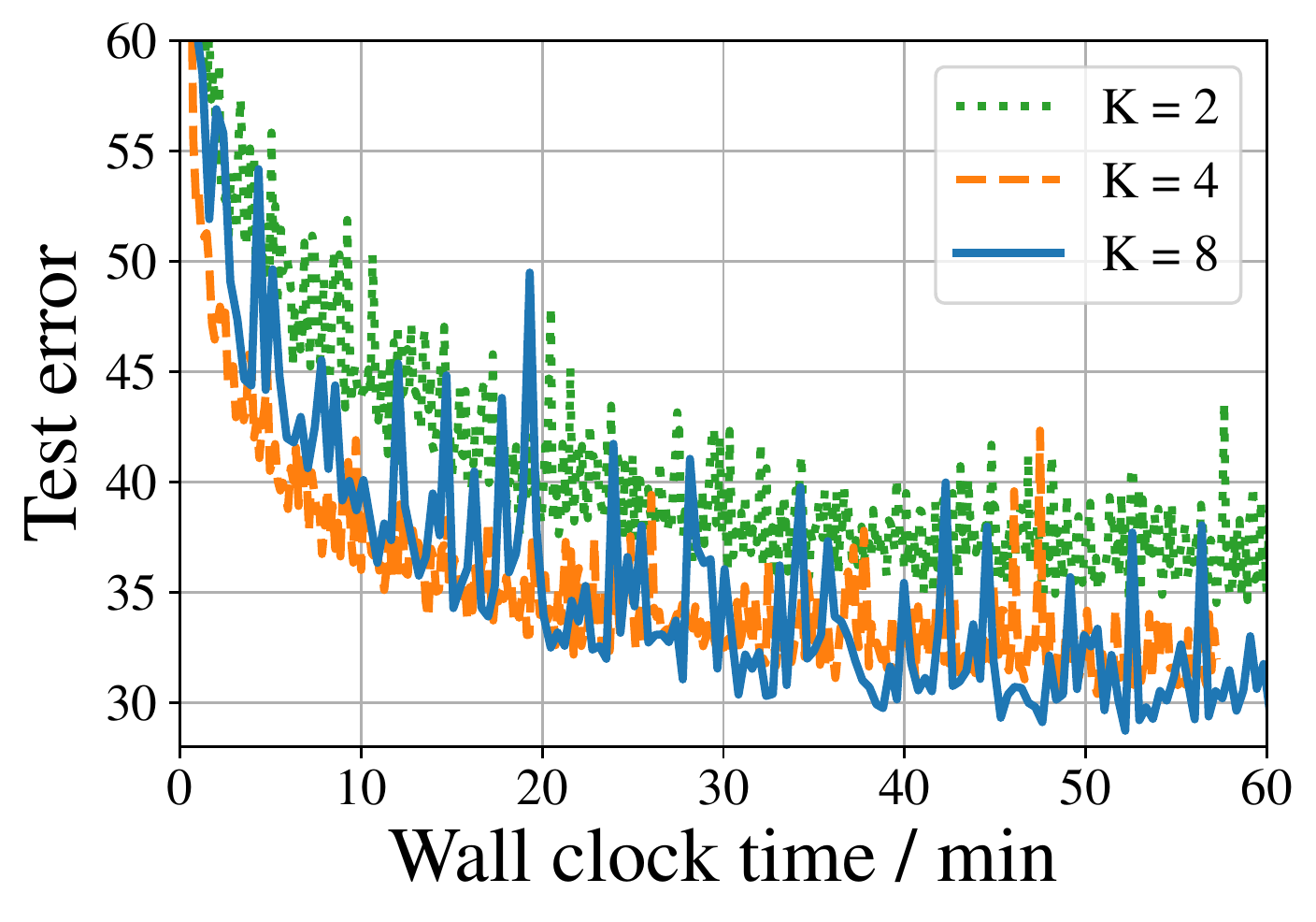}
    \caption{Delay with mean 0.02s.}
    \label{fig:cifar_async_2}
    \end{subfigure}%
    ~
    \begin{subfigure}{.33\textwidth}
    \centering
    \includegraphics[width=\textwidth]{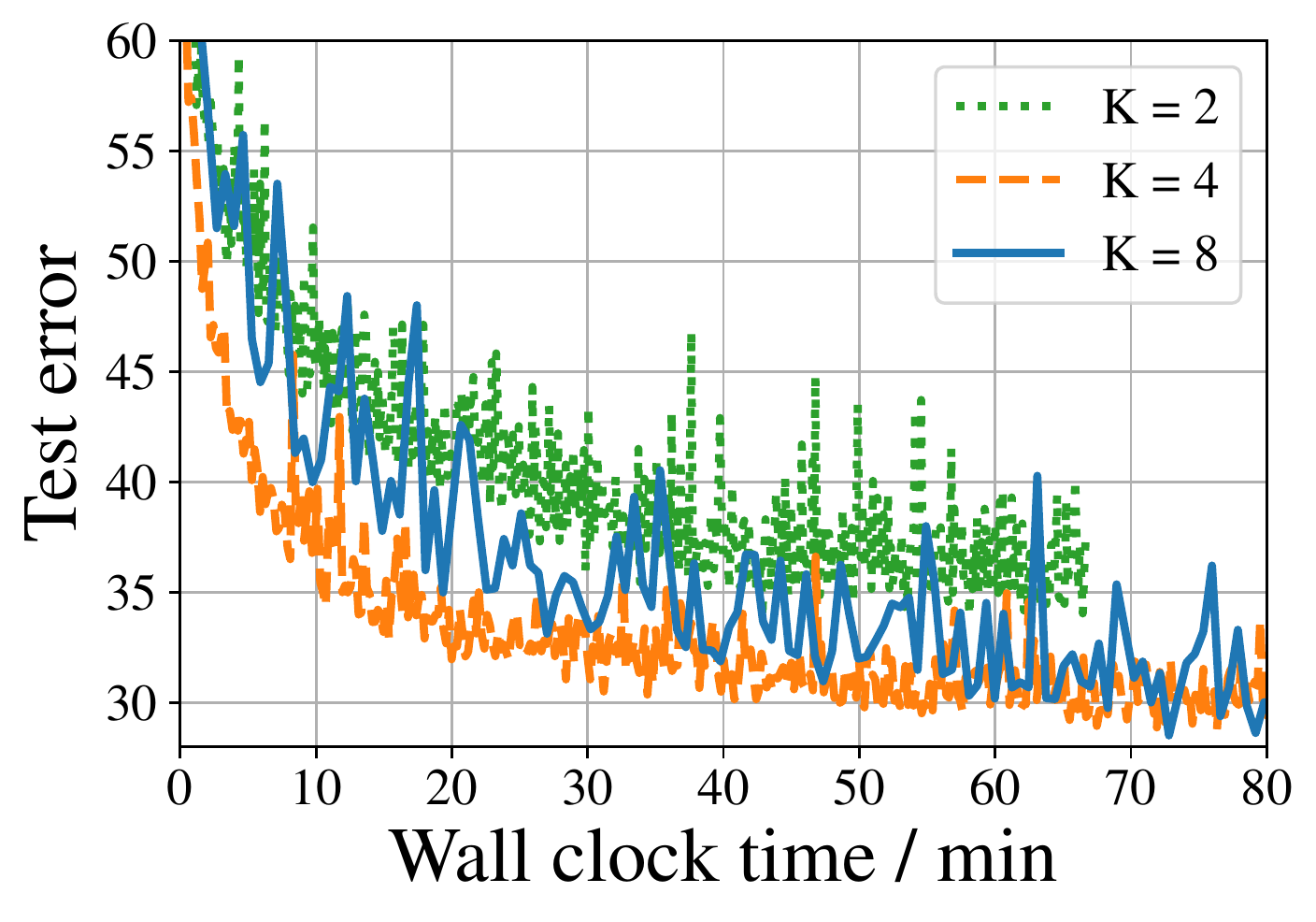}
    \caption{Delay with mean 0.05s.}
    \label{fig:cifar_async_5}
    \end{subfigure}%
    \caption{Test error ($\%$) of \textbf{K-async SGD} on CIFAR-10 with $8$ worker nodes. We demonstrate the error-runtime trade-off for the case with no artificial delay, and then also plot the trade-off as we add an exponential delay on each worker. As the mean of the additional delay increases, using an intermediate value of $K$ achieves a better error-runtime trade-off. The trend of training loss is again similar (see \Cref{fig:cifar_kasync_loss} in \Cref{app:additional_experimental_results}).
    \label{fig:cifar_kasync}}
\end{figure}

\section{AdaSync: Adaptive Synchronicity for Achieving the Best Error-Runtime Trade-Off}
\label{sec:vary_synchronicity}

As we observed both theoretically and empirically above, asynchronous SGD converges faster (with respect to wallclock time) but has a higher error floor. On the other hand, synchronous SGD converges slower (again with respect to wallclock time) but achieves a lower error floor. In this section, our goal is to try to achieve the best of both worlds, \emph{i.e.}, attain the most desirable error-runtime trade-off by gradually varying the level of synchronicity (parameter $K$) for the different SGD variants.

Let us partition the training time into intervals of time $t$ each, such that, after every slot of time $t$ we vary $K$. The number of iterations performed within time $t$ is assumed to be approximately $N(t)\approx t/\E{T}$ where $\E{T}$ is the expected runtime per-iteration for the chosen SGD variant\footnote{Note that, for exponential inter-arrival times, this approximation holds in the limit of large $t$.}. From \Cref{thm:non_convex}, we can write a (heuristic) upper bound on the average of $\E{||\nabla F(\wts)||_2^2}$  within each time interval $t$ as follows:
\begin{equation}
\text{Upper Bound as a function of }K = u(K)=
\frac{2(F(\wts_{start})-F^*)\E{T}}{t\eta \gamma'} + \frac{L\eta \sigma^2}{Km\gamma'},
\end{equation}
where $\wts_{start}$ denotes the value of the model $\wts$ at the beginning of that time interval\footnote{Note that $\gamma'$ can be set as $1$ for synchronous variants. Though, this does not matter much here as the minimizing $K$ does not depend on it.} Our goal is to minimize $u(K)$ with respect to $K$ for each time interval.

Observe that,
\begin{equation}
\frac{d u(K)}{d K} = \frac{2(F(\wts_{start}))}{t\eta \gamma'}\frac{d \E{T}}{d K} - \frac{L\eta \sigma^2}{K^2m\gamma'}. 
\end{equation}
Setting $\frac{d u(K)}{d K} $ to $0$ therefore provides a rough heuristic on how to choose parameter $K$ for each time interval, as long as, $\frac{d^2 u(K)}{d K^2}  $ is positive. We derive the rule for adaptively varying $K$ for each of the SGD variants in \Cref{sec:adasync_derivation}. Here, we include the method for one variant $K$-async to demonstrate the key idea.

For general distributions, the runtime of $K$-async is upper bounded by that of $K$-sync (see \Cref{lem:runtime kasync}). For exponential distributions, the two become equal and an algorithm similar to $K$-sync works. Here, we examine the interesting case of shifted-exponential distribution. We approximate $\E{T}\approx \frac{K \Delta}{P} + \frac{K \log{P}}{P \mu}$ (see \Cref{lem:tight kasync}). This leads to
$$K^2=   \frac{L\eta \sigma^2t\eta P \mu }{2m(F(\wts_{start}))(\Delta \mu + \log{P} )  }.$$ 
To actually solve for this equation, we would need the values of the Lipschitz constant, variance of the gradient etc. which are not always available. We sidestep this issue by proposing a heuristic here that relies on the ratio of the parameter $K$ at different instants.

Observe that, the larger is $F(\wts_{start})$, the smaller is the value of $K$ required to minimize $u(K)$. We assume that $F(\wts_{start})$ is maximum at the beginning of training, \textit{i.e.}, when $\wts_{start}=\wts_{0}$. Hence we start with the smallest initial $K$,~e.g., $K_0=1$. 
Thus, we could start with a small $K_0$ and after each time interval $t$, we can update $K$ by solving for
$$ K^2 = K_0^2 \frac{F(\wts_{0})}{F(\wts_{start})}.$$ We can also verify that the second derivative is positive, \textit{i.e.},
$$\frac{d^2 u(K)}{d K^2} =   \frac{2L\eta \sigma^2}{K^3m} >0 .$$

The detailed algorithm for AdaSync for $K$-async SGD in described in \Cref{algo}.

\begin{algorithm}[H]
\caption{AdaSync for $K$-async SGD}
\label{algo}
\begin{algorithmic}[1]
\STATE Start with $K=K_0$ (typically $K_0=1$).
\STATE \textbf{While} Wallclock Time $\leq$ Total Time Budget \textbf{ do:}
\STATE \hspace{0.2cm} Perform an iteration of $K$-async SGD
\STATE \hspace{0.2cm} \textbf{If} (Wallclock Time \% $t =0$) and ($K <P$) \textbf{ do:}
\STATE \hspace{1cm} Update $K$ as follows:
$K = K_0\sqrt{ \frac{F(\wts_{0})}{F(\wts_{start})}}$
\end{algorithmic}
\end{algorithm}

For $K$-sync SGD under exponential assumption, the update rule for $K$ is derived by solving for $K$ in the quadratic equation:  $\frac{K^2}{P-K} = \frac{K_0^2}{P-K_0} \frac{F(\wts_{0})}{F(\wts_{start})} ,$ as discussed in \Cref{sec:adasync_derivation}.
For the two other variants of distributed SGD, the adaptive update rule for $K$ is $K = K_0\sqrt{ \frac{ F(\wts_{0})}{F(\wts_{start})} }$, under certain assumptions on the runtime distribution, as also discussed in \Cref{sec:adasync_derivation}.

\begin{figure}[!htbp]
    \centering
    \begin{subfigure}{.4\textwidth}
    \centering
    \includegraphics[width=\textwidth]{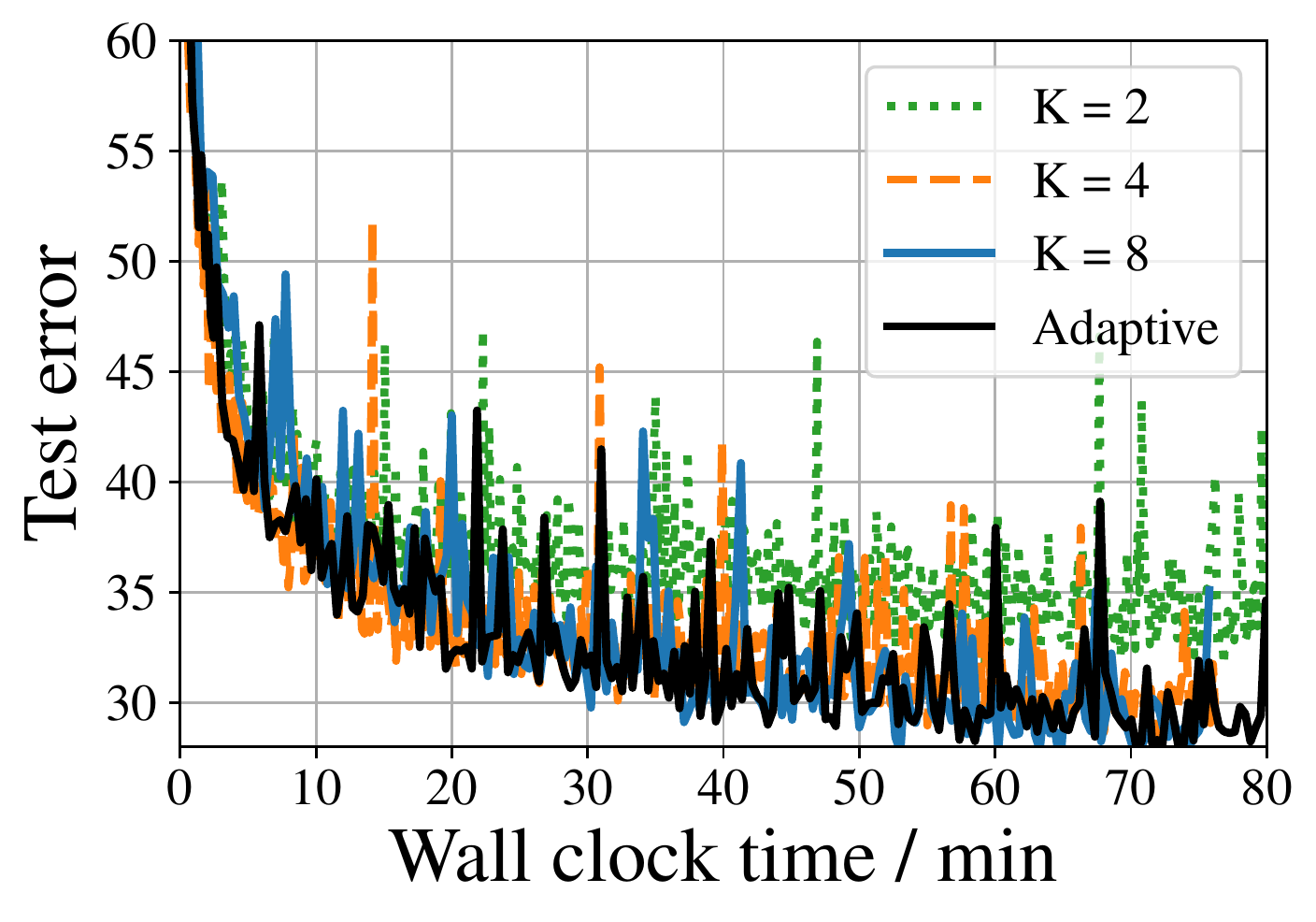}
    \caption{Adaptive K-sync. SGD.}
    \label{fig:cifar_sync_var}
    \end{subfigure}%
    ~
    \begin{subfigure}{.4\textwidth}
    \centering
    \includegraphics[width=\textwidth]{exp_figs/kasyncv.pdf}
    \caption{Adaptive K-async. SGD.}
    \label{fig:cifar_async_var}
    \end{subfigure}%
    \caption{Test error of \textbf{AdaSync SGD} on CIFAR-10 with $8$ worker nodes. We add an exponential delay with mean $0.02$s on each worker. The value of $K$ is changed after every $60$ seconds.
    \label{fig:cifar_adasync}}
\end{figure}

\subsection{AdaSync Experimental Results}
\label{subsec:adasync_experiments}
In this subsection, we evaluate the effectiveness of AdaSync for both K-sync SGD and K-async SGD algorithms. An exponential delay with mean $0.02$s is added to each worker node independently. We fix $K$ for every $t=60$ seconds (about 10 epochs). The initial values of $K$ are fine-tuned and set to $2$ and $4$ for K-sync SGD and K-async SGD, respectively. As shown in \Cref{fig:cifar_adasync}, the adaptive strategy achieves the fastest convergence in terms of error-versus-time. The adaptive K-async algorithm can even achieve a better error-runtime trade-off than the $K=8$ case (i.e., fully synchronous case).

\section{Concluding Remarks}
\label{sec:conclusion}


This work introduces a novel analysis of error-runtime trade-off of distributed SGD, accounting for both error reduction per iteration as well as the wallclock runtime in a delay-prone computing environment. Furthermore, we also give a new analysis of asynchronous SGD by relaxing some commonly made assumptions in existing literature. Lastly, we also propose a novel strategy called AdaSync that adaptively increases synchronicity during distributed machine learning to achieve the best error-runtime trade-off. Our results provide valuable insights into distributed machine learning that could inform choice of workers and preferred method of parallelization for a particular distributed SGD algorithm in a chosen distributed computing environment.

As future work, we plan to explore methods of gradually increasing synchrony in other distributed optimization frameworks, e.g., federated learning, decentralized SGD, elastic averaging etc., that is closely related to \cite{zhang2016parallel,yin2017gradient,zhou2017convergence,zhang2015deep,wang2018adaptive, wang2018cooperative,mcmahan2016communication,konevcny2016federated,li2019federated,wang2019matcha}. Our proposed techniques can also inform hyperparameter tuning. Given some knowledge of the computing environment, our technique could allow one to simulate the expected error-runtime trade-off in advance, and possibly choose training parameters such as parameter $K$, mini-batch size $m$ etc. It is also an interesting future direction to extend our current theoretical analysis for non iid\ scenarios, i.e., when the runtime or the dataset of different workers are not independent and identically distributed.

\section*{Acknowledgements}
The authors thank Soumyadip Ghosh, Parijat Dube,  Pulkit Grover, Souptik Sen, Li Zhang, Wei Zhang, and Priya Nagpurkar for helpful discussions. This work was supported in part by an IBM Faculty Award, and NSF CRII Award (CCF-1717314), and the Qualcomm Innovation fellowship.

\bibliographystyle{unsrt}
\bibliography{sample}

\appendix

\section{Strong Convexity Discussion}
\label{sec:strong_convexity}
\begin{defn}[Strong-Convexity]
A function $h(\mathbf{u})$ is defined to be $c$-strongly convex, if the following holds for all $\mathbf{u}_1$ and $\mathbf{u}_2$ in the domain:
\begin{equation*}
h(\mathbf{u}_2) \geq h(\mathbf{u}_1) + [\nabla h(\mathbf{u}_1)]^T(\mathbf{u}_2-\mathbf{u}_1) + \frac{c}{2} ||\mathbf{u}_2-\mathbf{u}_1 ||_2^2
\end{equation*}
\end{defn}

For strongly convex functions, the following result holds for all $\mathbf{u}$ in the domain of $h(.)$.
\begin{equation}
2c(h(\mathbf{u})- h^*) \leq ||\nabla h(\mathbf{u}) ||_2^2
\end{equation}
The proof is derived in \cite{bottou2016optimization}. For completeness, we give the sketch here.
\begin{proof}
Given a particular $\mathbf{u}$, let us define the quadratic function as follows:
$$q(\mathbf{u}')= h(\mathbf{u}) + \nabla h(\mathbf{u})^T (\mathbf{u}' - \mathbf{u}) + \frac{c}{2} ||\mathbf{u}'-\mathbf{u}  ||_2^2  $$
Now, $q(\mathbf{u}')$ is minimized at $\mathbf{u}' = \mathbf{u} - \frac{1}{c} \nabla h(\mathbf{u}) $ and the value is $h(\mathbf{u}) - \frac{1}{2c} ||\nabla h(\mathbf{u})   ||_2^2$. Thus, from the definition of strong convexity we now have,
\begin{align*}
h^* &\geq h(\mathbf{u}) + \nabla h(\mathbf{u})^T (\mathbf{u}' - \mathbf{u}) + \frac{c}{2} ||\mathbf{u}'-\mathbf{u}  ||_2^2 \nonumber \\
&\geq h(\mathbf{u}) - \frac{1}{2c} ||\nabla h(\mathbf{u})   ||_2^2 \ \ [\text{minimum value of } q(\mathbf{u}')].
\end{align*}
\end{proof}

\section{Runtime Analysis Proofs}
\label{subsec:runtime_K_sync}
Here we provide all the remaining proofs and supplementary information for the results in \Cref{subsec:runtime_lemmas}.

\subsection{Runtime of K-sync SGD}
\label{subsec:runtime_K_statistic}
\textbf{$K$-th statistic of exponential distributions:}
Here we give a sketch of why the $K$-th order statistic of $P$ exponentials scales as $\log (P/P-K)$. A detailed derivation can be obtained in \cite{sheldon2002first}. Consider $P$ i.i.d.\ exponential distributions with parameter $\mu$. The minimum $X_{1:P}$ of $P$ independent exponential random variables with parameter $\mu$ is exponential with parameter $P \mu$. Conditional on $X_{1:P}$, the second smallest value $X_{2:P}$ is distributed like the sum of $X_{1:P}$ and an independent exponential random variable with parameter $(P-1)\mu $. And so on, until the $K$-th smallest value $X_{K:P}$ which is distributed like the sum of $X_{(K-1):P}$ and an independent exponential random variable with parameter $(P-K+1)\mu$. Thus, $$X_{K:P}=Y_P+Y_{P-1}+ \dots+Y_{P-K+1} $$
where the random variables $Y_i$s are independent and exponential with parameter $i \mu$. Thus,
$$\E{X_{K:P}} = \sum_{i=P-K+1}^{P} \frac{1}{i\mu} =\frac{H_P-H_{P-K}}{\mu} \approx \frac{\log{\frac{P}{P-K}}}{\mu}. $$ 
Here $H_P$ and $H_{P-K}$ denote the $P$-th and $(P-K)$-th harmonic numbers respectively.

For the case where $K=P$, the expectation is given by,
$$\E{X_{P:P}} = \frac{1}{\mu}\sum_{i=1}^{P} \frac{1}{i} =\frac{1}{\mu}H_P \approx \frac{1}{\mu}\log{P}. $$
 

\subsection{Runtime of K-batch-sync SGD}
\label{subsec:runtime_K_batch_sync}

\begin{proof}[Proof of \Cref{lem:runtime kbatch sync}]
For new-longer-than-used distributions observe that the following holds:
\begin{align}
\Pr(X_i > u+t| X_i>t) \leq \Pr(X_i>u).
\end{align}

Thus the random variable corresponding to the remaining time, \textit{i.e.}, $X_i-t | X_i>t$ is thus stochastically dominated by $X_i$. 

Assume that we want to compute the expected computation time of one iteration of $K$-batch-sync SGD starting at time instant $t_0$. The time taken to complete the first mini-batch is the minimum of $P$ iid\ random variables, denoted as $\E{X_{1:P}}$. Suppose the first mini-batch finishes at time $t_1$, and the worker which finished first has index $i'$. Now, the time taken to complete the next mini-batch is the minimum of the remaining runtimes of all the $P$ workers of which the $i'$-th worker just started afresh at time $t_1$ while the $P-1$ other workers had been computing since time $t_0$. Let $Y_1,Y_2,\dots Y_P$ be the random variables denoting the remaining computation time of the $P$ workers after time $t_1$. Thus, 
\begin{equation}
Y_i = \begin{cases}
X_i-(t_1-t_0) | X_i> (t_1-t_0) \ \ \text{for all } i\neq i' \\
 X_{i} \text{  otherwise.}
\end{cases}
\end{equation}
Now each of the $Y_i$ s are independent and are stochastically dominated by $X_i$ s. 

\begin{equation}
\Pr(Y_i >u ) \leq \Pr(X_i >u ) \ \forall \ i = 1,2,\dots,P.
\end{equation}
The expectation of the minimum of $\{Y_1,Y_2,\dots,Y_P\}$ is the expected runtime of computing the second mini-batch. Let us denote $h_1(x_1,x_2,\dots,x_P)$ as the minimum of $P$ numbers  $(x_1,x_2,\dots,x_P)$. And let us denote $g_{1,\bm{s}}(x)$ as the minimum of $P$ numbers where $P-1$ of them are given as $\bm{s}_{1 \times (P-1)}$ and $x$ is the $P-$th number. Thus
$$g_{1,\bm{s}}(x)=h_1(x, s(1),s(2),\dots,s(P-1)). $$
First observe that $g_{1,\bm{s}}(x) $ is an increasing function of $x$ since given all the other $P-1$ values, the minimum will either stay the same or increase with $x$.  Now we use the property that if $Y_i$ is stochastically dominated by $X_i$, then for any increasing function $g(.)$, we have
$$\Esub{Y_1}{g(Y_1)} \leq \Esub{X_1}{g(X_1)}.$$
This result is derived in \cite{kreps1990course}.

This implies that for a given $\bm{s}$,
$$ \Esub{Y_1}{g_{K,\bm{s}}(Y_1)} \leq \Esub{X_1}{g_{K,\bm{s}}(X_1)}. $$
This leads to,
\begin{multline}
 \Esub{Y_1|Y_2=s(1),Y_3=s(2) \dots Y_P=s(P-1)}{h_1(Y_1,Y_2,\dots Y_P)} 
\\ \leq \Esub{X_1|Y_2=s(1),Y_3=s(2) \dots Y_P=s(P-1) }{h_1(X_1,Y_2,\dots Y_P)} . 
 \end{multline}

From this, 
\begin{align}
\E{h_1(Y_1,Y_2,\dots Y_P)} 
&=\Esub{Y_2,\dots,Y_P}{ \Esub{Y_1|Y_2, Y_3 \dots Y_P}{h_1(Y_1,Y_2,\dots Y_P)}} \nonumber \\
&\leq \Esub{Y_2,\dots,Y_P}{\Esub{X_1|Y_2,Y_3 \dots Y_P }{h_1(X_1,Y_2,\dots Y_P)} } \nonumber \\
& = \E{h_1(X_1,Y_2,\dots Y_P)}.
 \end{align}

This step proceeds inductively. Thus, similarly
\begin{align}
\E{h_1(X_1,Y_2,\dots Y_P)} 
& =\Esub{X_1,Y_3,\dots,Y_P}{ \Esub{Y_2|X_1, Y_3 \dots Y_P}{h_1(X_1,Y_2,\dots Y_P)}} \nonumber \\
&\leq \Esub{X_1,Y_3,\dots,Y_P}{\Esub{X_2|X_1, Y_3 \dots Y_P }{h_1(X_1,X_2, Y_3,\dots Y_P)} } \nonumber \\
&= \E{h_1(X_1,X_2,Y_3\dots Y_P)}.
 \end{align}

Thus, finally combining, we have,
\begin{align}
\E{h_1(Y_1,Y_2,\dots Y_P)}
&\leq \E{h_1(X_1,Y_2,\dots Y_P)} \nonumber \\
&\leq \E{h_1(X_1,X_2, Y_3\dots Y_P)} \nonumber \\& \leq \dots \nonumber \\
&\leq \E{h_1(X_1,X_2,X_3\dots X_P)}.
 \end{align}
 
Therefore, the expected runtime of the second mini-batch is upper bounded by $\E{X_{1:P}}$. A similar argument can be made for all the $K$ mini-batches, leading to an upper bound of $K\E{X_{1:P}}$ for all the $K$ mini-batches.
\end{proof}

In general, the exact analytical expression for the expected runtime per iteration of $K$-batch-sync SGD is not tractable but for the special case of exponentials it follows the distribution $Erlang(K, P\mu)$. This is obtained from the memoryless property of exponentials. 

All the workers start their computation together. The expected time taken by the first mini-batch to be completed is the minimum of $P$ i.i.d.\ exponential random variables $X_1,X_2,\dots,X_P \sim \Exp{(\mu)}$ is another exponential random variable  distributed as $ \Exp{(P\mu)}$. At the time when the first mini-batch is complete, from the memoryless property of exponentials, it may be viewed as $P$ i.i.d.\ exponential random variables $X_1,X_2,\dots,X_P \sim \Exp{(\mu)}$ starting afresh again. Thus, the time to complete each mini-batch is distributed as $\Exp{(P\mu)}$, and an iteration being the sum of the time to complete $K$ such mini-batches, has the distribution $Erlang(K, P\mu)$.

\subsection{Runtime of K-async SGD}
\label{subsec:runtime_K_async}
\begin{proof}[Proof of \Cref{lem:runtime kasync}]
For new-longer-than-used distributions observe that the following holds:
\begin{align}
\Pr(X_i > u+t| X_i>t) \leq \Pr(X_i>u).
\end{align}

Thus the random variable $X_i-t | X_i>t$ is thus stochastically dominated by $X_i$. Now let us assume we want to compute the expected computation time of one iteration of $K$-async starting at time instant $t_0$. Let us also assume that the workers last read their parameter values at time instants $t_1,t_2,\dots t_P$ respectively where any $K$ of these $t_1,t_2,\dots t_P$ are equal to $t_0$ as $K$ out of $P$ workers were updated at time $t_0$ and the remaining $(P-K)$ of these $t_1,t_2,\dots t_P$ are $< t_0$. 
Let $Y_1,Y_2,\dots Y_P$ be the random variables denoting the computation time of the $P$ workers starting from time $t_0$. Thus, 
\begin{equation}
Y_i = X_i-(t_0-t_i) | X_i> (t_0-t_i) \ \ \forall \  i=1,2,\dots,P.
\end{equation}
Now each of the $Y_i$ s are independent and are stochastically dominated by $X_j$ s.

\begin{equation}
\Pr(Y_i >u ) \leq \Pr(X_j >u ) \ \forall \ i,j = 1,2,\dots,P.
\end{equation}
The expectation of the $K$-th statistic of $\{Y_1,Y_2,\dots,Y_P\}$ is the expected runtime of the iteration. 
Let us denote $h_K(x_1,x_2,\dots,x_P)$ as the $K$-th statistic of $P$ numbers  $(x_1,x_2,\dots,x_P)$. And let us us denote $g_{K,\bm{s}}(x)$ as the $K$-th statistic of $P$ numbers where $P-1$ of them are given as $\bm{s}_{1 \times (P-1)}$ and $x$ is the $P-$th number. Thus 
$$g_{K,\bm{s}}(x)=h_K(x, s(1),s(2),\dots,s(P-1)) $$
First observe that $g_{K,\bm{s}}(x)$ is an increasing function of $x$ since given the other $P-1$ values, the $K$-th order statistic will either stay the same or increase with $x$. Now we use the property that if $Y_i$ is stochastically dominated by $X_i$, then for any increasing function $g(.)$, we have
$$\Esub{Y_1}{g(Y_1)} \leq \Esub{X_1}{g(X_1)}.$$
This result is derived in \cite{kreps1990course} .

This implies that for a given $\bm{s}$,
$$ \Esub{Y_1}{g_{K,\bm{s}}(Y_1)} \leq \Esub{X_1}{g_{K,\bm{s}}(X_1)}. $$

This leads to,
\begin{multline}
 \Esub{Y_1|Y_2=s(1),Y_3=s(2) \dots Y_P=s(P-1)}{h_K(Y_1,Y_2,\dots Y_P)} 
\\ \leq \Esub{X_1|Y_2=s(1),Y_3=s(2) \dots Y_P=s(P-1) }{h_K(X_1,Y_2,\dots Y_P)} . 
 \end{multline}

From this, 
\begin{align}
\E{h_K(Y_1,Y_2,\dots Y_P)} 
&=\Esub{Y_2,\dots,Y_P}{ \Esub{Y_1|Y_2, Y_3 \dots Y_P}{h_K(Y_1,Y_2,\dots Y_P)}} \nonumber \\
&\leq \Esub{Y_2,\dots,Y_P}{\Esub{X_1|Y_2,Y_3 \dots Y_P }{h_K(X_1,Y_2,\dots Y_P)} } \nonumber \\
& = \E{h_K(X_1,Y_2,\dots Y_P)}.
 \end{align}

This step proceeds inductively. Thus, similarly
\begin{align}
\E{h_K(X_1,Y_2,\dots Y_P)} 
& =\Esub{X_1,Y_3,\dots,Y_P}{ \Esub{Y_2|X_1, Y_3 \dots Y_P}{h_K(X_1,Y_2,\dots Y_P)}} \nonumber \\
&\leq \Esub{X_1,Y_3,\dots,Y_P}{\Esub{X_2|X_1, Y_3 \dots Y_P }{h_K(X_1,X_2, Y_3,\dots Y_P)} } \nonumber \\
&= \E{h_K(X_1,X_2,Y_3\dots Y_P)}.
 \end{align}

Thus, finally combining, we have,
\begin{align}
\E{h_K(Y_1,Y_2,\dots Y_P)}
&\leq \E{h_K(X_1,Y_2,\dots Y_P)} \nonumber \\
&\leq \E{h_K(X_1,X_2, Y_3\dots Y_P)} \nonumber \\& \leq \dots \nonumber \\
&\leq \E{h_K(X_1,X_2,X_3\dots X_P)}.
 \end{align}
\end{proof}

\subsubsection{Discussion for special distributions} 
\label{subsec:runtime_K_async_exp}
For exponential distributions, the inequality in 
\Cref{lem:runtime kasync} holds with equality. This follows from the memoryless property of exponentials. Let us consider the scenario of the proof of \Cref{lem:runtime kasync} where we similarly define $Y_i=X_i-(t_0-t_i)|X_i>(t_0-t_i)$. From the memoryless property of exponentials \cite{sheldon2002first}, if $X_i \sim \Exp(\mu)$, then $Y_i \sim \Exp(\mu)$. Thus, the expectation of the $K$-th statistic of $Y_i$s can be easily derived as all the $Y_i$s are now i.i.d.\ with distribution $\Exp(\mu)$. Thus, the expected runtime per iteration is given by,
$$\E{T}= \E{Y_{K:P}}= \frac{1}{\mu} \sum_{i=P-K+1}^P\frac{1}{i} \approx \frac{1}{\mu} \log{\frac{P}{P-K}} .$$

\begin{proof}[Proof of \Cref{lem:tight kasync}]
Here, we derive an alternate upper bound on the runtime for $n$ consecutive iterations for shited-exponential distributions where $P=nK$. Suppose, we are looking at the first $n$ consecutive iterations starting from time $t_0$. Assuming all the workers launched afresh for the first iteration, the expected time for the first iteration is given by $\E{X_{K:P}}= \Delta + \E{\widetilde{X}_{K:P}}$ where $\widetilde{X} \sim \Exp{(\mu)}$. 

For the second iteration, the runtime is the $K$-th order statistic of the remaining runtimes of the $P$ workers, where $K$ of them just launched afresh (hence shifted-exponential) and the remaining $P-K$ of them have been computing since $t_0$. The distribution of the remaining runtime for these $P-K$ workers correspond to exponential distributions with no shift ($\widetilde{X} \sim \Exp{(\mu)}$) because they have been running for at least time $\Delta$. The expectation of the $K$-th order statistic of all the $P$ runtimes can thus be upper-bounded by the $K$-th order statistic of a subset of the $P$ remaining runtimes, specifically, these $P-K$ workers having the exponential runtime distribution of $\widetilde{X}$. 

Similarly, the expected runtime of the third iteration is the $K$-th order statistic of the remaining runtimes of the $P$ workers of which at least $P-2K$ of them have been running since $t_0$ (for at least time $\Delta$) and hence have the distribution of $\widetilde{X}$. Thus, the expected runtime of the third iteration is upper-bounded by the $K$-th order statistic of these $P-2K$ exponential random variables. Extending a similar argument for $n$ iterations, the expected runtime for $n$ consecutive iterations can be upper-bounded as 
\begin{equation}
\E{T_1+T_2+\ldots+T_n} \leq \Delta + \sum_{i=0}^{n-1} \E{\widetilde{X}_{K:(P-iK)}}.
\end{equation}

From the new-longer-than-used property, the expected runtime of any $n$ consecutive iterations is upper-bounded by the expected runtime of the first $n$ consecutive iterations when all the $P$ workers are launched afresh.
\end{proof}

\subsection{Runtime of K-batch-async SGD}
\label{subsec:runtime_K_batch_async}
Here we include a discussion on renewal processes for completeness, to provide a background for the proof of \Cref{lem:runtime kbatch}, which gives the expected runtime of $K$-batch-async SGD. The familiar reader can merely skim through this and refer to the proof provided in the main section of the paper in \Cref{subsec:runtime_lemmas}. 

\begin{defn}[Renewal Process]
A renewal process is an arrival process where the inter-arrival intervals are positive, independent and identically distributed random variables.
\end{defn}

\begin{lem}[Elementary Renewal Theorem]
\cite[Chapter 5]{gallager2013stochastic}
\label{lem:renewal}
Let $\{N(t), t>0\} $ be a renewal counting process denoting the number of renewals in time $t$. Let $\E{Z}$ be the mean inter-arrival time. Then,
\begin{equation}
\lim_{t \to \infty} \frac{\E{N(t)}}{t} = \frac{1}{\E{Z}}.
\end{equation}
\end{lem}

Observe that for asynchronous SGD or $K$-batch-async SGD, every gradient push by a worker to the PS can be thought of as an arrival process. The time between two consecutive pushes by a worker follows the distribution of $X_i$ and is independent as computation time has been assumed to be independent across workers and mini-batches. Thus the inter-arrival intervals are positive, independent and identically distributed and hence, the gradient pushes are a renewal process. 

\section{Error Analysis Proofs}
\label{sec:async_proof}

In this section, we first discuss and understand the significance of the quantity $p_0$ that was introduced in the error-analysis in \Cref{thm:error kasync} in \Cref{subsec:proof_lem_p_0}. Next, for ease of understanding, we discuss and analyze the error convergence of asynchronous SGD first in \Cref{subsec:async_proof} because of its simplicity. This will be followed by the proof of the more general result, \textit{i.e.}, the error convergence analysis of $K$-async SGD (\Cref{thm:error kasync}) in \Cref{subsec:K_async_proof}. 

\subsection{Discussion on $p_0$}
\label{subsec:proof_lem_p_0}
Let us denote the conditional probability of $\tau(l,j)=j$ given all the past delays and parameters as $p_0^{(l,j)}$. Now $p_0 \leq p_0^{(l,j)} \ \forall l,j$. Clearly the value of $p_0^{(l,j)}$ will differ for different distributions and accordingly the value of $p_0$ will differ. Here we include a brief discussion on the possible values of $p_0$ for different distributions. These also hold for $K$-async and $K$-batch-async SGD.

\begin{proof}[Proof of \Cref{lem:p_0}]
Let $t_0$ be the time when the $j$-th iteration occurs, and suppose that worker $i'$ pushed its gradient in the $j$-th iteration. Now similar to the proof of \Cref{lem:runtime kasync}, let us also assume that the workers last read their parameter values at time instants $t_1,t_2,\dots t_P$ respectively where $t_i'=t_0$ and the remaining $(P-1)$ of these $t_i$s are $< t_0$. Let $Y_1,Y_2,\dots Y_P$ be the random variables denoting the computation time of the $P$ workers starting from time $t_0$. Thus, $Y_i= X_i -(t_0-t_i)| X_i >(t_0-t_i)$. For exponentials, from the memoryless property, all these $Y_i$ s become i.i.d.\ and thus from symmetry the probability of $i'$ finishing before all the others is equal, \textit{i.e.} $ \frac{1}{P}$. Thus, $p_0^{(j)} =p_0=\frac{1}{P}$.
For new-longer-than-used distributions, as we have discussed before all the $Y_i$s with $i \neq i'$ will be stochastically dominated by $Y_{i'}=X_{i'}$. Thus, probability of $i$s with $i \neq i'$ finishing first is higher than $i'$. Thus, $p_0^{(j)} \leq \frac{1}{P}$ and so is $p_0$. Similarly, for new-shorter-than-used distributions, $Y_{i'}$ is stochastically dominated by all the $Y_i$s and thus probability of $i'$ finishing first is more. So, $p_0^{(j)} \geq \frac{1}{P}$ and so is $p_0$.
\end{proof}

We use \Cref{lem:delay} below to prove \Cref{thm:error kasync}. 
\begin{lem} 
\label{lem:delay}
Suppose that $p_0^{(l,j)}$ is the conditional probability that $\tau(l,j)=j$ given all the past delays and all the previous $\wts$, and $p_0 \leq p_0^{(l,j)}$ for all $j$. Then,
\begin{equation}
\E{||\nabla F(\wts_{\tau(l,j)})||_2^2} \geq p_0 \E{||\nabla F(\wts_{j})||_2^2}.
\end{equation}
\end{lem}
\begin{proof}
By the law of total expectation,
\begin{align*}
&\E{||\nabla F(\wts_{\tau(l,j)})||_2^2}\nonumber \\
&= p_0^{(l,j)} \E{||\nabla F(\wts_{\tau(l,j)})||_2^2|\tau(j)= j}    
+ (1-p_0^{(l,j)}) \E{||\nabla F(\wts_{\tau(l,j)})||_2^2|\tau(j) \neq j} \nonumber \\
&  \geq p_0 \E{||\nabla F(\wts_{j})||_2^2}. 
\end{align*}
\end{proof}

\subsection{Async-SGD with fixed learning rate}
\label{subsec:async_proof}
First we prove a simplified version of \Cref{thm:error kasync} for the case $K=1$. While this is actually a corollary of the more general \Cref{thm:error kasync}, we prove this first for ease of understanding and simplicity. The proof of the more general \Cref{thm:error kasync} is then provided in \Cref{subsec:K_async_proof}. 

The corollary is as follows:
\begin{coro} Suppose that the objective function $F(\wts)$ is strongly convex with parameter $c$ and the learning rate $\eta \leq \frac{1}{2\lips(\frac{M_G}{m}+ 1)}$. Also assume that $\E{|| \nabla F(\wts_{j}) - \nabla F(\wts_{\tau(j)})||_2^2 } \leq \gamma \E{|| \nabla F(\wts_{j}) ||_2^2 }$ for some constant $\gamma \leq 1$. Then, the error after $J$ iterations of Async SGD is given by,
\begin{align*}
 \E{F(\wts_{J})} & -F^* 
\leq \frac{\eta L\sigma^2}{2c\gamma' m} +  (1 - \eta c\gamma')^{J} \left(\E{F(\wts_{0})}-F^* - \frac{\eta L\sigma^2}{2c \gamma' m}  \right),
\end{align*}
where $\gamma'= 1-\gamma + \frac{p_0}{2}$ and $p_0$ is a non-negative lower bound on the conditional probability that $\tau(j)=j$ given all the past delays and parameters.
\label{coro:fixed learning rate}
\end{coro}

To prove the result, we will use the following lemma.
\begin{lem}
\label{lem:bias-variance}
Let us denote $\mathbf{v}_j = g(\wts_{\tau(j)}, \xi_{j})$, and assume that $ \Esub{\xi_{j}|
\wts}{g(\wts,\xi_{j})}= \nabla F(\wts) $. Then,
\begin{align*}
\E{||\nabla F(\wts_{j}) - \mathbf{v}_j||^2_2 } 
& \leq  \E{||\mathbf{v}_j||^2_2} -  
 \E{||\nabla F(\wts_{\tau(j)})||_2^2 } \\
 & \hspace{2cm} +  \E{|| \nabla F(\wts_{j}) - \nabla F(\wts_{\tau(j)})||_2^2 }.
\end{align*}
\end{lem}

\begin{proof}[Proof of \Cref{lem:bias-variance}]
Observe that,
\begin{align}
\E{||\nabla F(\wts_{j}) - \mathbf{v}_j||^2_2 }  &= \E{||\nabla F(\wts_{j}) -\nabla F(\wts_{\tau(j)}) + \nabla F(\wts_{\tau(j)}) - \mathbf{v}_j||^2_2 }  \nonumber \\
&= \E{||\nabla F(\wts_{j}) - \nabla F(\wts_{\tau(j)})||^2_2 } 
 + \E{||\mathbf{v}_j - \nabla F(\wts_{\tau(j)})||^2_2}.
\label{update-bias-variance}
\end{align}

The last line holds since the cross term is $0$ as derived below.
\begin{align*}
&\E{(\nabla F(\wts_{j}) - \nabla F(\wts_{\tau(j)})^T(\mathbf{v}_j - \nabla F(\wts_{\tau(j)}))}  \nonumber \\
& =  \mathbb{E}_{\wts_{\tau(j)},\wts_{j}} [(\nabla F(\wts_{j}) - \nabla F(\wts_{\tau(j)})^T  \Esub{\xi_j|\wts_{\tau(j)},\wts_{j} }{(\mathbf{v}_j - \nabla F(\wts_{\tau(j)}))} ] \nonumber \\
& =  \mathbb{E}_{\wts_{\tau(j)},\wts_{j}} [(\nabla F(\wts_{j}) - \nabla F(\wts_{\tau(j)})^T
(\Esub{\xi_j|\wts_{\tau(j)}}{\mathbf{v}_j} - \nabla F(\wts_{\tau(j)}))] \nonumber \\
&= 0.
\end{align*}
Here again the last line follows from Assumption 2 in \Cref{sec:system model} which states that $$\Esub{\xi_j|\wts_{\tau(j)}}{\mathbf{v}_j} = \nabla F(\wts_{\tau(j)})).$$
Returning to \eqref{update-bias-variance}, observe that the second term can be further decomposed as,
\begin{align*}
 \E{||\mathbf{v}_j - \nabla F(\wts_{\tau(j)})||^2_2} 
& = \Esub{\wts_{\tau(j)}}{ \Esub{\xi_j|\wts_{\tau(j)}}{||\mathbf{v}_j - \nabla F(\wts_{\tau(j)})||^2_2}} \\
&=\Esub{\wts_{\tau(j)}}{ \Esub{\xi_j|\wts_{\tau(j)}}{||\mathbf{v}_j||_2^2}} -2 \Esub{\wts_{\tau(j)}}{ \Esub{\xi_j|\wts_{\tau(j)}}{\mathbf{v}_j^T \nabla F(\wts_{\tau(j)}) }} \\
& \hspace{2cm} + \Esub{\wts_{\tau(j)}}{\Esub{\xi_j|\wts_{\tau(j)}}{|| \nabla F(\wts_{\tau(j)})||_2^2 }}\\
& = \E{||\mathbf{v}_j||_2^2} - 2\E{||\nabla F(\wts_{\tau(j)})  ||_2^2}  + \E{||\nabla F(\wts_{\tau(j)})  ||_2^2 } \\
& = \E{||\mathbf{v}_j||_2^2} - \E{||\nabla F(\wts_{\tau(j)})  ||_2^2 }.
\end{align*}
\end{proof}
We will also be proving a $K$-worker version of this lemma \Cref{subsec:K_async_proof} to prove \Cref{thm:error kasync}. Now we proceed to provide the proof of \Cref{coro:fixed learning rate}.

\begin{proof}[Proof of \Cref{coro:fixed learning rate}]
\begin{align}
 F(\wts_{j+1})   \leq & F(\wts_{j}) + (\wts_{j+1}-\wts_{j} )^T \nabla F(\wts_{j})    + \frac{L}{2} ||\wts_{j+1}-\wts_{j}   ||_2^2 
 \nonumber \\
 = & F(\wts_{j}) + (-\eta \mathbf{v}_j )^T \nabla F(\wts_{j}) + \frac{L \eta^2}{2} || \mathbf{v}_j   ||_2^2 
 \nonumber \\
 = & F(\wts_{j}) - \frac{\eta}{2}||\nabla F(\wts_{j})||_2^2 - \frac{\eta}{2}||\mathbf{v}_j||_2^2 
 + \frac{\eta}{2}||\nabla F(\wts_{j})- \mathbf{v}_j ||_2^2 + \frac{L\eta^2}{2}||\mathbf{v}_j||_2^2. 
 \label{lipschitz condition}
\end{align}

Here the last line follows from $2\bm{a}^T\bm{b} = ||\bm{a}||_2^2 + ||\bm{b}||_2^2 - ||\bm{a}-\bm{b}||_2^2 $. Taking expectation,
\begin{align}
\E{F(\wts_{j+1})} 
 &\leq \E{F(\wts_{j})} - \frac{\eta}{2}\E{||\nabla F(\wts_{j})||_2^2}  - \frac{\eta}{2}\E{||\mathbf{v}_j||_2^2} + \frac{\eta}{2}\E{||\nabla F(\wts_{j})- \mathbf{v}_j ||_2^2}+ \frac{L\eta^2}{2}\E{||\mathbf{v}_j||_2^2} 
 \nonumber\\
& \overset{(a)}{\leq}  \E{F(\wts_{j})} - \frac{\eta}{2}\E{||\nabla F(\wts_{j})||_2^2} - \frac{\eta}{2}\E{||\mathbf{v}_j||_2^2}  
 + \frac{ \eta}{2} \E{|| \mathbf{v}_j||_2^2} - \frac{\eta}{2}\E{||\nabla F(\wts_{\tau(j)})||_2^2} 
\nonumber \\
&\hspace{3cm} +   \frac{\eta}{2}\E{|| \nabla F(\wts_{j}) - \nabla F(\wts_{\tau(j)})||_2^2 } + \frac{L\eta^2}{2}\E{||\mathbf{v}_j||_2^2}. 
\label{termdiff} 
\end{align}
Here, (a) follows from Lemma \ref{lem:bias-variance} that we just derived. Now, again bounding from \eqref{termdiff}, we have
\begin{align}
\E{F(\wts_{j+1})} & \overset{(b)}{\leq} \E{F(\wts_{j})} - \frac{\eta}{2}\E{||\nabla F(\wts_{j})||_2^2}  - \frac{\eta}{2}\E{||\nabla F(\wts_{\tau(j)})||_2^2}  
 + \frac{\eta}{2} \gamma \E{|| \nabla F(\wts_{j})||_2^2 } 
 \nonumber \\
& \hspace{2cm}
 + \frac{L\eta^2}{2}\E{||\mathbf{v}_j||_2^2} \nonumber \\
&\overset{(c)}{\leq} \E{F(\wts_{j})} - \frac{\eta}{2}(1-\gamma)\E{||\nabla F(\wts_{j})||_2^2} + \frac{L\eta^2\sigma^2}{2m}  
\nonumber \\ 
& \hspace{2cm} - \frac{\eta}{2} \left(1 -L\eta(\frac{ M_G}{m}+1)  \right)\E{||\nabla F(\wts_{\tau(j)})||_2^2} 
\nonumber \\
& \overset{(d)}{\leq} \E{F(\wts_{j}) } - \frac{\eta}{2}(1-\gamma)\E{||\nabla F(\wts_{j})||_2^2 } + \frac{L\eta^2\sigma^2}{2m} 
 - \frac{\eta}{4}\E{||\nabla F(\wts_{\tau(j)})||_2^2 } 
\nonumber \\
& \overset{(e)}{\leq} \E{F(\wts_{j}) } - \frac{\eta}{2}(1-\gamma)\E{||\nabla F(\wts_{j})||_2^2 } + \frac{L\eta^2\sigma^2}{2m} 
 - \frac{\eta}{4}p_0\E{||\nabla F(\wts_{j})||_2^2 }. \label{recursion}
\end{align}

Here  (b) follows from the statement of the theorem that $$\E{|| \nabla F(\wts_{j}) - \nabla F(\wts_{\tau(j)})||_2^2 } \leq \gamma \E{|| \nabla F(\wts_{j}) ||_2^2 }$$ for some constant $\gamma \leq 1$.
The next step (c) follows from Assumption 4 in \Cref{sec:system model} which lead to   $$\E{|| \mathbf{v}_j ||_2^2  } \leq \frac{\sigma^2}{m} + \left(\frac{M_G}{m}+1 \right)\E{|| \nabla F(\wts_{\tau(j)}) ||_2^2 }   .$$
Step (d) follows from choosing $\eta < \frac{1}{2L(\frac{M_G}{m} + 1)} $ and finally (e) follows from Lemma \ref{lem:delay}. 

Now one might recall that the function $F(w)$ was defined to be  strongly convex with parameter $c$. 

Using the standard result of strong-convexity \cref{eq:strong-convexity} in \eqref{recursion}, we obtain the following result:
 \begin{align*}
\E{F(\wts_{j+1})}&-F^*  \leq  \frac{\eta^2L\sigma^2}{2m} + (1 - \eta c(1-\gamma + \frac{p_0}{2})) (\E{F(\wts_{j})}-F^*   ).   
\end{align*}
Let us denote $\gamma'= (1-\gamma + \frac{p_0}{2}) $. Then, using the above recursion, we thus have,
\begin{align*}
& \E{F(\wts_{J})} -F^* 
\leq \frac{\eta L\sigma^2}{2c\gamma' m}   + (1 - \eta \gamma'c)^J (\E{F(\wts_{0})}-F^* - \frac{\eta L\sigma^2}{2c\gamma' m}  ).
\end{align*}
\end{proof}

\subsection{K-async SGD under fixed learning rate}
\label{subsec:K_async_proof}
In this subsection, we provide a proof of \Cref{thm:error kasync}.

Before we proceed to the proof of this theorem, we first extend our Assumption 4 from the variance of a single stochastic gradient to sum of stochastic gradients in the following Lemma.

\begin{lem}
\label{lem:assumption4}If the variance of the stochastic updates is bounded as $$\Esub{\xi_{j}|\wts_{\tau{l,j}}}{||g(\wts_{\tau(l,j)},\xi_{l,j})- \nabla F(\wts_{\tau(l,j)}) ||_2^2} \nonumber \\
\leq \frac{\sigma^2}{m} + \frac{M_G}{m} ||\nabla F(\wts_{\tau(l,j)}) ||_2^2 \ \forall \ \tau(l,j) \leq j,$$ then for $K$-async, the variance of the sum of stochastic updates given all the parameter values $\wts_{\tau(l,j)}$ is also bounded as follows:
\begin{align}
&\Esub{\xi_{1,j},\dots,\xi_{K,j}|\wts_{\tau(1,j)} \dots \wts_{\tau(K,j)}}{||\sum_{l=1}^K g(\wts_{l,j}, \xi_{l,j})||_2^2} \nonumber \\ & \leq 
\frac{K\sigma^2}{m} + \left(\frac{M_G}{m}+K\right)||\sum_{l=1}^K \nabla F(\wts_{\tau(l,j)})  ||_2^2.
\end{align}
\end{lem}
\begin{proof}
First let us consider the expectation of any cross term such that $l \neq l'$. For the ease of writing, let $\Omega= \{ \wts_{\tau(1,j)} \dots \wts_{\tau(K,j)} \} $. 

Now observe the conditional expectation of the cross term as follows:
\begin{align}
&\mathbb{E}_{\xi_{1,j},\dots,\xi_{K,j}| \Omega}[(g(\wts_{l,j}, \xi_{l,j})-\nabla F(\wts_{\tau(l,j)}))^T  ((g(\wts_{l',j}, \xi_{l',j})-\nabla F(\wts_{\tau(l',j)}))] \nonumber 
\\
& = \mathbb{E}_{\xi_{l,j},\xi_{l',j}| \Omega}[(g(\wts_{l,j}, \xi_{l,j})-\nabla F(\wts_{\tau(l,j)}))^T ((g(\wts_{l',j}, \xi_{l',j})-\nabla F(\wts_{\tau(l',j)}))] \nonumber 
\\
& = \mathbb{E}_{\xi_{l',j}| \Omega }[\mathbb{E}_{\xi_{l,j}|\xi_{l',j},\Omega}[ (g(\wts_{l,j}, \xi_{l,j})-\nabla F(\wts_{\tau(l,j)}))^T] (g(\wts_{l',j}, \xi_{l',j})-\nabla F(\wts_{\tau(l',j)})] \nonumber 
\\
& = \mathbb{E}_{\xi_{l',j}| \Omega } [0^T  (g(\wts_{l',j}, \xi_{l',j})-\nabla F(\wts_{\tau(l',j)})] 
=0.
\end{align}

Thus the cross terms are all $0$. So the expression simplifies as,
\begin{align}
 & \Esub{\xi_{1,j},\dots,\xi_{K,j}|\Omega}{||\sum_{l=1}^K g(\wts_{l,j}, \xi_{l,j}) - F(\wts_{\tau(l,j)})||_2^2} \nonumber \\
& \overset{(a)}{=} \sum_{l=1}^K \Esub{\xi_{1,j},\dots,\xi_{K,j}|\Omega}{||g(\wts_{l,j}, \xi_{l,j})  -F(\wts_{\tau(l,j)})||_2^2} \nonumber \\
& \leq \sum_{l=1}^K \frac{\sigma^2}{m} + \frac{M_G}{m} ||\nabla F(\wts_{\tau(l,j)}) ||_2^2.
\end{align}
Thus,
\begin{align}
&\Esub{\xi_{1,j},\dots,\xi_{K,j}|\Omega}{||\sum_{l=1}^K g(\wts_{l,j}, \xi_{l,j})||_2^2} \nonumber 
\\ 
& = \Esub{\xi_{1,j},\dots,\xi_{K,j}|\Omega}{||\sum_{l=1}^K g(\wts_{l,j}, \xi_{l,j}) - F(\wts_{\tau(l,j)}) ||_2^2} 
+ \Esub{\xi_{1,j},\dots,\xi_{K,j}|\Omega}{||\sum_{l=1}^K  F(\wts_{\tau(l,j)}) ||_2^2} 
\nonumber 
\\
& \leq \frac{K\sigma^2}{m} + \sum_{l=1}^K \frac{M_G}{m}||  F(\wts_{\tau(l,j)}) ||_2^2 + ||\sum_{l=1}^K  F(\wts_{\tau(l,j)}) ||_2^2 \nonumber \\
&\leq \frac{K\sigma^2}{m} + \sum_{l=1}^K \frac{M_G}{m}||  F(\wts_{\tau(l,j)}) ||_2^2 + \sum_{l=1}^K K||  F(\wts_{\tau(l,j)}) ||_2^2.
\end{align}
\end{proof}

Now we return to the proof of the theorem.

\begin{proof}[Proof of \Cref{thm:error kasync}] Let  $\mathbf{v}_j= \frac{1}{K} \sum_{l=1}^K g(\wts_{l,j}, \xi_{l,j})$. Following steps similar to the Async-SGD proof, from Lipschitz continuity we have the following.
\begin{align}
 F(\wts_{j+1})  & \leq F(\wts_{j}) + (\wts_{j+1}-\wts_{j} )^T \nabla F(\wts_{j})  
  + \frac{L}{2} ||\wts_{j+1}-\wts_{j}   ||_2^2 
 \nonumber \\
 = & F(\wts_{j}) - \frac{\eta}{K} \sum_{l=1}^K g(\wts_{l,j}, \xi_{l,j})^T \nabla F(\wts_{j}) 
+ \frac{L}{2} ||\eta \mathbf{v}_j   ||_2^2  \nonumber 
 \\
 \overset{(a)}{=} & F(\wts_{j}) -\frac{\eta}{2K} \sum_{l=1}^K ||\nabla F(\wts_{j})||_2^2 
- \frac{\eta}{2K} \sum_{l=1}^K ||g(\wts_{l,j}, \xi_{l,j})||_2^2
 \nonumber \\
 &  \hspace{2cm} + \frac{\eta}{2K} \sum_{l=1}^K ||g(\wts_{l,j}, \xi_{l,j}) ||_2^2
 - \frac{\eta}{2K} \sum_{l=1}^K || \nabla F(\wts_{j})  ||_2^2 
 + \frac{L \eta^2}{2} || \mathbf{v}_j   ||_2^2 
 \nonumber \\
 = & F(\wts_{j}) - \frac{\eta}{2}||\nabla F(\wts_{j})||_2^2 
 - \frac{\eta}{2K} \sum_{l=1}^K ||g(\wts_{l,j}, \xi_{l,j})||_2^2
 \nonumber \\
 & \hspace{2cm}
 + \frac{\eta}{2K} \sum_{l=1}^K ||g(\wts_{l,j}, \xi_{l,j}) - \nabla F(\wts_{j})  ||_2^2 
 + \frac{L \eta^2}{2} || \mathbf{v}_j   ||_2^2. 
\end{align}

Here (a) follows from $2\bm{a}^T\bm{b} = ||\bm{a}||_2^2 + ||\bm{b}||_2^2 - ||\bm{a}-\bm{b}||_2^2 $. 
Taking expectation,
\begin{align}
\E{F(\wts_{j+1})} 
 &\leq \E{F(\wts_{j})} - \frac{\eta}{2}\E{||\nabla F(\wts_{j})||_2^2}   - \frac{\eta}{2K} \sum_{l=1}^K \E{||g(\wts_{l,j}, \xi_{l,j})||_2^2} \nonumber\\
 & \hspace{2cm} + \frac{\eta}{2K} \sum_{l=1}^K \E{||\nabla F(\wts_{j})- g(\wts_{l,j}, \xi_{l,j}) ||_2^2}
+ \frac{L\eta^2}{2}\E{||\mathbf{v}_j||_2^2} 
 \nonumber
 \\
& \overset{(a)}{\leq}  \E{F(\wts_{j})} - \frac{\eta}{2}\E{||\nabla F(\wts_{j})||_2^2} 
- \frac{\eta}{2K} \sum_{l=1}^K \E{||g(\wts_{l,j}, \xi_{l,j})||_2^2}
\nonumber \\
&\hspace{2cm}  + \frac{ \eta}{2K}  \sum_{l=1}^K  \E{|| g(\wts_{l,j}, \xi_{l,j})||_2^2}   -\frac{ \eta}{2K}  \sum_{l=1}^K  \E{||\nabla F(\wts_{\tau(l,j)})||_2^2} 
 \nonumber \\
&\hspace{2cm} +   \frac{\eta}{2K}\sum_{l=1}^K \E{|| \nabla F(\wts_{j}) 
 - \nabla F(\wts_{\tau(l,j)})||_2^2 }
 + \frac{L\eta^2}{2}\E{||\mathbf{v}_j||_2^2} 
 \\
& \overset{(b)}{\leq} \E{F(\wts_{j})} - \frac{\eta}{2}\E{||\nabla F(\wts_{j})||_2^2}  - \frac{\eta}{2K}\sum_{l=1}^K \E{||\nabla F(\wts_{\tau(l,j)})||_2^2}  
 + \frac{\eta}{2} \gamma \E{|| \nabla F(\wts_{j})||_2^2 } 
 \nonumber \\
&\hspace{4cm}
+ \frac{L\eta^2}{2}\E{||\mathbf{v}_j||_2^2} \nonumber \\
&\overset{(c)}{\leq} \E{F(\wts_{j})} - \frac{\eta}{2}(1-\gamma)\E{||\nabla F(\wts_{j})||_2^2} + \frac{L\eta^2\sigma^2}{2Km}  
\nonumber \\ 
& \hspace{2cm} - \frac{\eta}{2K} \sum_{l=1}^K \left(1 -L\eta \left(\frac{ M_G}{Km}+\frac{1}{K}\right)  \right)\E{||\nabla F(\wts_{\tau(l,j)})||_2^2} 
\nonumber \\
& \overset{(d)}{\leq} \E{F(\wts_{j}) } - \frac{\eta}{2}(1-\gamma)\E{||\nabla F(\wts_{j})||_2^2 } + \frac{L\eta^2\sigma^2}{2Km} 
 - \frac{\eta}{4K}\sum_{l=1}^K \E{||\nabla F(\wts_{\tau(l,j)})||_2^2 } 
\nonumber \\
& \overset{(e)}{\leq} \E{F(\wts_{j}) } - \frac{\eta}{2}(1-\gamma)\E{||\nabla F(\wts_{j})||_2^2 } + \frac{L\eta^2\sigma^2}{2Km} 
 - \frac{\eta}{4}p_0\E{||\nabla F(\wts_{j})||_2^2 }. 
\label{recursion2}
\end{align}

Here step (a) follows from \Cref{lem:bias-variance} and step (b) follows from the assumption that $$\E{|| \nabla F(\wts_{j}) - \nabla F(\wts_{\tau(l,j)})||_2^2 } \leq \gamma \E{|| \nabla F(\wts_{j}) ||_2^2 }$$ for some constant $\gamma \leq 1$. The next step (c) follows from the \Cref{lem:assumption4} that bounds the variance of the sum of stochastic gradients. Step (d) follows from choosing $\eta < \frac{1}{2L(\frac{M_G}{Km} + \frac{1}{K})} $ and finally (e) follows from \Cref{lem:delay} in \Cref{sec:main_async_fixed} that says $ \E{||\nabla F(\wts_{\tau(l,j)})||_2^2} \geq p_0\E{||\nabla F(\wts_{j})||_2^2}$ for some non-negative constant $p_0$ which is a lower bound on the conditional probability that $\tau(l,j)=j$ given all past delays and parameter values. 

Finally, since $F(\wts)$ is strongly convex, using the inequality $2c(F(\wts)-F^*) \leq ||\nabla F(\wts)  ||_2^2$ in \eqref{recursion2}, we finally obtain the desired result.
\end{proof}

\subsection{Extension to Non-Convex case}
\label{subsec:nonconvex}
\begin{proof}
Recall the recursion derived in the last proof in \eqref{recursion2}. After re-arrangement, we obtain the following:
\begin{align}
\E{||\nabla F(\wts_j)   ||_2^2} & \leq \frac{2(\E{F(\wts_{j})}-\E{F(\wts_{j+1}}))}{\eta \gamma'} + \frac{L\eta\sigma^2}{Km\gamma'}.
\end{align}

Taking summation from $j=0$ to $j=J-1$, we get,
\begin{align}
\frac{1}{J}  \sum_{j=0}^{J-1} \E{|| \nabla F(\wts_j)   ||_2^2}    & \leq\frac{2(\E{F(\wts_0)}-\E{F(\wts_J)})}{J \eta \gamma' } +  \frac{L\eta\sigma^2}{Km\gamma'} \nonumber \\
& \overset{(a)}{\leq}\frac{2(F(\wts_0)-F^*)}{J \eta \gamma' } +  \frac{L\eta\sigma^2}{Km\gamma'}.
\end{align}
Here (a) follows since we assume $\wts_0$ to be known and also from   $\E{F(\wts_J)} \geq F^*$.

\end{proof}

\section{AdaSync: Derivation of optimal $K$ for achieving the best error-runtime trade-off}
\label{sec:adasync_derivation}
\subsection{K-sync SGD} For general distributions, we are required to solve for $$\frac{d u(K)}{d K} =  \frac{2(F(\wts_{start}))}{t\eta }\frac{d \E{X_{K:P}}}{d K} - \frac{L\eta \sigma^2}{K^2m}=0.$$ For the case of exponential or shifted-exponential distributions, this reduces to a quadratic equation in $K$ as follows:
$$ \left(\frac{2(F(\wts_{start}))}{t\eta } \right) K^2 = \frac{L\eta \sigma^2 \mu}{m} (P-K).$$
To actually solve for this equation, we would need the values of the Lipschitz constant, variance of the gradient etc. which are not always available. So, we propose a heuristic here. 

Observe that, the larger is $F(\wts_{start})$, the smaller is the value of $K$ required to minimize $u(K)$. We assume that $F(\wts_{start})$ is maximum at the beginning of training, \textit{i.e.}, when $\wts_{start}=\wts_{0}$. Hence we start with the smallest initial $K$,~e.g., $K_0=1$. Subsequently, after every time interval $t$, we choose $K$ as follows:
$$ \frac{K^2}{P-K} = \frac{K_0^2}{P-K_0} \frac{F(\wts_{0})}{F(\wts_{start})} ,$$
where we assume that $K_0$ satisfies $\left(\frac{2(F(\wts_{0}))}{t\eta } \right) K_0^2 = \frac{L\eta \sigma^2 \mu}{m} (P-K_0)$. Thus, to summarize, our method for gradually increasing $K$ is as follows:
\begin{itemize}
\item Start with an initial $K_0$ (typically $1$).
\item After time $t$, update $K$ by solving the quadratic equation:
$K^2 =  \frac{K_0^2 F(\wts_{0})}{(P-K_0)F(\wts_{start})} (P-K).$
\end{itemize}

We also verify that the second derivative is positive, \textit{i.e.},
$$\frac{d^2 u(K)}{d K^2} =  \frac{2(F(\wts_{start}))}{t\eta (P-K)^2 } + \frac{2L\eta \sigma^2}{K^3m} >0 .$$
\subsection{K-batch-sync SGD} 
For general distributions, if we assume $\E{T} \approx K \E{X_{1:P}}$ (see \Cref{lem:runtime kbatch sync}) we are required to solve for $$\frac{d u(K)}{d K} =  \frac{2(F(\wts_{start}))}{t\eta }\frac{d K \E{X_{1:P}}}{d K} - \frac{L\eta \sigma^2}{K^2m}=0.$$ This leads to
$$   K^2 =  \frac{L\eta \sigma^2t\eta}{2m(F(\wts_{start}))\E{X_{1:P}}}.$$

The interesting observation is that even if we do not know the constants or $\E{X_{1:P}}$, we can still use the trick that we used in the previous case. We can start with the smallest initial $K$, e.g., $K_0=1$. Then, after each time interval $t$, we update $K$ by solving for
$$ K^2 = K_0^2 \frac{F(\wts_{0})}{F(\wts_{start})}.$$
Thus, our method of gradually varying synchronicity is as follows:
\begin{itemize}
\item Start with an initial $K_0$ (typically $1$).
\item After time $t$, update $K$ as follows: $K = K_0\sqrt{ \frac{ F(\wts_{0})}{F(\wts_{start})} }.$
\end{itemize}
We also verify that the second derivative is positive, \textit{i.e.},
$$\frac{d^2 u(K)}{d K^2} =   \frac{2L\eta \sigma^2}{K^3m} >0 .$$

\subsection{K-async SGD} For general distributions, the runtime of $K$-async is upper bounded by that of $K$-sync (see \Cref{lem:runtime kasync}). For exponential distributions, the two become equal and an algorithm similar to $K$-sync works. Here, we examine the interesting case of shifted-exponential distribution. We approximate $\E{T}\approx \frac{K \Delta}{P} + \frac{K \log{P}}{P \mu}$ (see \Cref{lem:tight kasync}). This leads to
$$K^2=   \frac{L\eta \sigma^2t\eta P \mu }{2m(F(\wts_{start}))(\Delta \mu + \log{P} )  }.$$ 
Thus, we could start with a small $K_0$ and after each time interval $t$, we can update $K$ by solving for
$$ K^2 = K_0^2 \frac{F(\wts_{0})}{F(\wts_{start})},$$ in a manner similar to $K$-batch-sync. We can also verify that the second derivative is positive, \textit{i.e.},
$$\frac{d^2 u(K)}{d K^2} =   \frac{2L\eta \sigma^2}{K^3m} >0 .$$

\subsection{K-batch-async SGD} For any general distribution, the runtime is given by $\E{T}= \frac{K \E{X}}{P}$. This leads to
$$K^2=   \frac{L\eta \sigma^2t\eta P  }{2m(F(\wts_{start}))(\E{X} )  }.$$ 
Again, similar to the previous case, we could start with a small $K_0$ and after each time interval $t$, we can update $K$ by solving for
$$ K^2 = K_0^2 \frac{F(\wts_{0})}{F(\wts_{start})}.$$  We can also verify that the second derivative is positive, \textit{i.e.},
$$\frac{d^2 u(K)}{d K^2} =   \frac{2L\eta \sigma^2}{K^3m} >0 .$$

\section{Additional Experimental Results}
\label{app:additional_experimental_results}

Here, we include some additional experimental results to complement the results in the main paper.

\begin{figure}[!htbp]
    \centering
    \begin{subfigure}{.33\textwidth}
    \centering
    \includegraphics[width=\textwidth]{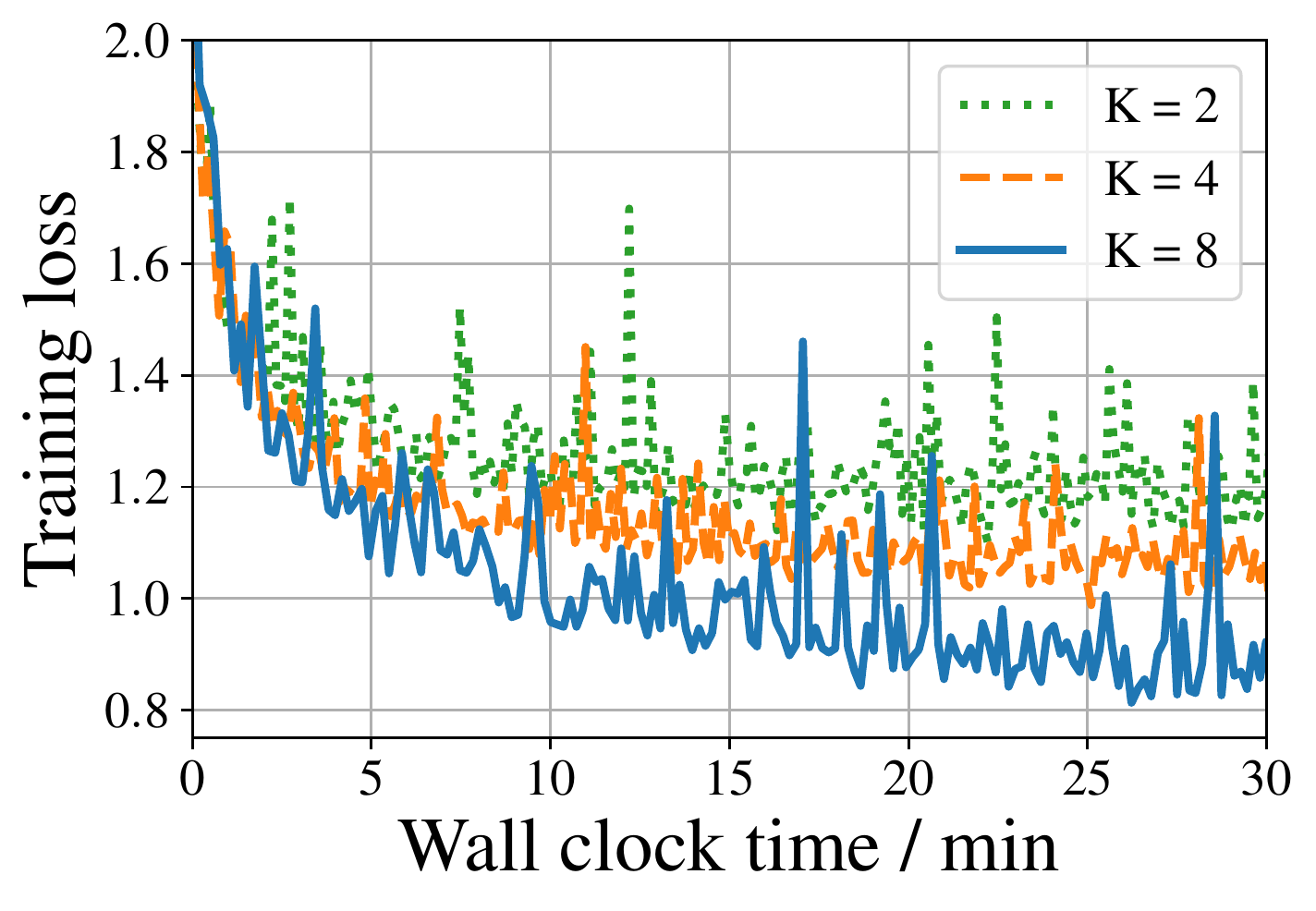}
    \caption{No artificial delay.}
    \label{fig:cifar_sync_loss_0}
    \end{subfigure}%
    ~
    \begin{subfigure}{.33\textwidth}
    \centering
    \includegraphics[width=\textwidth]{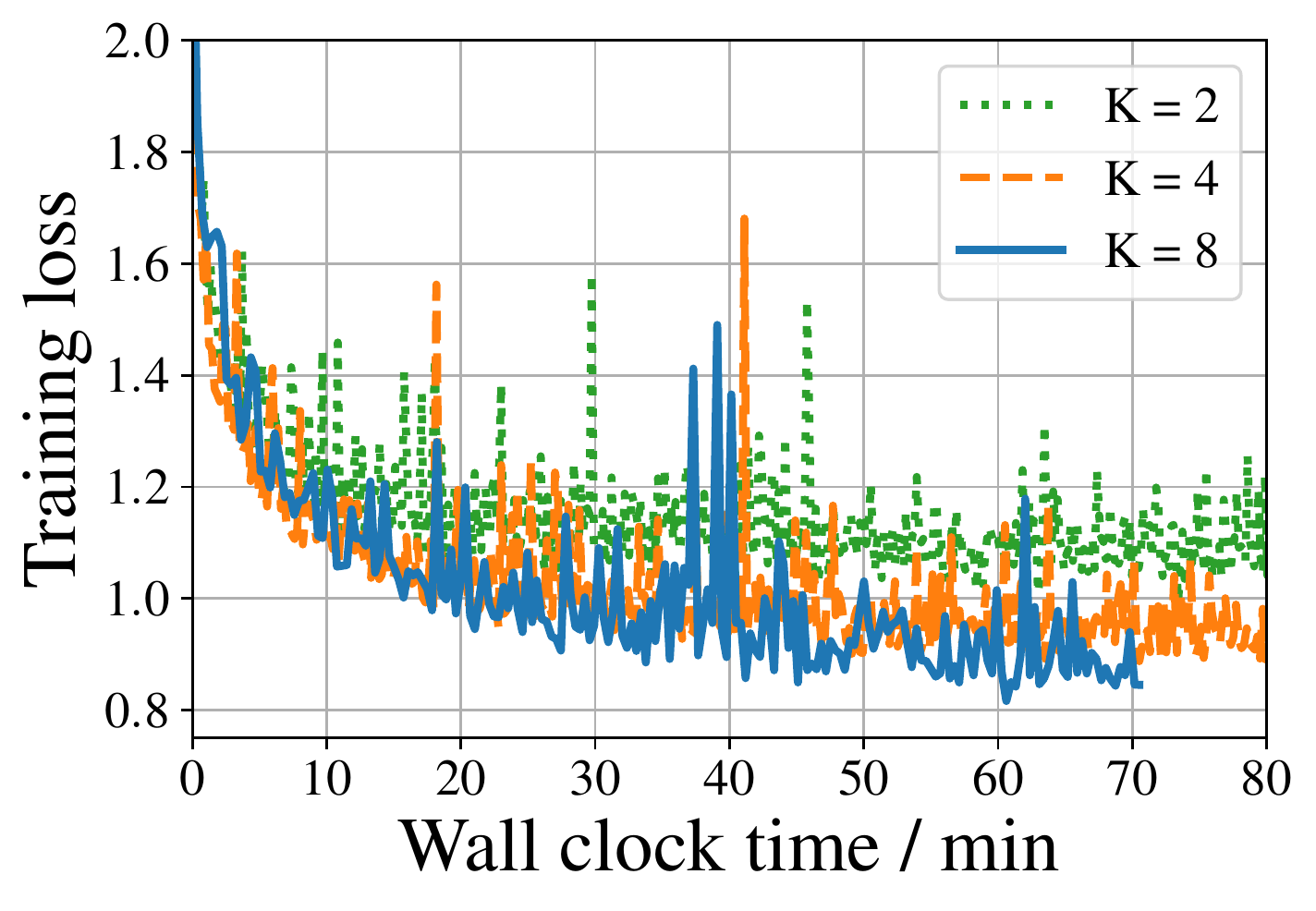}
    \caption{Delay with mean 0.02s.}
    \label{fig:cifar_sync_loss_2}
    \end{subfigure}%
    ~
    \begin{subfigure}{.33\textwidth}
    \centering
    \includegraphics[width=\textwidth]{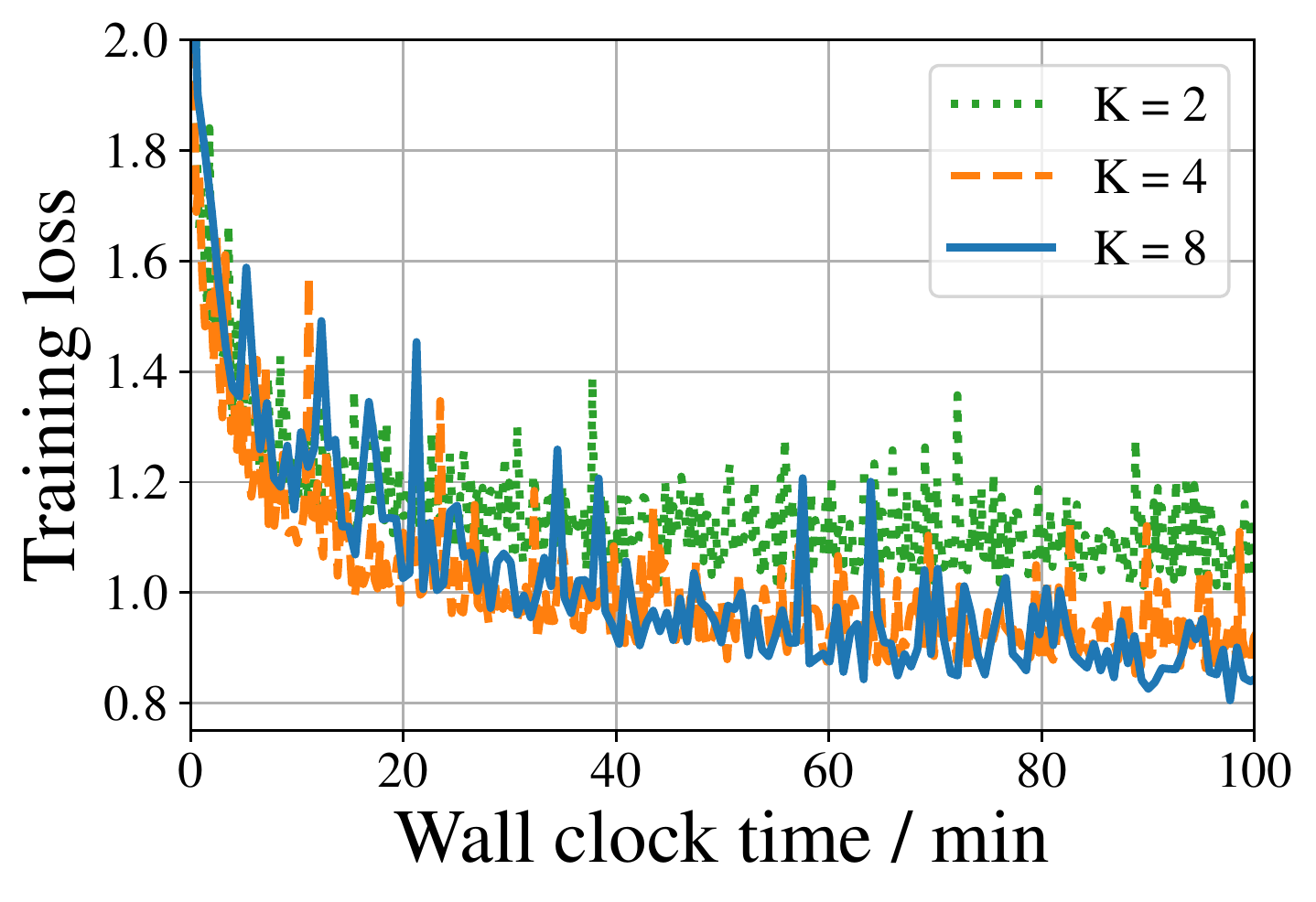}
    \caption{Delay with mean 0.05s.}
    \label{fig:cifar_sync_loss_5}
    \end{subfigure}%
    \caption{Training loss of \textbf{K-sync SGD} on CIFAR-10 with $8$ worker nodes. 
    \label{fig:cifar_ksync_loss}}
\end{figure}

\begin{figure}[!htbp]
    \centering
    \begin{subfigure}{.33\textwidth}
    \centering
    \includegraphics[width=\textwidth]{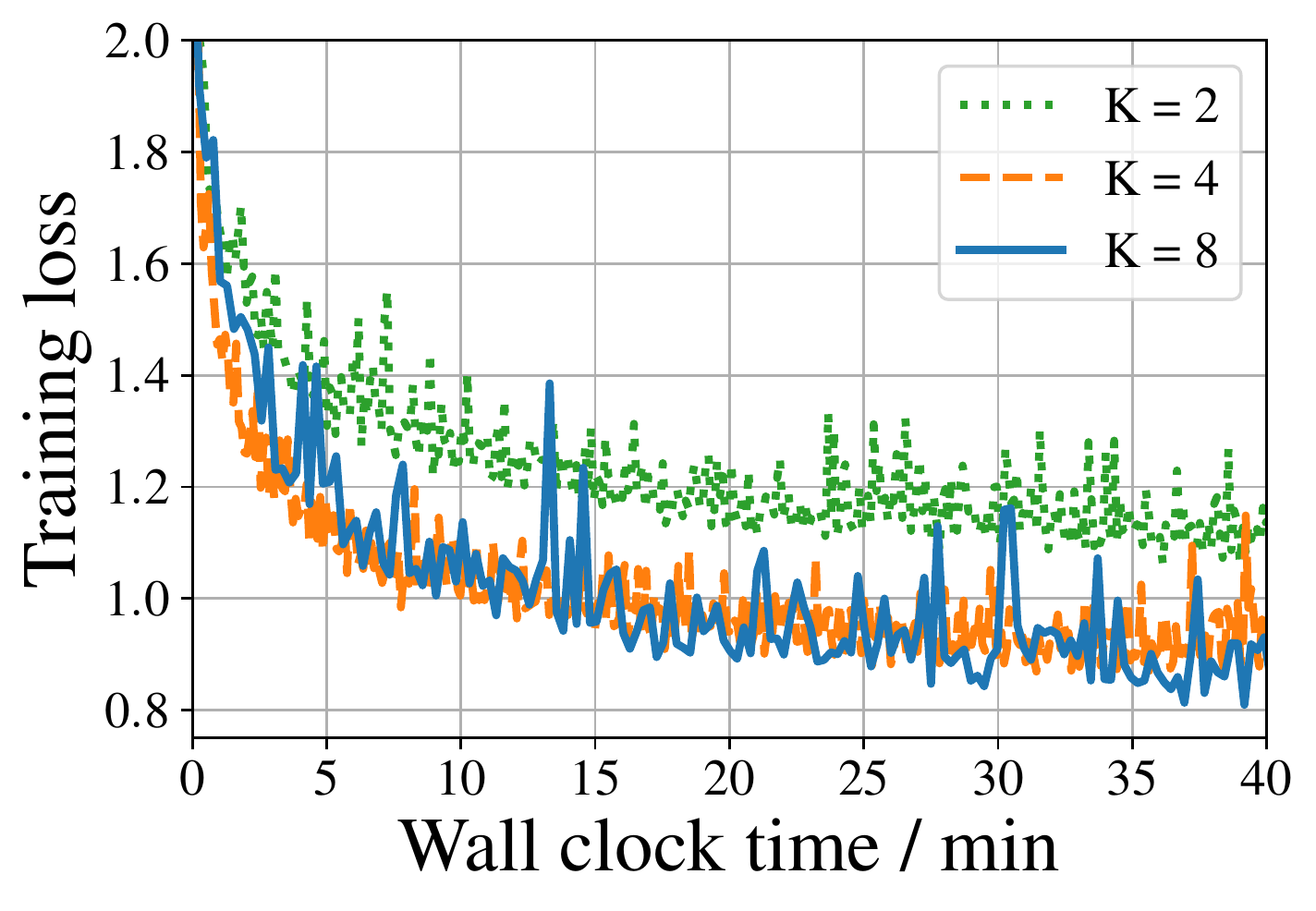}
    \caption{No artificial delay.}
    \label{fig:cifar_async_loss_0}
    \end{subfigure}%
    ~
    \begin{subfigure}{.33\textwidth}
    \centering
    \includegraphics[width=\textwidth]{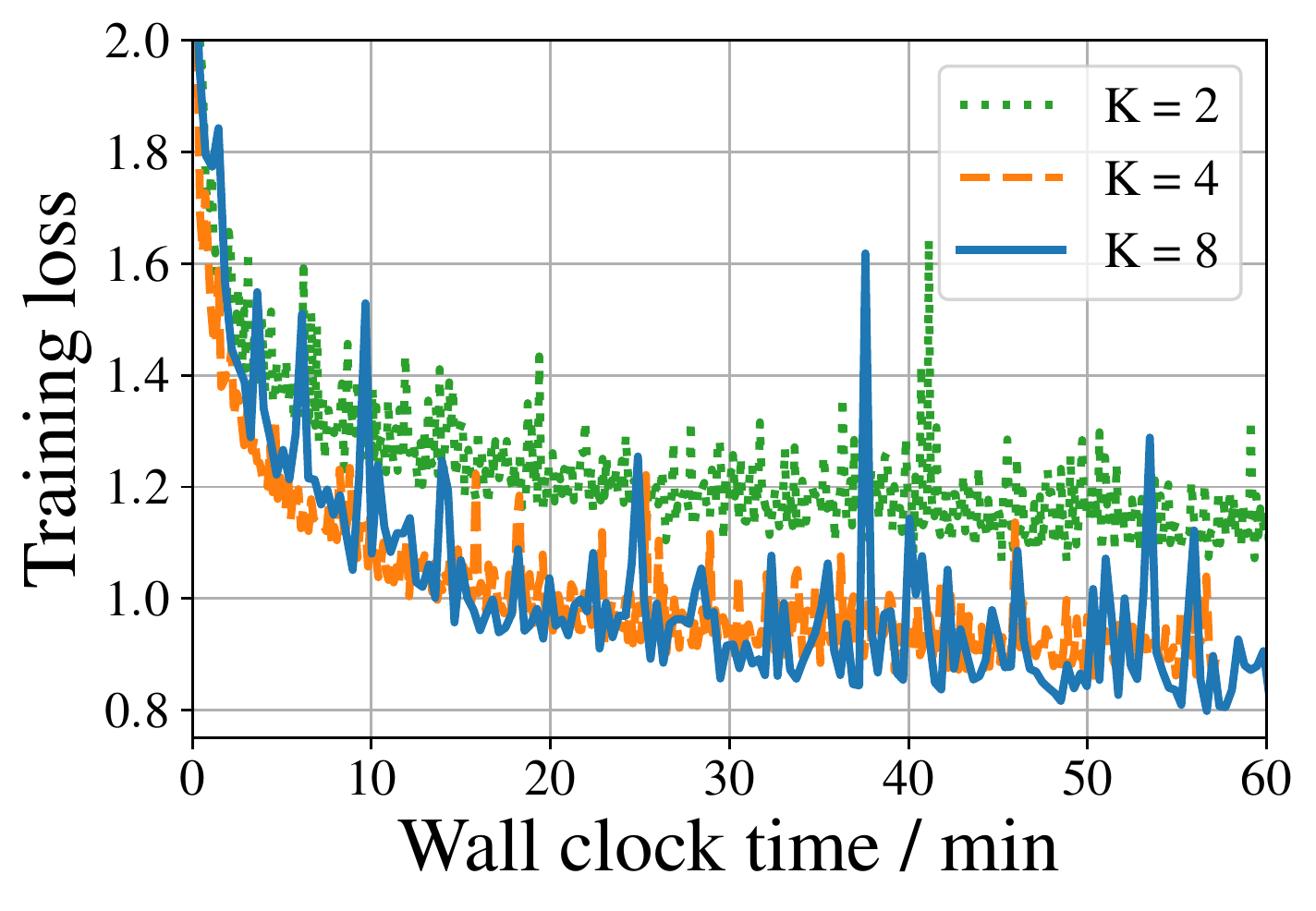}
    \caption{Delay with mean 0.02s.}
    \label{fig:cifar_async_loss_2}
    \end{subfigure}%
    ~
    \begin{subfigure}{.33\textwidth}
    \centering
    \includegraphics[width=\textwidth]{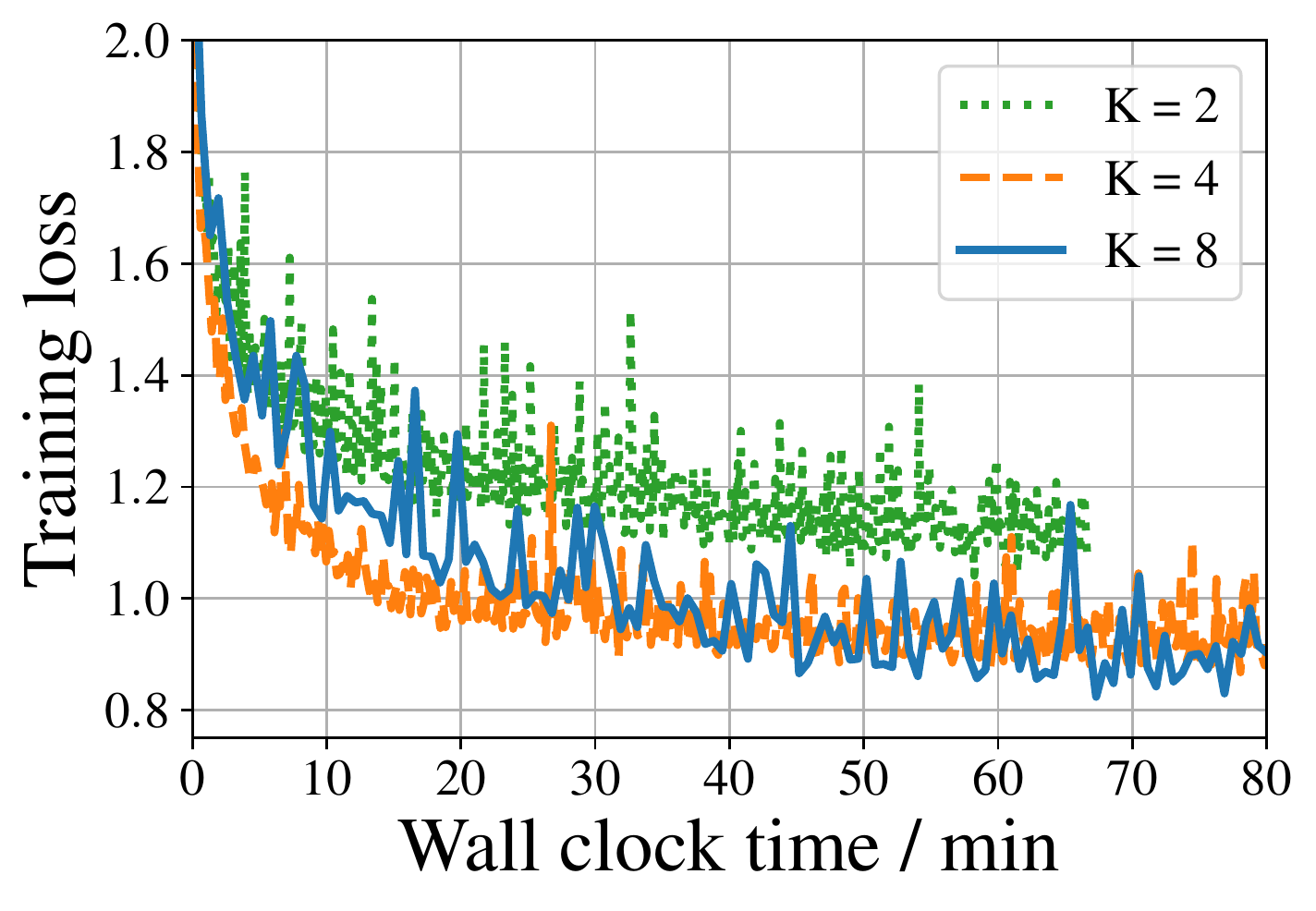}
    \caption{Delay with mean 0.05s.}
    \label{fig:cifar_async_loss_5}
    \end{subfigure}%
    \caption{Training loss of \textbf{K-async SGD} on CIFAR-10 with $8$ worker nodes. 
    \label{fig:cifar_kasync_loss}}
\end{figure}

\end{document}